\documentclass[pdftex]{article}

\usepackage[ruled,lined,noend]{algorithm2e}
\usepackage{multirow}
\usepackage{diagbox}
\usepackage{colortbl}
\usepackage{amsthm}
\usepackage{float}
\usepackage{overpic}
\usepackage{subfig} 
\usepackage{subcaption}
\newtheorem{myDef}{Definition}
\newtheorem{remark}{Remark}
\newtheorem{corollary}{Corollary}
\usepackage{amsthm,amsmath,amssymb}
\usepackage{mathrsfs}
\usepackage[nonatbib,preprint]{neurips_2024}
\usepackage{authblk}
\usepackage{graphicx}

\usepackage[utf8]{inputenc} % allow utf-8 input
\usepackage[T1]{fontenc}    % use 8-bit T1 fonts
\usepackage{hyperref}       % hyperlinks
\usepackage{url}            % simple URL typesetting
\usepackage{booktabs}       % professional-quality tables
\usepackage{amsfonts}       % blackboard math symbols
\usepackage{nicefrac}       % compact symbols for 1/2, etc.
\usepackage{microtype}      % microtypography
\usepackage{multirow}
\usepackage{caption, subcaption}
\usepackage[dvipsnames]{xcolor}
\usepackage{amsmath}  % 数学公式
\usepackage{wrapfig}

\usepackage{algpseudocode}
\usepackage{float}
\usepackage{color}
\usepackage{colortbl}
\usepackage{rotating}
\usepackage{times}
\usepackage{epsfig}
\usepackage{enumitem}
\usepackage{makecell}
\usepackage{lineno}
\usepackage{tabularx}
\usepackage{xcolor}
\usepackage{multirow}

\usepackage{hhline}
\usepackage{textcomp,booktabs}
\usepackage{caption}
\usepackage{stmaryrd}
\usepackage[mathscr]{eucal}
\usepackage{lineno}
\usepackage[export]{adjustbox}
\usepackage{comment}
\usepackage{multicol,lipsum}
\usepackage{amsfonts}
\usepackage{arydshln}
\usepackage{amsthm}

\def\1{mathbb{1}}
\def\0{{\bf 0}}
\def\1{{\bf 1}}

\definecolor{purple}{rgb}{0.56,0.27,0.68}
\definecolor{red}{rgb}{0.95,0.4,0.4}
\definecolor{purered}{rgb}{1,0,0}
\definecolor{blue}{rgb}{0.4,0.4,0.95}
\definecolor{darkblue}{rgb}{0,0,0.8}
\definecolor{grey}{rgb}{0.6,0.6,0.6}
\definecolor{col1}{RGB}{232, 161, 148}
\definecolor{col2}{RGB}{148, 187, 232}
\definecolor{col3}{RGB}{206, 239, 255}
\definecolor{lightgrey}{rgb}{0.85,0.85,0.85}
\definecolor{lightlightgrey}{rgb}{0.9,0.9,0.9}
\definecolor{verylightBG}{rgb}{0.9,0.99,0.99}
\definecolor{darkgreen}{rgb}{0.3, 0.75, 0.3}

\RequirePackage{xspace}
\makeatletter
\DeclareRobustCommand\onedot{\futurelet\@let@token\@onedot}
\def\@onedot{\ifx\@let@token.\else.\null\fi\xspace}

\def\eg{\emph{e.g}\onedot} 

\def\ie{\emph{i.e}\onedot}

\makeatother

\newtheorem{theorem}{Theorem}
\newtheorem{lemma}{Lemma}

\definecolor{remark}{rgb}{1,.5,0} 
\definecolor{citecolor}{rgb}{0,0.443,0.737} 
\definecolor{linkcolor}{rgb}{0.956,0.298,0.235} 
\hypersetup{
    colorlinks=true,%
    citecolor=citecolor,%
    filecolor=citecolor,%
    linkcolor=linkcolor,%
    urlcolor=magenta
}

\title{FedHPL: Efficient Heterogeneous Federated Learning with Prompt Tuning and Logit Distillation}

\author[1]{Yuting Ma}
\author[2$^{\dagger}$]{Lechao Cheng}
\author[2]{Yaxiong Wang}
\author[3]{Zhun Zhong}
\author[1$^{\dagger}$]{Xiaohua Xu}
\author[2]{Meng Wang}
\affil[1]{University of Science and Technology of China}
\affil[2]{Hefei University of Technology}
\affil[3]{University of Nottingham}

\begin{document}

\maketitle

\begin{abstract}
  Federated learning (FL) is a popular privacy-preserving paradigm that enables distributed clients to collaboratively train models with a central server while keeping raw data locally.
  In practice, distinct model architectures, varying data distributions, and limited resources across local clients inevitably cause model performance degradation and a slowdown in convergence speed.
  However, existing FL methods can only solve some of the above heterogeneous challenges and have obvious performance limitations.
  Notably, a unified framework has not yet been explored to overcome these challenges.
  Accordingly, we propose FedHPL, a parameter-efficient unified $\textbf{Fed}$erated learning framework for $\textbf{H}$eterogeneous settings based on $\textbf{P}$rompt tuning and $\textbf{L}$ogit distillation.
  Specifically, we employ a local prompt tuning scheme that leverages a few learnable visual prompts to efficiently fine-tune the frozen pre-trained foundation model for downstream tasks, thereby accelerating training and improving model performance under limited local resources and data heterogeneity.
  Moreover, we design a global logit distillation scheme to handle the model heterogeneity and guide the local training.
  In detail, we leverage logits to implicitly capture local knowledge and design a weighted knowledge aggregation mechanism to generate global client-specific logits. 
  We provide a theoretical guarantee on the generalization error bound for FedHPL.
  The experiments on various benchmark datasets under diverse settings of models and data demonstrate that our framework outperforms state-of-the-art FL approaches, with less computation overhead and training rounds.

\end{abstract}
\section{Introduction}
\label{sec:intro}

Federated learning (FL)~\cite{FedAvg} is a privacy-preserving machine learning paradigm that enables decentralized parties to collaboratively train models in a distributed manner. 
This is achieved by sharing local model updates without exposing the underlying private data to the central server. 
Although FL has gained significant traction in homogeneous settings, it encounters challenges~\cite{fl_chall, hetepro1, DFRD, FedFed} in heterogeneous settings over real-world clients.
\textit{Data heterogeneity} is a major challenge in FL, primarily caused by the imbalanced data distribution among clients.
Most previous works~\cite{FedDC, SCAFFOLD, MOON, FedProx} have focused on addressing this challenge by optimizing the loss objective and aggregation process to reduce the discrepancy between the global distribution and clients' distributions.
However, these algorithms usually converge slowly and require more computing resources because they need to train the entire model from scratch.

A feasible mechanism~\cite{SGPT, FedPR, pFedPG} to speed up training and reduce trainable parameters over limited local resources is to select a pre-trained foundation model as the backbone of the local model and fine-tune it with visual prompt tuning (VPT)~\cite{VPT}.
Specifically, clients freeze the large pre-trained backbone and only need to employ a few learnable task-specific prompts and a simple classification head for fine-tuning, which accelerates the convergence speed and reduces the computational burden with fewer trainable parameters.
However, these studies only consider homogeneous clients and cannot generalize algorithms to clients with distinct local resources and different pre-trained models.

Recently, personalized federated learning (pFL)~\cite{FedRep, PerFedAVG, PGFed, pFedME}, which considers the varying edge resource, has attracted research interests.
In pFL, clients adopt personalized models for local training to maximize the utilization of local resources and adapt to local data distributions.
In particular, if clients select different model structures, it constitutes another challenge in FL: \textit{model heterogeneity}.
The varying structures and embedding dimensions in model heterogeneity, leading to difficulties of global parameter aggregation in FL.
To handle the problem, existing efforts generally fall into two categories:
1) exploiting knowledge distillation~\cite{FedGEMS, FedHKT, FedMD} to implicitly capture and transfer local distributions among clients instead of model updates;
2) aligning model architectures (\eg adding projection layer or additional model structure)~\cite{FedProto, FedGH, FedTGP, FedGEN} for model parameter aggregation.
However, the above methods usually need more public resources and only consider model heterogeneity of small models, causing limited representation ability and performance improvement.
But directly applying the large foundation models usually requires more computing resources and training rounds.
Hence, we raise a question:
\textit{How to implement federated learning in the setting of model heterogeneity and data heterogeneity among clients, while utilizing large foundation models in limited local resources and training rounds to improve the local model performance?}

To answer this question, we propose a parameter-efficient unified FL framework named FedHPL in various settings of model architecture and data distribution as shown in Figure~\ref{fig:tesaer}.
Then we introduce our framework from two perspectives: \textit{local prompt tuning} and \textit{global logit distillation}.
From the local perspective, we leverage the large pre-trained foundation model as the backbone of the local model and use its strong representations to alleviate performance degradation and adapt data heterogeneity.
Moreover, we freeze the backbone and use a few learnable task-specific prompts along with a linear layer to fine-tune the local model over the limited local resources, further reducing the computing overhead and accelerating training speed.
From the global perspective, we employ logit distillation~\cite{FedHE}, which solely depends on the number of labels, to handle model heterogeneity in FL.
Specifically:
1) Clients only upload correctly predicted logits which implicitly represent local empirical knowledge distribution.
2) For generating global per-class logit for each client on the server side, we design a weighted knowledge aggregation mechanism based on the proportion of latent dimensions among local models.
3) The global client-specific logits can transfer knowledge and guide local training. 
Furthermore, clients can average the uploading logits by class to reduce the communication overhead.
Then, we provide a generalization error bound for FedHPL from prompt tuning and logit distillation, further demonstrating the impact of components in FedHPL on model performance, and accordingly investigate these components with experiments.

\begin{figure}
    \centering 
    \includegraphics[width=0.98\textwidth]{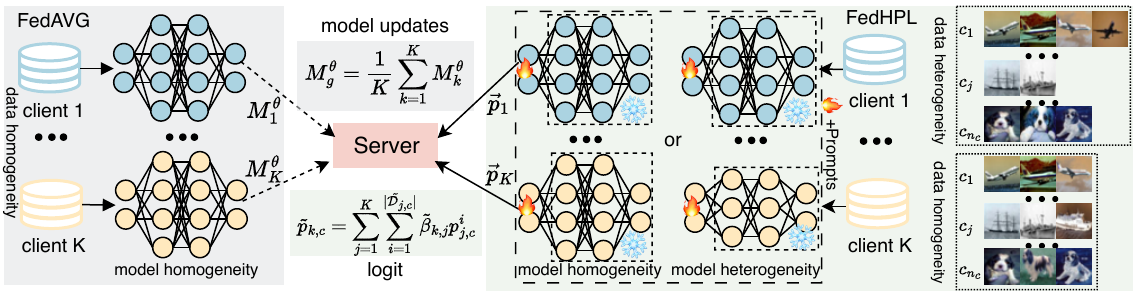}
    \caption{FedAVG only considers the data homogeneity and model homogeneity based on parameter aggregation of models and trains the entire model parameters. Compared to it, FedHPL further considers the data heterogeneity and model heterogeneity based on logit distillation and only trains a few parameters with the frozen backbone by prompt tuning.}
    \label{fig:tesaer}
\vspace{-2em}
\end{figure}

\textbf{Contributions.}
In conclusion, our main contributions are:
1) We propose a novel parameter-efficient unified FL framework named FedHPL to address the challenge of model heterogeneity and data heterogeneity with limited local resources. We leverage local prompt tuning with pre-trained backbones and design a global logit distillation scheme with a weighted knowledge aggregation mechanism.
2) We present a theoretical guarantee on the generalization error bound to show the influence of each component in FedHPL.
3) Experiments on three benchmark datasets in various settings show that FedHPL outperforms SOTA methods, with fewer trainable parameters and communication rounds.

\section{Related work}

\textbf{Federated learning.}
FedAVG~\cite{FedAvg} is a distributed machine learning paradigm in which edge devices train models locally and upload model updates to a central server for parameter aggregation in collaborative learning. 
However, in real-world scenarios, data and models are usually heterogeneous, leading to challenges such as performance degradation. 
Some approaches~\cite{FedDyn, SCAFFOLD, MOON, FedProx} add penalty terms to the loss function or align representations to address heterogeneous challenges, but they require that all clients and the central server share the same model structure. 
Alternatively, other studies~\cite{FedPer, PGFed, Flow} address the model heterogeneity challenge by designing customized models on the client side, known as pFL. 
However, it requires more convergence epochs to adapt simple model heterogeneity and often results in limited performance improvement. 
In contrast to them, our method selects appropriate pre-trained models as the backbone of local models and fine-tunes them with local prompt tuning and global logit distillation instead of designing customized models and aggregating model parameters in pFL, further handling both model heterogeneity and data heterogeneity. 

\textbf{Visual prompt tuning in federated learning.}
With the advancement of available public pre-trained foundation models~\cite{ViT, LM1, LM2, LM3}, VPT~\cite{VPT} and its variants~\cite{wang2024revisiting, yoo2023improving} have gained increasing popularity in federated learning. 
VPT exploits a large pre-trained model, a few task-specific learnable prompts, and a linear head for adapting downstream tasks.
Recently, studies~\cite{FedPR, pFedPT, pFedPG} have applied VPT to alleviate the model degradation due to data heterogeneity and reduce training time. 
For instance, pFedPG~\cite{pFedPG} trains personalized visual prompts on local devices with frozen pre-trained backbones and observes the local optimization direction to generate client-specific visual prompts through a prompt generator under data heterogeneity. 
However, these methods do not address the challenge of model heterogeneity, nor do they provide a theoretical bound for efficient fine-tuning in FL. 
This is a critical aspect where our work markedly differs from these approaches. 

\textbf{Knowledge distillation in federated learning.}
Knowledge distillation (KD)~\cite{KD, wang2023improving} involves transferring knowledge from a larger pre-trained ``teacher'' model to a smaller ``student'' model, where the student learns to emulate the teacher's behavior by producing similar outputs on a shared dataset. 
This technique leverages knowledge (\eg feature embeddings or logits) transfer across clients rather than model parameter aggregation, making it a viable FL approach~\cite{EKD4, FedHKT, FedMD, FedDF, DFRD, FedGEN} over model heterogeneity. 
Traditional methods (\eg FedHKT~\cite{FedHKT}, FedMD~\cite{FedMD}, and FedDF~\cite{FedDF}) need to exploit public data to guide knowledge transfer between a central server and local clients.
Recently, a data-free manner in knowledge distillation~\cite{FedHE, DFRD, FedGEN} has emerged in FL.
For example, FedGen~\cite{FedGEN} introduces a lightweight global generator over model heterogeneity among clients and produces synthetic data to capture the global distribution, thereby eliminating the need for public datasets. 
The above methods tend to introduce shared datasets or additional model architecture to adapt to model heterogeneity, whereas we only select qualified logits without public data or extra models to facilitate knowledge transfer over model heterogeneity.
\section{Proposed method}

\subsection{Overview}

As shown in Figure~\ref{fig:framework}, we consider $K$ clients in FL, each client $k$ has a local private dataset $\mathcal{D}_k=\{(x_{k}^{i}, y_{k}^{i})\}_{i=1}^{|\mathcal{D}_k|}$, where $|\mathcal{D}_k|$ is the number of local samples, $x$ is the sample and $y$ is the corresponding label. 
A local model can be divided into a backbone $F(\cdot;\omega)$ and a classification head $H(\cdot;\theta)$, parameterized as $\omega$ and $\theta$ respectively. 
$L_{E,k}(\cdot)$ is used to segment and embed each image $x_k^i$ into latent space, then obtain the collection of patch embeddings.
The local models in FedHPL are trained based on \textit{local prompt tuning} and \textit{global logit distillation} with a weighted knowledge aggregation mechanism.
During the phase of \textit{local prompt tuning}, considering limited local resources (\eg computational resources), clients load different scales of the large pre-trained foundation model from the central server to the local backbone $F$ and perform downstream tasks.
To further reduce trainable parameters, clients freeze the parameter of the backbone (\ie $\omega$), which is denoted as $\omega^*$, and perform classification tasks with trainable prompts and the head $H$.
The loss function of client $k$ for local prompt tuning is based on the cross-entropy loss function $\ell^{ce}$:
\begin{equation}
  \mathcal{L}_k^{pt} = \frac{1}{|\mathcal{D}_k|}\sum_{i=1}^{|\mathcal{D}_k|} \ell^{ce}(H_k(F_k([[cls]_k^i,\mathsf{P}_k,L_{E,k}(x_k^i)];\omega_k^*);\theta_k),y_{k}^{i}),
\label{eq:local_objective}
\end{equation}
where $[cls]_{k}^{i}$ is the token for classification tasks and $\mathsf{P}_k$ is the learnable prompts.
During the phase of \textit{global logit distillation}, clients upload qualified local logits and the server employ a weighted knowledge aggregation mechanism to generate global logits over model heterogeneity.
Finally, with the local logit $p_k^i$ and global client-specific logit $\tilde{p}_{k,c}$ corresponding to each class $c$ which can refer to Eq. (\ref{eq:ClS_Head}) and Eq. (\ref{eq:sum_logit1}), we formally design an objective function for client $k$ in various FL settings:
\begin{equation}
  \underset{\mathsf{P}_k, \theta_k} {arg\min} [\mathcal{L}_k := \mathcal{L}_{k}^{pt} + \gamma \mathcal{L}_{k}^{kd} = \mathcal{L}_{k}^{pt} + \frac{\gamma}{|\mathcal{D}_k|} \sum_{c=1}^{n_c} \sum_{\forall(p_{k}^{i},y_{k}^{i}) \in \mathcal{D}_k, y_{k}^{i}=c} \ell^{kd}(\tilde{p}_{k,c},p_{k}^{i})],
\label{eq:global_objective}
\end{equation}
where $n_{c}$ is the number of labels and $\gamma$ controls the trade-off between $\mathcal{L}_{k}^{pt}$ and $\mathcal{L}_{k}^{kd}$ ($\ell^{kd}$ is shown in Eq. (\ref{eq:KD_loss})).
As a result, $\mathsf{P}_k$ and $H_k$ can be optimized by gradient descent with learning rate $\eta$.

\begin{figure}
    \centering
    \includegraphics[width=0.98\textwidth]{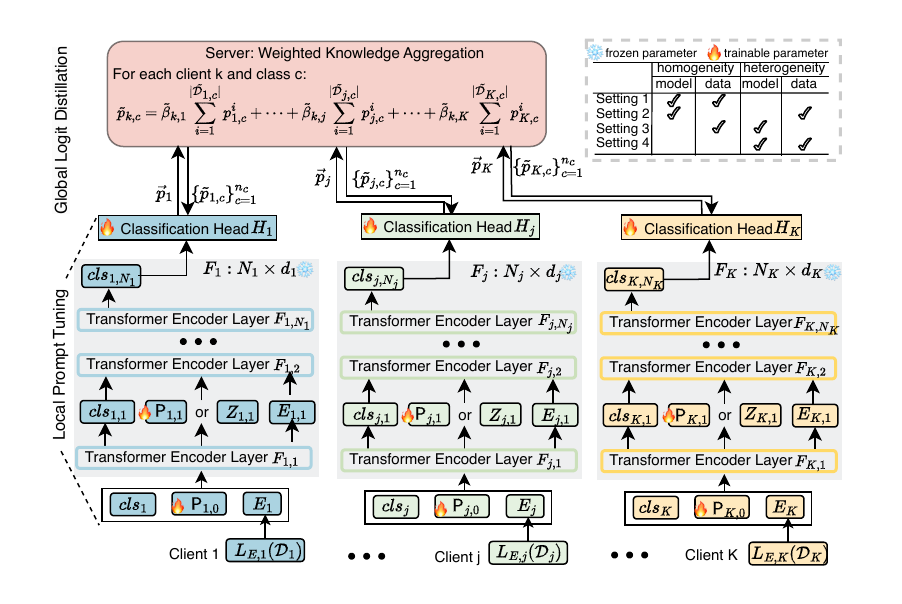}
    \caption{The FedHPL framework over various heterogeneous FL settings including homogeneity. The backbone among clients has distinct layer numbers ($N_k$) and embedding dimensions ($d_k$). FedHPL consists of local prompt tuning and global logit distillation with a weighted knowledge aggregation mechanism. Clients upload all correctly predicted logits which only related to the number of labels $n_c$ to a server, and the server generates global per-class knowledge for each client.}
    \label{fig:framework}
\vspace{-1em}
\end{figure}

\subsection{Local prompt tuning}
Before performing classification tasks, each client downloads the appropriate pre-trained parameters of the foundation model from the server to $F_k$ and keeps the backbone parameters frozen throughout the entire training phase. 
Then, each client injects a small number of learnable continuous visual parameters, denoted as prompts $\mathsf{P}_k$, into the input space of the backbone (typically using Transformer structures).
These prompts can encode client-specific data distribution and guide the local training with the frozen backbone to adapt downstream tasks.
After local prompt tuning based on VPT~\cite{VPT}, the head output (\ie logit) can implicitly capture local distribution knowledge.

Specifically, an image $x_k^i$ is divided into $M$ patches and embedded into a $d_k$-dimensional latent space with position encoding, then gets $E_{k}^i=L_{E,k}(x_k^i), E_{k}^i \in \mathbb{R}^{M \times d_k}$.
Subsequently, these patch embeddings $E_{k}^i$ are stacked and concatenated together with an initial classification token $[$cls$]_{k}^i \in \mathbb{R}^{d_k}$ and $\mathsf{P}_{k,0} \in \mathbb{R}^{n \times d_k}$, where $n$ is the number of prompts of a backbone layer.
Then feed them into the $1$-th backbone layer:
\begin{equation}
     [[cls]_{k,1}^i, Z_{k,1}^i, E_{k,1}^i] = F_{k,1}([[cls]_{k}^i, \mathsf{P}_{k,0}, E_{k}^i];\omega_{k,1}^*).
\label{eq:Transformer}
\end{equation}
We use $a$ to represent the layer index of the backbone $F_k$ (\ie the $a$-th backbone layer $F_{k,a}$ with the frozen parameter $\omega_{k,a}^*$), where $a \in [1, N_k]$ and $N_k$ is the number of backbone layers.
Thus, $[cls]_{k,a}^i$ and $E_{k,a}^i$ represent the subsequent generated classification token and patch embeddings computed by $F_{k,a}$.
$Z_{k,a}^i \in \mathbb{R}^{n \times d_k}$ is the latent features.
Due to different VPT variants, there have two different insertion positions for prompts in subsequent layers ($a>1$):
\begin{equation}
     [[cls]_{k,a}^i, Z_{k,a}^i, E_{k,a}^i] \overset{\text{shallow}}{=} F_{k,a}([[cls]_{k,a-1}^i, Z_{k,a-1}^i, E_{k,a-1}^i];\omega_{k,a}^*); \\
\label{eq:Transformer1}
\end{equation}
\begin{equation}
     [[cls]_{k,a}^i, Z_{k,a}^i, E_{k,a}^i] \overset{\text{deep}}{=} F_{k,a}([[cls]_{k,a-1}^i, \mathsf{P}_{k,a-1}, E_{k,a-1}^i];\omega_{k,a}^*). \\
\label{eq:Transformer2}
\end{equation}
If FedHPL adopts VPT-shallow, prompts $\mathsf{P}_k$ (\ie $\mathsf{P}_{k,0}$) are only inserted into the first backbone layer, whereas in VPT-deep, prompts are inserted into each layer (\ie $\mathsf{P}_k=\{\mathsf{P}_{k,a}\}_{a=0}^{N_k-1}, \mathsf{P}_{k,a} \in \mathbb{R}^{n \times d_k}$).
Then, the final token $[cls]_{k,N_k}^{i}$ is then fed into the classification head to generate the predicted logit:
\begin{equation}
  p_k^i = H_k([cls]_{k,N_k}^i;\theta_{k}).
\label{eq:ClS_Head}
\end{equation}
As only $\mathsf{P}_k$ and $H_k$ need to be updated, clients reduce training burdens.
Meanwhile, the well-trained backbone helps clients generate strong feature embeddings over limited data samples and data heterogeneity, thereby mitigating performance degradation and shortening the training time.

\subsection{Global logit distillation}
Since clients have different latent dimensions in model heterogeneity, traditional parameter aggregation on model updates or prompts is not suitable for collaborative learning.
To solve this problem, we exploit knowledge distillation based on logits, which only related to the number of classes, and propose a weighted knowledge aggregation mechanism based on qualified logits.
We now describe the global logit distillation from \textit{client uploading}, \textit{server aggregation}, and \textit{logit distillation}.

\textbf{Client uploading.} 
Client $k$ transfers the correctly predicted logits with corresponding labels as the local knowledge $\vec{p}_{k}:=\{p_{k}^{i}, y_{k}^{i}\}_{i=1}^{|\tilde{\mathcal{D}}_{k}|}$ to the central server $S$ after local training, where $|\tilde{\mathcal{D}}_{k}|$ is the number of correct logits.
For those logits that are misclassified, \ie $y_{k}^{i}\ne arg\max_{c\in[0,n_c-1]}[p_{k}^{i}]_c$, we consider them as untrustworthy and do not upload them. 

\textbf{Server aggregation.} 
After collecting all local logits from clients, the server generates the global logits with a weighted knowledge aggregation mechanism.
Generally, if a local model predicts more accurately in a certain label $c$, then the corresponding uploading count of correct logits in this class $|\tilde{\mathcal{D}}_{k,c}|$ will relatively increases, thus performing a higher influence on this class of global aggregation~\cite{confidence}.
Therefore, we do not need to specially design a weight for the quantity and quality of uploading logits in the global aggregation.
However, it is necessary to focus on the impact of the model heterogeneity on global aggregation.
Inspired by~\cite{FedHKT, kornblith2019similarity}, we consider that models, that have similar architectures (\eg latent dimensions), tend to learn comparable feature representations, and consequently capture similar knowledge distribution within logits for the same label.
Accordingly, we intend to design the weight coefficient of model heterogeneity for the global logit aggregation from the latent dimension.
Formally, given the latent embedding dimension $d_j$ of the model in client $j$, the weight $\beta_{k,j}$, which determines the contribution of any client $j$ to the global client-specific logits in client $k$, is denoted as:
\begin{equation}
    \beta_{k,j} = \min ({d_{k}}/{d_{j}},{d_{j}}/{d_{k}}), \forall j \in \{1, \cdots, K\}.
\label{eq:alpha_compute}
\end{equation}
A higher $\beta_{k,j}$ implies the closer structures of two models, thereby leading to a more efficient aggregation.
Then, $S$ generates the global per-class logit for client $k$ with the weight coefficient:
\begin{equation}
\tilde{p}_{k,c}=\frac{\sum_{j=1}^{K}\beta_{k,j}\sum_{\forall (p_{j}^{i},y_{j}^{i})\in \vec{p}_j, y_{j}^{i}=c}p_{j}^{i}}{1+\sum_{j=1}^{K}\beta_{k,j}|{\mathcal{D}}_{j,c}|}=\sum_{j=1}^{K}\tilde{\beta}_{k,j}\sum_{i=1}^{|\mathcal{\tilde{D}}_{j,c}|}p_{j,c}^{i},
\label{eq:sum_logit1}
\end{equation}
where $\tilde{\beta}_{k,j} = \beta_{k,j}/({1+\sum_{j=1}^{K}\beta_{k,j}|{\mathcal{D}}_{j,c}|})$ is a constant and $|\mathcal{D}_{j,c}|$ represents the number of samples ($p_{j,c}^{i}$) of client $j$ in class $c$.
We add one to the denominator to avoid the extreme case where all $|{\mathcal{D}}_{j,c}|$ equals $0$.
Moreover, the global per-class logit $\tilde{p}_{k,c}$ mixes the local distribution of the certain class $c$ among clients rather than aggregating parameters trained on all labels, further alleviating the data heterogeneity caused by imbalanced class distribution.
Particularly, the weighted aggregation mechanism is also suitable in homogeneous settings. 

\paragraph{Logit distillation.}
The global client-specific logits fuse the local knowledge of clients for each class and thus guide the local training without uploading private data.
The local optimization for client $k$ is based on global logits under corresponding label $c$ (\ie $y_k^i$) and the distillation loss:
\begin{equation}
    \ell^{kd} = {KL}(\frac{exp(\tilde{p}_{k,c}/\mathcal{T})}{{\textstyle \sum_{c'=1}^{n_c}}exp([\tilde{p}_{k,c}]_{c'}/\mathcal{T})}\left|\right|\frac{exp(p_{k}^i/\mathcal{T})}{{\textstyle \sum_{c'=1}^{n_c}}exp([p_{k}^i]_{c'}/\mathcal{T})}),
\label{eq:KD_loss}
\end{equation}
where $\mathcal{T}$ is a temperature coefficient in KD and ${KL}$ is the Kullback-Leibler divergence.
Furthermore, clients can average local correctly predicted logits by category and only upload the per-class logits $\{\bar{p}_{k,c}\}_{c=1}^{n_c}$ with the count of logits $|\tilde{\mathcal{D}}_{k,c}|$ for each label $c$ to reduce the communication cost, where
\begin{equation}
    \bar{p}_{j,c} = \frac{1}{|\tilde{\mathcal{D}}_{j,c}|} \textstyle \sum_{\forall (p_{j}^{i},y_{j}^{i})\in \vec{p}_j, y_{j}^{i}=c} p_{j}^{i}.
\label{eq:average_logit}
\end{equation}
Theoretically speaking, $\bar{p}_{j,c}$ and $p_{j,c}$ are equivalent in global aggregation which we shown in Eq. (\ref{eq:appendix_logit}) in Appendix.
In summary, the training detail of FedHPL is shown in Algorithm~\ref{alg:algorithm_name}. 

\begin{algorithm}
\caption{The training procedure of FedHPL}
\label{alg:algorithm_name}
\KwIn{Global rounds $T$; Local epochs $T_c$; Batch size $bs$.}
\KwOut{Optimal parameters $\{\mathsf{P}_k, \theta_k\}_{k=1}^{K}$ for all clients.}
\text{Initialization: load the pre-trained parameter from the server $S$ to $F$ and freeze it as $\omega^*$.}

\For{$t=1, 2, \cdots, T$}{
    \ForEach{client k}{
        Send $\vec{p}_k \gets$ \texttt{LocalTrain($k$,$\{\tilde{p}_{k,c}\}_{c=1}^{n_c}$)} to the server. \\
    }
    $S$ generates the global client-specific logits $\{\tilde{p}_{k,c}\}_{c=1}^{n_c}$ for each class $c$; \hfill{$\rhd$ in Eq. $(\ref{eq:sum_logit1})$} \\
    Return $\{\tilde{p}_{k,c}\}_{c=1}^{n_c}$ to each client k.
}

\SetKwFunction{myFunction}{LocalTrain}
\SetKwProg{myProc}{Function}{}{}
\SetFuncSty{texttt}
\myProc{\myFunction{$k$, $\{\tilde{p}_{k,c}\}_{c=1}^{n_c}$}}{
    \For{$t=1, 2, \cdots, T_c$}{
        \For{batch $b=\{(x_k^i, y_k^i)\}_{i=1}^{bs}$ of $\mathcal{D}_k$}{
            \hbox{Compute $\{[cls]_{k, N_k}^i\}_{i=1}^{bs}$ with $F_k$ and $\{[cls]_{k}^i, \mathsf{P}_{k}, L_{E,k}(x_{k}^i)\}_{i=1}^{bs}$;\hfill{$\rhd$ in Eq. $(\ref{eq:Transformer})$$\sim$$(\ref{eq:Transformer2})$}}
            $\{p_k^i\}_{i=1}^{bs} \gets H_k(\{[cls]_{k,N_k}^i\}_{i=1}^{bs};\theta_{k})$; \\
            $\mathcal{L}_k = \frac{1}{bs}[\sum_{i=1}^{bs}l^{ce}(p_k^i,y_k^i) + \gamma \sum_{c=1}^{n_c} \sum_{\forall(p_{k}^{i},y_{k}^{i}) \in b \wedge (y_{k}^{i}=c)} \ell^{kd}(\tilde{p}_{k,c},p_{k}^{i})]$; \\
            $\mathsf{P}_k  \gets \mathsf{P}_k-\eta\frac{\partial \mathcal{L}_k}{\partial \mathsf{P}_k}$, $\theta_k \gets \theta_k-\eta\frac{\partial \mathcal{L}_k}{\partial \theta_k}$; \\
        }
    }
    \For{batch $b=\{(x_k^i, y_k^i)\}_{i=1}^{bs}$ of $\mathcal{D}_k$}{
            $\forall$ $i$ in $b$, \If{$y_{k}^{i} == arg\max_{c} H_k(F_k([[cls]_k^i,\mathsf{P}_k,L_{E,k}(x_k^i)];\omega_k^*);\theta_{k})_c$}{
                $\vec{p}_k = \vec{p}_k \cup \{p_{k}^{i}, y_{k}^{i}\}$, $|\tilde{\mathcal{D}}_{k,c}|$++, $|\tilde{\mathcal{D}}_{k}|$++;\\
        }
    }
    \Return{$\vec{p}_k: \{p_k^i, y_k^i\}_{i=1}^{|\tilde{\mathcal{D}}_k|}$}\\
}
\end{algorithm}
\vspace{-1em}

\subsection{Generalization error bound}
Here, we investigate the bound of generalization error $\mathbf{R}_{\mathbb{D}_T}(h_k)$ in FedHPL for each model $h_k$ of client $k$ over the test dataset $\mathcal{D}_T$ with its distribution $\mathbb{D}_T$ in arbitrary model and data settings.
With an input space $\mathcal{X}$ and label space $\mathcal{Y}$, a hypothesis $h: \mathcal{X} \to \mathcal{Y}$ is the local model and $\mathcal{H}$ is a hypotheses space on $\mathcal{X}$.
Suppose the local true and local empirical distribution for client $k$ over $\mathcal{X} \times \mathcal{Y}$ as $\mathbb{D}_k$ and $\hat{\mathbb{D}}_k$.
Then we connect the error bound with the local training of client $k$ over the private dataset $\mathcal{D}_k=(X_k,Y_k)$ with $|\mathcal{D}_k|$ samples.
The local training error is from the local prompt tuning error based on the cross-entropy loss and the logit distillation error between the weighted global logits and local logits.
Accordingly, we define the local prompt tuning error as $\mathbf{R}_{\hat{\mathbb{D}}_k}^{ce}(h_k)$ and the difference between the initial model $h_0$ with the frozen pre-trained parameters and the fine-tuned model $h_k$ as $kl(h_k||h_0)$. 
With the distillation loss (restricted by the bound $C_k$), consisting of cross-entropy loss $\ell^{ce}(\cdot,\cdot)$ and information entropy $I(\cdot)$, we present the multi-class generalization error bound for $h_k$ over $\mathcal{D}_T$ in Theorem~\ref{thm:generalization bound} and the proofs can deferred to Appendix~\ref{Appendix1}.
\begin{theorem}
\label{thm:generalization bound}
Suppose that K clients in FedHPL, let $d_{\mathcal{C}_h}(\cdot,\cdot)$ represents the distribution discrepancy between two data distributions.
Given any data distribution $\mathbb{D}_k$ and $\hat{\mathbb{D}}_k$ over client $k$ and the local model $h_k \in \mathcal{H}$ which fine-tuned from $h_0$ (the distribution is $\pi_k$).
Taking any $t>0$ and $\lambda_k=-\log \pi_k$, the generalization error bound for $h_k$ over $\mathcal{D}_T$ holds with the probability of at least $1-\epsilon$:
\begin{equation*}
\begin{aligned}
    &\mathbf{R}_{\mathbb{D}_T}(h_k)  \le \mathbf{R}_{\hat{\mathbb{D}}_k}^{ce}(h_k) + \sqrt{\frac{kl(h_k||h_0)+\ln\sqrt{4|{\mathcal{D}}_k|}-\ln \epsilon}{2|{\mathcal{D}}_k|}} + \lambda_{\mathbb{D}_k, \mathbb{D}_T}(h_0)+ d_{\mathcal{C}_h}((\mathbb{D}_k)_X,(\mathbb{D}_T)_X)\\
    &+\ell^{ce}(\phi_{\mathcal{T}}(\sum_{j=1}^{K}\tilde{\beta}_{k,j} h_j(X_{j})), \phi_{\mathcal{T}}(h_k(X_{k}))) - I(\phi_{\mathcal{T}}(\sum_{j=1}^{K}\tilde{\beta}_{k,j} h_j(X_j)))+\frac{\lambda_k-\log \epsilon}{t|\mathcal{D}_k|} + \frac{tC_k^2}{8},
\end{aligned}
\end{equation*}
where $\phi_{\mathcal{T}}$ is the softmax function with a temperature factor $\mathcal{T}$ and $\mathcal{C}_h=h\Delta\mathcal{H}$.
$\lambda_{\mathbb{D}_k,\mathbb{D}_T}(h_0)=\mathbf{R}_{\mathbb{D}_k}(h_0)+\mathbf{R}_{\mathbb{D}_T}(h_0)$ measures the adaptation error of $h_0$.
\end{theorem}

\textbf{Discussion.}
Based on this, we deduce the following:
(1) If the initial model $h_0$ (with frozen pre-trained parameters of $F_k$ and initial $\mathsf{P}_k$ and $H_k$) has a small error $\lambda_{\mathbb{D}_k,\mathbb{D}_T}(h_0)$, it is beneficial for improving generalization ability and model performance.
(2) More samples $|\mathcal{D}_k|$ can reduce generalization error and enhance the model utility.
(3) A local distribution that is more similar to the global distribution, along with more precise global logits can reduce the error and promote local training, as analyzed in the Remark~\ref{KDremark} of Appendix.
(4) A higher distillation loss bound $C_k$ and local tuning error $\mathbf{R}_{\mathbb{\hat{D}}_k}^{ce}(h_k)$ undermine the generalization ability and model effectiveness.
We can exploit the weight coefficient $\tilde{\beta}_{k,j}$ and study the global knowledge to reduce them which we later show in Figure~\ref{fig:ablation_theorem}.
(5) Obviously, less distribution discrepancy $d_{\mathcal{C}_h}$ can reduce the estimation error on $\mathcal{D}_T$.

\section{Experiments}

\subsection{Experimental settings}

\textbf{Datasets and models.}
We evaluate our method on CIFAR10, CIFAR100, and SVHN datasets under four settings shown in Figure~\ref{fig:tesaer}.
For \textit{IID} data homogeneity setting, we randomly shuffle and partition data samples into clients, while we employ Dirichlet distribution Dir($\alpha$) with random $\alpha$ to perform \textit{Dir} and \textit{Non-IID} data heterogeneity settings, where the former has overlapping samples and the latter does not.
For the model setting, we employ ViT-B/16~\cite{ViT} in the homogeneous model experiments while using different ViT~\cite{ViT} or ResNet~\cite{ResNet} as client backbones in the heterogeneous model experiments.
These foundation models are both pre-trained on ImageNet1k~\cite{ImageNet}.
The classification head is a linear fully connected layer.
For more details, see Appendix~\ref{appendix_model_dataset}.
Furthermore, we resize image pixels to 224 $\times$ 224 for aligning the experiment setting of pre-trained backbones.

\textbf{Implementation details.}
We train models by using the SGD optimizer with a learning rate of 0.01, a weight decay rate of 1e-4, and a momentum of 0.9.
We perform 10 global rounds in CIFAR10 and 15 rounds in CIFAR100 and SVHN over 5 clients, while the local epoch is 1.
The default style is VPT-deep with $n=3$ prompts for each backbone layer, and we fix $\mathcal{T}=4.5$ and $\gamma=1$.
For fair comparison, we run 100 global rounds on all baselines under the same model and data settings.

\textbf{Baselines.}
We compare FedHPL against advanced FL approaches over two model settings with various data settings.
For the \textit{homogeneous model} setting, we compare our framework with loss-based FL (FedAVG~\cite{FedAvg}, FedProx~\cite{FedProx}, SCAFFOLD~\cite{SCAFFOLD}), pFL (FedBABU~\cite{FedBABU}, FedRep~\cite{FedRep}), and VPT-based FL (pFedPT~\cite{pFedPT}, pFedPG~\cite{pFedPG}).
For the \textit{heterogeneous model} setting, we compare with model-based FL (FedGen~\cite{FedGEN}, FedGH~\cite{FedGH}), prototype-based FL (FedProto~\cite{FedProto}, FedTGP~\cite{FedTGP}), and distillation-based FL (FedMD~\cite{FedMD}, FedHE~\cite{FedHE}).
More details are in Appendix~\ref{Appendix2}.

\subsection{Performance comparison with SOTA approaches}

\textbf{FedHPL achieves excellent results in model homogeneity.}
We compare FedHPL with existing FL approaches in the \textit{homogeneous model} setting and the results are shown in Table~\ref{tab:homo_model}. 
It can be seen that FedHPL outperforms other algorithms across all settings with the least trainable parameters.
For example, the average test accuracy is basically 30\% and 45\% higher than other methods (except pFedPG) on CIFAR10 and CIFAR100.
We attribute such consistent outperformance to the strong representations of pre-trained backbones and effective guidance of global knowledge.
\begin{table}
\caption{The average test accuracy (\%) and the average trainable parameters over CIFAR10 dataset in the homogeneous model setting (use ViT-B/16 refer to Table~\ref{tab:model_hete_settings}). More details are in Appendix~\ref{Appendixparameterhomo}.}
\label{tab:homo_model}
\centering
\resizebox{\textwidth}{!}{
    \begin{tabular}{c|c|ccccccccc} 
    \hline
    \multirow{2}{*}{Method}  & \multirow{2}{*}{Param (M)}& \multicolumn{3}{c}{CIFAR10}    & \multicolumn{3}{c}{CIFAR100}    & \multicolumn{3}{c}{SVHN}  \\ 
    \cline{3-11}
      &                                               & IID& Dir    & Non-IID& IID& Dir  & Non-IID& IID& Dir  & Non-IID\\ 
    \hline
    \multicolumn{11}{l}{{\cellcolor[rgb]{0.902,0.902,0.902}}\textit{\textbf{Loss-based Federated Learning}}}  \\
    \multicolumn{1}{l|}{FedAVG~\cite{FedAvg}}         & 81.83   & 66.03  & 60.90  & 64.72  & 39.35  & 36.38 & 39.97 & 86.36  & 86.87  & 88.61  \\
    \multicolumn{1}{l|}{FedProx~\cite{FedProx}}       & 81.83   & 65.13  & 59.63  & 63.02  & 39.50  & 35.00 & 37.94 & 88.91  & 86.52  &   87.17\\
    \multicolumn{1}{l|}{SCAFFOLD~\cite{SCAFFOLD}}    & 81.83   & 67.79  & 62.23  & 67.87  & 40.65  & 35.89 & 40.05 & 89.61  & 87.10  & 88.34 \\ 
    \hline
    \multicolumn{11}{l}{{\cellcolor[rgb]{0.902,0.902,0.902}}\textit{\textbf{Personalized Federated Learning}}} \\
    \multicolumn{1}{l|}{FedBABU~\cite{FedBABU}}       & 81.82  & 63.27& 57.12& 60.01& 39.21& 35.28& 36.97&   87.45&   85.61&   85.16\\
    \multicolumn{1}{l|}{FedRep~\cite{FedRep}}         & 81.83  & 64.32& 58.93& 59.98& 39.92& 33.30& 36.82&   88.59&   86.80&   86.11\\ 
    \hline
    \multicolumn{11}{l}{{\cellcolor[rgb]{0.902,0.902,0.902}}\textit{\textbf{VPT-based Federated Learning}}} \\
    \multicolumn{1}{l|}{pFedPT~\cite{pFedPT}}         & 1.028& 58.60  & 52.42  & 54.10  & 28.54  & 26.40  & 27.15 & 83.28  &   79.95&   80.49\\
    \multicolumn{1}{l|}{pFedPG~\cite{pFedPG}}         & 1.526& 97.20  & 96.07  & 96.47  & 85.30  & 82.08  & 81.70 & 93.04 & 91.66& 90.21  \\
    \multicolumn{1}{l|}{\textbf{FedHPL}}                 & 0.034 & \textbf{97.89} & \textbf{96.37} & \textbf{96.60}  &   \textbf{89.68}&   \textbf{86.82}&   \textbf{86.23}&   \textbf{94.57}& \textbf{91.71}& \textbf{90.46}\\
    \hline
    \end{tabular}
}
\vspace{-1em}
\end{table}

\textbf{FedHPL can adapt model heterogeneity.}
We next compare the model performance with other FL methods in the \textit{heterogeneous model} setting. 
Notably, due to the prevalent implementation of CNN in these approaches, FedHPL adopts a mechanism of padding learnable height and width pixels as prompts to the image space and performs local training over CNN backbones instead of Transformer for fair comparison.
Table~\ref{tab:hete_model} illustrates the comparison results and it can be seen that FedHPL achieves the highest accuracy across CIFAR10 and CIFAR100 datasets.
For instance, in the \textit{Dir} data setting, the test accuracy in FedHPL is 35.99\% and 42.85\% higher than the lowest accuracy of baselines on CIFAR10 and CIFAR100, while in the \textit{IID} data setting, it is 44.21\% and 52.70\% higher, respectively.
It proves the feasibility and effectiveness of FedHPL in model heterogeneity. 

\begin{table}
\caption{Comparison of average test accuracy (\%) and the sum of trainable parameters over CIFAR10 dataset in the heterogeneous model setting (all methods, except for FedHPL (ViT), use the heterogeneous ResNet setting refer to Table~\ref{tab:model_hete_settings}). More details can refer to Appendix~\ref{Appendixparameterhete}. }
\label{tab:hete_model}
\centering
\resizebox{\textwidth}{!}{
    \begin{tabular}{c|c|ccccccccc} 
    \hline
    \multirow{2}{*}{Method} & \multirow{2}{*}{Param (M)} & \multicolumn{3}{c}{CIFAR10}    & \multicolumn{3}{c}{CIFAR100}    & \multicolumn{3}{c}{SVHN}  \\ 
    \cline{3-11}
    &      & IID& Dir   & Non-IID& IID& Dir   & Non-IID& IID& Dir   & Non-IID\\ 
    \hline
    \multicolumn{11}{l}{{\cellcolor[rgb]{0.902,0.902,0.902}}\textit{\textbf{Model-based Federated Learning}}}  \\
    \multicolumn{1}{l|}{FedGen~\cite{FedGEN}}     & 84.867& 39.45 & 37.39 & 38.74   & 12.20  & 12.91 & 13.15    & 73.45  & 69.30  & 60.84 \\
    \multicolumn{1}{l|}{FedGH~\cite{FedGH}}       & 86.325& 68.05 & 54.49 & 58.08   & 29.86  & 25.34 & 25.60    & 92.33  & 89.10  & 87.28 \\
    \hline
    \multicolumn{11}{l}{{\cellcolor[rgb]{0.902,0.902,0.902}}\textit{\textbf{Prototype-based Federated Learning}}}  \\
    \multicolumn{1}{l|}{FedProto~\cite{FedProto}} & 89.360 & 40.12 & 37.66 & 39.49   & 14.46  & 13.73 & 13.74    & 79.80  & 78.16 & 66.04 \\
    \multicolumn{1}{l|}{FedTGP~\cite{FedTGP}}     & 84.866& 38.84& 34.46 & 39.87   & 10.03  & 11.83 & 11.75    & 76.05  & 67.79 & 62.55 \\ 
    \hline
    \multicolumn{11}{l}{{\cellcolor[rgb]{0.902,0.902,0.902}}\textit{\textbf{Distillation-based Federated Learning}}}    \\
    \multicolumn{1}{l|}{FedMD~\cite{FedMD}}       & 84.375& 67.39 & 66.53 & 66.62  & 29.71  & 32.22 & 30.71    & 90.05  & 89.96 & \textbf{90.44}\\
    \multicolumn{1}{l|}{FedHE~\cite{FedHE}}       & 84.335& 67.06 & 55.32 & 56.48  & 29.87  & 26.84& 26.47    & 92.59  & 88.36 & 88.17 \\
    \multicolumn{1}{l|}{\textbf{FedHPL (CNN)}}              & 0.078& 83.05& 70.45& 76.53& 62.73& 54.68& 52.92&  72.43 & 63.28 & 59.78  \\
    \multicolumn{1}{l|}{\textbf{FedHPL (ViT)}}              & 0.181& \textbf{97.06} & \textbf{96.10} & \textbf{95.75} &  \textbf{88.84}&  \textbf{85.85}  & \textbf{85.57}&   \textbf{94.19}&  \textbf{90.75} & 90.02\\
    \hline
    \end{tabular}
}
\vspace{-1em}
\end{table}

\textbf{Select ViT as the pre-trained backbone.}
We also notice that FedHPL with ViT achieves competitive performance over most situations while the performance with CNN is not ideal on SVHN dataset.
For example, the test accuracy improves 21.76\%, 27.47\%, and 30.24\% in FedHPL from ResNet to ViT on SVHN dataset. 
Refer to~\cite{VPT}, the advantage of VPT can be better fulfilled with Transformers and diminish with smaller CNN.
We think $\lambda_{\mathbb{D}_k, \mathbb{D}_T}(h_0)$ in ViT is smaller than CNN and learnable parameters can be better trained over ViT backbones.
Therefore, it is preferred to choose ViT structures as pre-trained backbones.
See Appendix~\ref{Appendix_table_details} for more details about model performance.

\subsection{Ablation study and analysis}
According to Theorem~\ref{thm:generalization bound}, the model performance is affected by many factors.
To investigate their influence, we conduct the ablation study over the CIFAR10 dataset with ViT backbones.

\textbf{Effect of prompt length and insertion position.}
Figure~\ref{fig:ablation_theorem}(a) exhibits performance comparison results with VPT-shallow and VPT-deep over the varying prompt length for a backbone layer (\ie $n$) in the \textit{Dir} data and \textit{heterogeneous model} setting.
With tolerant trainable parameters (shown in Table~\ref{tab:parameter_vit2}), VPT-deep has better performance (especially for the lowest client accuracy with a maximum performance improvement of 6.6\%) by promoting the generalization ability of $h_0$ with efficient $\mathsf{P}_k$ and reducing $\lambda_{\mathbb{D}_k,\mathbb{D}_T}(h_0)$. 
Thus, we use VPT-deep as the default insertion style in FedHPL.

\begin{figure}[t]
\setlength{\belowcaptionskip}{-0.2cm} 
    \centering
    \subfloat[prompt length]{\includegraphics[width=.41\linewidth]{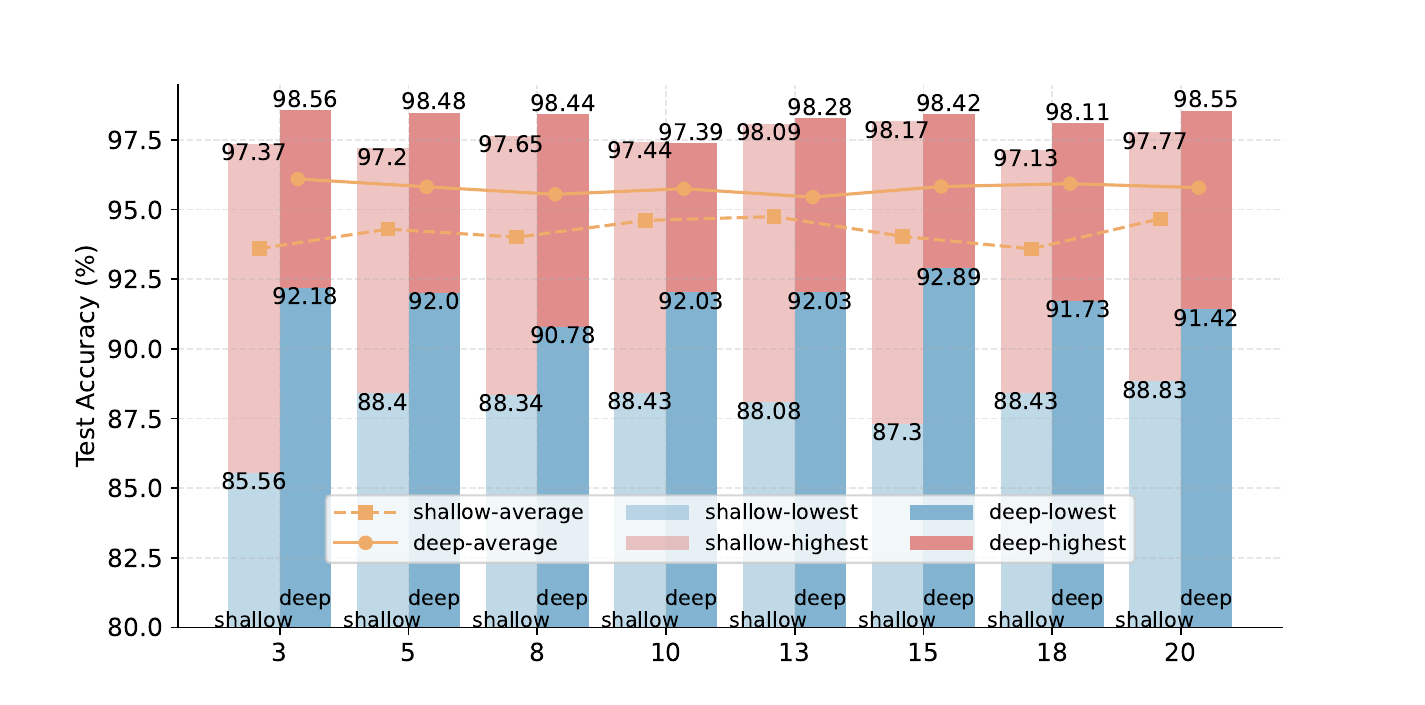}}
    \subfloat[sample percentage]{\includegraphics[width=.29\linewidth]{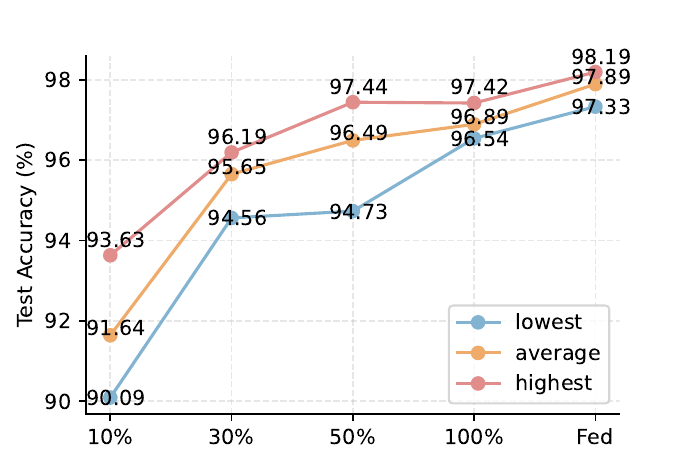}}
    \subfloat[component]{\includegraphics[width=.3\linewidth]{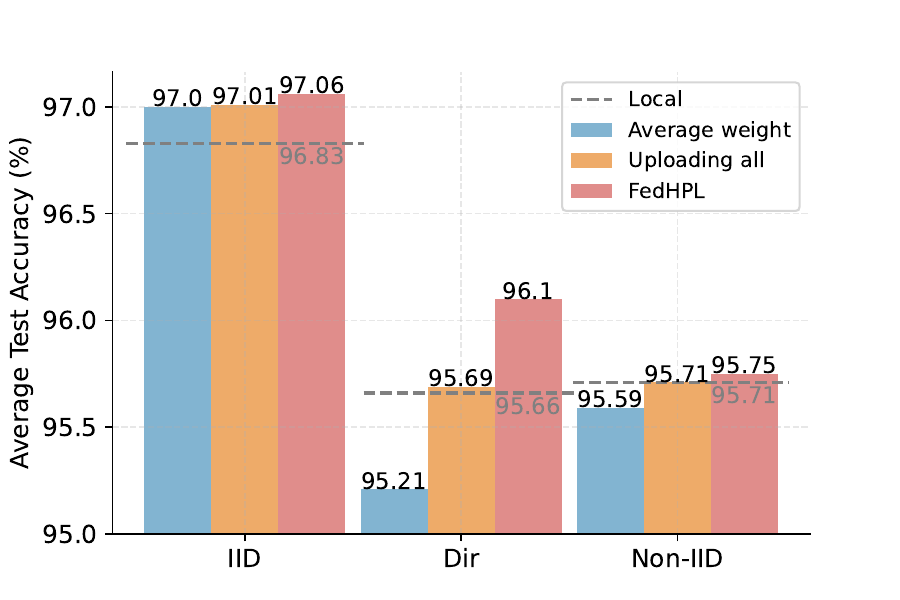}}
\caption{The ablation study on FedHPL over CIFAR10 dataset. Entire experimental settings, more comparison results, and more hyper-parameter studies are shown in Appendix~\ref{appendix_analysis_details}.}
\label{fig:ablation_theorem}
\end{figure}

\begin{table}[t]
\caption{Test accuracy (\%) on CIFAR10. We control `$\alpha$' in the Non-IID setting. `+$\mathsf{P}$' and `+${H}$' represent FedHPL with the aggregation of prompts and head parameters on the same dimension.}
\label{tab:analysis_component}
\centering
\subfloat{
\resizebox{0.49\textwidth}{!}{
    \begin{tabular}{c|ccc|ccc} 
    \hline
    \multirow{2}{*}{Setting} &\multicolumn{3}{c|}{Homogeneous Model}   & \multicolumn{3}{c}{Heterogeneous Model}  \\
    \cline{2-7}
                  & Lowest & Average & Highest  & Lowest & Average & Highest \\ 
    \hline         
    $\alpha$=0.1  &        84.94& 91.43  &           95.15&        90.24& 92.48  &   97.53\\
    $\alpha$=0.5  &        93.96& 96.13  &           97.48&        91.67& 95.51  &  97.30\\
    $\alpha$=1.0  &        96.65& 96.90  &           97.61&        94.25& 96.23  &  98.26\\
    IID           &        97.33& 97.89  &           98.19&        95.63& 97.06  &  98.40\\
    \hline
    \end{tabular}
}}
\subfloat{
\resizebox{0.49\textwidth}{!}{
    \begin{tabular}{c|ccc|ccc} 
    \hline
    \multirow{2}{*}{Policy} &\multicolumn{3}{c|}{Homogeneous Model}   & \multicolumn{3}{c}{Heterogeneous Model}         \\
    \cline{2-7}
     &IID      & Dir   & Non-IID    & IID   & Dir  & Non-IID   \\ 
    \hline
    Local &96.89   & 95.77  & 96.20      & 96.83  & 95.66  & 95.71 \\
    \rowcolor[rgb]{0.902,0.902,0.902}FedHPL &97.89\textcolor{cyan}{$\uparrow$} & 96.37\textcolor{cyan}{$\uparrow$} & 96.60{\textcolor{cyan}{$\uparrow$}} & 97.06\textcolor{cyan}{$\uparrow$} & 96.10\textcolor{cyan}{$\uparrow$} & 95.75\textcolor{cyan}{$\uparrow$}  \\
    +$\mathsf{P}$ &98.11\textcolor{cyan}{$\uparrow$} & 96.60\textcolor{cyan}{$\uparrow$} & 96.51\textcolor{cyan}{$\uparrow$} & 97.14\textcolor{cyan}{$\uparrow$} & 96.64\textcolor{cyan}{$\uparrow$} & 95.93\textcolor{cyan}{$\uparrow$}  \\
    +$H$   &97.84\textcolor{cyan}{$\uparrow$}& 96.67\textcolor{cyan}{$\uparrow$} &                                               95.81\textcolor{red}{$\downarrow$}& 97.10\textcolor{cyan}{$\uparrow$} & 96.25\textcolor{cyan}{$\uparrow$} & 95.84\textcolor{cyan}{$\uparrow$}  \\
    \hline
    \end{tabular}
}}
\vspace{-2em}
\end{table}

\textbf{Sensitivity to the number of involved training samples.}
Figure~\ref{fig:ablation_theorem}(b) shows the exploration results among clients over the \textit{IID} and \textit{homogeneous model} setting.
It provides compelling evidence that a higher number percentage of data involved in local training (\ie more samples $|\mathcal{D}_k|$) can improve the model performance, indicating a reduction in the error bound $\mathbf{R}_{\mathbb{D}_T}(h_k)$.
Meanwhile, the error can further decrease by collaborative learning to strengthen the local model $h_k$ and further alleviate the negative impact caused by insufficient local samples.
Furthermore, Table~\ref{tab:analysis_component}(left) shows that imbalanced data (the smaller $\alpha$, the higher data heterogeneity) only cause a slight performance degradation and variance, revealing the capacity of FedHPL to handle data heterogeneity. 

\textbf{Necessity of weighted aggregation.}
In addition to local prompt tuning, global logit distillation also influences the model utility.
We selectively remove specific components from FedHPL (\ie weighted aggregation $\to$ average aggregation, uploading correct logits $\to$ uploading all logits) in the \textit{heterogeneous model} setting to investigate their contributions.
From Figure~\ref{fig:ablation_theorem}(c), we can observe that the average aggregation mechanism causes an accuracy decline, highlighting the effectiveness of the weighted aggregation mechanism.
The weight factor can increase the logit proportion of similar models, further decreasing the distillation error.
In addition, uploading all logits yields slightly lower accuracy than uploading correct predictions.
We believe that incorrect predictions have a negative impact on global logits and detailed analyze at Remark~\ref{KDremark}.
Interestingly, the performance of \textbf{Local} (\ie only exploit local prompt tuning) is good enough.
This also explains why uploading incorrect predictions only has a slight negative effect, as the number of incorrect predictions is small.

\textbf{Exploiting homogeneous parameters.}
We next give insights on facilitating performance in FedHPL from the same latent dimension. 
Specifically, clients additionally upload prompts or head parameters, which the server then aggregates by the same dimension and transmits the aggregated knowledge to corresponding clients.
Table~\ref{tab:analysis_component}(right) reports that aggregating prompts or head parameters with FedHPL can further improve the model performance in most cases.
Additional global prompts and head parameters can help clients to learn more knowledge and decrease the distribution difference to reduce the error bound.
Furthermore, FedHPL can save communication costs by uploading per-class average local logits while maintaining comparable performance, as we show in Appendix~\ref{appendix_communication}.

\textbf{Broader impact and limitation.}
FedHPL uploads logits instead of private data and provides a privacy-preserving paradigm in machine learning.
Furthermore, the generalization ability to unseen datasets that have large domain shifts to clients has not been explored. 
We leave it for future work.
\section{Conclusion}
In this work, we propose a novel unified framework named FedHPL, designed to tackle heterogeneity issues in federated learning.
To handle data heterogeneity and accelerate training, we leverage pre-trained foundation models and local prompt tuning over limited local resources to fine-tune local models.
For collaborative training among heterogeneous models, we design a global logit distillation scheme with a weighted knowledge aggregation mechanism.
Furthermore, we derived a generalization error bound for FedHPL to show how prompt tuning and logit distillation affect the model performance.
Extensive experiments demonstrate the effectiveness of FedHPL across various settings compared with other baselines and validate our theoretical guarantee.

\bibliographystyle{plain}
\bibliography{reference}
\newpage
\appendix
\section{Proofs of Theorem 1}
\label{Appendix1}
\subsection{Generalization error bound over local data distribution}
In this subsection, we show how the generalization error $\mathbf{R}_{\mathbb{D}_T}(h_k)$ over the local model $h_k$ of client $k$ in Theorem~\ref{thm:generalization bound} is connected from the test dataset $\mathcal{D}_T$ to the local private dataset $\mathcal{D}_k$.
Then, we get an error bound over the local private dataset $\mathcal{D}_k$.

\subsubsection{Preliminaries}
We first introduce several definitions and existing lemmas. 
\begin{myDef}[Risk]
\label{risk_def}
Inspired by~\cite{PAC-bayesian}, given the input space $\mathcal{X}$ with finite label space $\mathcal{Y}$ more than two classes and a hypothesis (\ie model) $h \in \mathcal{H}$.
The generalization error (or risk) $\mathbf{R}_{\mathbb{D}}$: $\mathcal{Y}^{\mathcal{X}} \to [0,1]$ is defined on the distribution $\mathbb{D}$ over $\mathcal{X} \times \mathcal{Y}$:
\begin{equation}
    \mathbf{R}_{\mathbb{D}}(h) \overset{def}{=} \mathbf{Pr}_{\left(X, Y\right) \sim \mathbb{D}}\left(M(h(X)) \ne Y\right),
\end{equation}
where $M(h(X))$=$arg\max_{c \in Y} h(X)_{c}$.
According to~\cite{PerADA}, given the indicator function $\mathbf{1}[\cdot]$, we have:
\begin{equation}
    \mathbf{Pr}_{(X,Y) \sim \mathbb{D}}(M(h(X)) \ne Y) = \mathbb{E}_{(X,Y) \sim \mathbb{D}} \mathbf{1} \left[M(h(X)) \ne Y\right].
\end{equation}
 
\end{myDef}

\begin{myDef}[$h\Delta\mathcal{H}$-divergence~\cite{PAC-bayesian}]
    Given a class of hypothesis $\mathcal{H} \subseteq \mathcal{Y}^{\mathcal{X}}$, for $h \in \mathcal{H}$,
    \begin{equation}
        h\Delta\mathcal{H} \overset{def}{=} \{x \mapsto 1-\mathbf{1}_{\{h(x)\}}\{h'(x)\}|h' \in \mathcal{H}\},
    \end{equation}
    which is a model-dependent extension of $\mathcal{H}\Delta\mathcal{H} \overset{def}{=} \{x \mapsto 1-\mathbf{1}_{\{h(x)\}}\{h'(x)\}|(h,h') \in \mathcal{H}^2\}$.
\end{myDef}

\begin{myDef}[Supremum inequality about $\mathcal{C}_h$~\cite{PAC-bayesian}] 
\label{supremum_def}
Given any distributions $\mathbb{D}_1$ and $\mathbb{D}_2$ over $\mathcal{X} \times \mathcal{Y}$, for any choice of $\mathcal{C}_h=h\Delta\mathcal{H}$,
\begin{equation}
    \begin{split}
    \xi & = \left| \mathbb{E}_{X_1 \sim (\mathbb{D}_1)_X} \mathbf{1}[M(h(X_1)) \ne M(h'(X_1))] - \mathbb{E}_{X_2 \sim (\mathbb{D}_2)_X} \mathbf{1}[M(h(X_2)) \ne M(h'(X_2))] \right| \\
    & \le d_{\mathcal{C}_h}((\mathbb{D}_2)_X, (\mathbb{D}_1)_X),
    \end{split}
  \end{equation}
where $\mathbb{D}_{\mathcal{X}}$ is the $\mathcal{X}$-marginal of $\mathbb{D}$.
More details about $\mathcal{C}_h$ can refer to the Theorem 5 in~\cite{PAC-bayesian}.
\end{myDef}

\begin{lemma}[PAC-Bayesian Theorem~\cite{PAC-domain}]
\label{True-Empi}
    Let $\ell$ be a [0,1]-bounded loss function, given the risk $\mathbf{R}_{\mathbb{D}}(h)$ over the distribution $\mathbb{D}$, the empirical risk $\mathbf{R}_{\mathbb{\hat{D}}}(h)$ trained with the loss function $\ell$ over the empirical distribution $\mathbb{\hat{D}}$, and $h_0$ (\eg the initial model with pre-trained parameters) be a probability distribution over $\mathcal{H}$. For $\epsilon \in (0,1)$, 
    \begin{equation}
        \mathbf{R}_{\mathbb{D}}(h) \le \mathbf{R}_{\mathbb{\hat{D}}}(h) + \sqrt{\frac{kl(h||h_0)+\ln\sqrt{4m}-\ln \epsilon}{2m}},
    \end{equation}
    where $kl$ is the KL-divergence and $m$ is the number of samples.
\end{lemma}
\begin{proof}
    For a probability distribution $\mathbb{P}$ (\ie $h_0$) over $\mathcal{H}$ and $\delta > 0$, the above lemma is the loosen result of~\cite{PAC-note}, which is given below:
    \begin{equation*}
        \mathbf{Pr}(\forall \mathbb{Q} \subseteq \mathcal{H}: kl(\mathbf{R}_{\mathbb{\hat{D}}}(\mathbb{Q})||\mathbf{R}_{\mathbb{D}}(\mathbb{Q})) \le \frac{kl(\mathbb{Q}||\mathbb{P})+\ln\sqrt{4m}-\ln \delta}{m}) \ge 1-\delta,
    \end{equation*}
    by applying Pinsker's inequality. 
    Inspired by~\cite{PAC-VPT, PAC-domain}, we replace $\mathbb{Q}$ with $h$.
\end{proof}

\begin{lemma}[Adaptation of the triangle-inequality~\cite{PAC_ben, PAC-bayesian}]
\label{tri}
    For any $(h, h') \in \mathcal{H}^2$, we have the below inequalities by using the triangle inequality:
    \begin{equation}
    \label{tri1}
        \mathbf{1}[M(h(X)) \ne Y] \le \mathbf{1}[M(h(X)) \ne M(h'(X))] + \mathbf{1}[M(h'(X)) \ne Y],
    \end{equation}
    \begin{equation}
    \label{tri2}
        \mathbf{1}[M(h(X)) \ne M(h'(X))] \le \mathbf{1}[M(h(X)) \ne Y] + \mathbf{1}[Y \ne M(h'(X))].
    \end{equation}
\end{lemma}

\subsubsection{Bound with local data distribution}
Based on the above definitions and lemmas, we have the generalization error over $\mathbb{D}_T$ bounded with the local data distribution $\mathbb{D}_k$.

\begin{theorem}
\label{theorem1}
    With K clients and a server in FL, $\mathbb{\hat{D}}_k$ is the empirical distribution of client $k$ from the private training dataset which has $|\mathcal{D}_k|$ samples and $\mathbb{D}_k$ is the true distribution of client $k$.
    $\mathbb{D}_T$ is the data distribution from the test dataset and it is usually used to estimate the model performance.
    $h_0$ is the initial model with a pre-trained backbone, trainable prompts, and a classification head. $h_k$ is the fine-tuned model from $h_0$ after training.
    Given any $h_k \in \mathcal{H}$ for client $k$ and $\forall \delta \in (0,1)$, with probability at least $1-\delta$, the generalization error bound over the test dataset is:
    \begin{equation*}
        \mathbf{R}_{\mathbb{D}_T}(h_k) \le \mathbf{R}_{\mathbb{{D}}_k}(h_k) + \lambda_{\mathbb{D}_k, \mathbb{D}_T}(h') + d_{\mathcal{C}_h}((\mathbb{D}_k)_X,(\mathbb{D}_T)_X),
    \end{equation*}
    where $\lambda_{\mathbb{D}_k,\mathbb{D}_T}(h')=\mathbf{R}_{\mathbb{D}_k}(h')+\mathbf{R}_{\mathbb{D}_T}(h')$ for any $h' \in \mathcal{H}$ and $\mathcal{C}_h=h\Delta\mathcal{H}$.
\end{theorem}

\begin{proof}
Recall Eq. (\ref{tri1}) in Lemma~\ref{tri}, for any $h'$, we have the below inequality by monotonicity and linearity of expectation
    \begin{equation*}
    \begin{aligned}
      \mathbb{E}_{(X, Y) \sim \mathbb{D}_T} \mathbf{1}[M(h(X)) \ne Y] & \le \mathbb{E}_{X \sim (\mathbb{D}_T)_X} \mathbf{1}[M(h(X)) \ne M(h'(X))] \\
      & + \mathbb{E}_{(X, Y) \sim \mathbb{D}_T} \mathbf{1}[M(h'(X)) \ne Y].
    \end{aligned}
    \end{equation*}
  Then, applying Definition~\ref{risk_def} and Definition~\ref{supremum_def}, we have:
    \begin{equation}
  \label{risk1}
        \begin{split}
            \mathbf{R}_{\mathbb{D}_T}(h) &= \mathbb{E}_{(X,Y) \sim \mathbb{D}_T} \mathbf{1}[M(h(X)) \ne Y] \\
            & \le \mathbb{E}_{X \sim (\mathbb{D}_T)_X} \mathbf{1}[M(h(X)) \ne M(h'(X))] + \mathbb{E}_{(X, Y) \sim \mathbb{D}_T} \mathbf{1}[M(h'(X)) \ne Y] \\
            & = \mathbb{E}_{X \sim (\mathbb{D}_T)_X} \mathbf{1}[M(h(X)) \ne M(h'(X))] + \mathbf{R}_{\mathbb{D}_T}(h') \\
            & \le \mathbb{E}_{X_k \sim (\mathbb{D}_k)_X} \mathbf{1}[M(h(X_k)) \ne M(h'(X_k))] + \mathbf{R}_{\mathbb{D}_T}(h') \\
            & + |\mathbb{E}_{X \sim (\mathbb{D}_T)_X} \mathbf{1}[M(h(X)) \ne M(h'(X))] - \mathbb{E}_{X_k \sim (\mathbb{D}_k)_X} \mathbf{1}[M(h(X_k)) \ne M(h'(X_k))]| \\
            & = \mathbb{E}_{X_k \sim (\mathbb{D}_k)_X} \mathbf{1}[M(h(X_k)) \ne M(h'(X_k))] + \mathbf{R}_{\mathbb{D}_T}(h') + d_{\mathcal{C}_h}((\mathbb{D}_k)_X,(\mathbb{D}_T)_X),
        \end{split}
    \end{equation}
    where $d_{\mathcal{C}_h}[\cdot,\cdot]$ represents the distribution discrepancy between two distributions.

Similarly, recall Eq. (\ref{tri2}) in Lemma~\ref{tri}, we have the below inequality with $(X_k, Y_k) \sim \mathbb{D}_k$:
\begin{equation}
\label{risk2}
\begin{split}
    \mathbb{E}_{X_k \sim (\mathbb{D}_k)_X} \mathbf{1}[M(h(X_k)) \ne M(h'(X_k))] & \le \mathbb{E}_{X_k \sim (\mathbb{D}_k)_X} \mathbf{1}[M(h(X_k)) \ne Y_k] \\
    & + \mathbb{E}_{X_k \sim (\mathbb{D}_k)_X} \mathbf{1}[Y_k \ne M(h'(X_k))] \\  
    &= \mathbf{R}_{\mathbb{D}_k}(h) + \mathbf{R}_{\mathbb{D}_k}(h').\\
\end{split}
\end{equation}
By combining Eq. (\ref{risk1}) and Eq. (\ref{risk2}), for any $h \in \mathcal{H}$,
\begin{equation*}
\begin{split}
    \mathbf{R}_{\mathbb{D}_T}(h) & \le \mathbf{R}_{\mathbb{D}_k}(h) + \mathbf{R}_{\mathbb{D}_k}(h') + \mathbf{R}_{\mathbb{D}_T}(h') + d_{\mathcal{C}_h}((\mathbb{D}_k)_X,(\mathbb{D}_T)_X).
\end{split}
\end{equation*}
Moreover, $h$ in FedHPL is the local fine-tuned model $h_k$, so we have:
\begin{equation}
\label{risk3}
    \mathbf{R}_{\mathbb{D}_T}(h_k) \le \mathbf{R}_{\mathbb{D}_k}(h_k) + \mathbf{R}_{\mathbb{D}_k}(h') + \mathbf{R}_{\mathbb{D}_T}(h') + d_{\mathcal{C}_h}((\mathbb{D}_k)_X,(\mathbb{D}_T)_X).
\end{equation}
\end{proof}
Furthermore, because $\lambda_{\mathbb{D}_k, \mathbb{D}_T}(h')$ is suitable for any model $h' \in \mathcal{H}$, Eq. (\ref{risk3}) can further derived when using $h_0$ to replace $h'$:
\begin{equation}
\label{risk4}
\begin{split}
    \mathbf{R}_{\mathbb{D}_T}(h_k) & \le \mathbf{R}_{\mathbb{D}_k}(h_k) + \mathbf{R}_{\mathbb{D}_k}(h_0) + \mathbf{R}_{\mathbb{D}_T}(h_0)+ d_{\mathcal{C}_h}((\mathbb{D}_k)_X,(\mathbb{D}_T)_X).
\end{split}
\end{equation}

\begin{remark}
    $h_k$ is fine-tuned from $h_0$ by the local prompt tuning and global logit distillation with a weighted knowledge aggregation mechanism.
    $\mathbf{R}_{\mathbb{D}_k}(h_0) + \mathbf{R}_{\mathbb{D}_T}(h_0)$ represents the model adaptation error of $h_0$ after transferring it from the source domain (\eg ImageNet) to the new domain (\eg CIFAR10, CIFAR100, SVHN).
    We consider $\lambda_{\mathbb{D}_k, \mathbb{D}_T}(h_0)$ reflects the generalization ability of $h_0$ (the initial model with pre-trained backbone and initial trainable parameters).
    The more robust the pre-trained backbone, the stronger the generalization ability in $h_0$ will be and the corresponding error $\lambda_{\mathbb{D}_k, \mathbb{D}_T}(h_0)$ will be reduced.
    Additionally, the trainable parameters (\ie visual prompts and the classification head) also have an impact on model generalization on $h_0$.
    Moreover, the distribution discrepancy $d_{\mathcal{C}_h}(\cdot,\cdot)$ between two data distributions also influences the generalization error.
\end{remark}

In FedHPL, $\mathbf{R}_{\mathbb{D}_k}(h_k)$ consists of the error $\mathbf{R}_{\mathbb{D}_k}^{ce}(h_k)$ in local prompt tuning and distillation error $\mathbf{R}_{\mathbb{D}_k}^{kd}(h_k)$ in global logit distillation with a weighted knowledge aggregation mechanism. 
According to Lemma~\ref{True-Empi}, we can simply use $\mathbf{R}_{\mathbb{\hat{D}}_k}^{ce}(h_k) + \sqrt{\frac{kl(h_k||h_0)+\ln\sqrt{4|\mathcal{D}_k|}-\ln \epsilon}{2|\mathcal{D}_k|}}$ to represent $\mathbf{R}_{\mathbb{D}_k}^{ce}(h_k)$ because the true data distribution is hard to estimate and we usually use the empirical data to estimate.
The trainable parameters can affect the error bound from the local training and $kl(h_k||h_0)$.
According to~\cite{bound_loss}, we use a cross-entropy loss function to train $h_k$ in the stage of local prompt tuning instead of $M(h_k(X_k))=arg\max_{c \in Y_k} h_k(X_k)_{c}$ since it is hard to optimize.
However, the distillation error with a KD loss $\ell^{kd}$ is more complicated and related to logits.
Furthermore, the KL loss is not mentioned in~\cite{bound_loss}, so we decided to analyze it from another perspective.

\subsection{Generalization error bound over global logit distillation}

In this subsection, we apply a simple PAC-Bayes bound for estimating the effectiveness of global logit distillation with a weighted knowledge distillation mechanism in federated learning and observe its influence on the generalization error.
Due to the fact that KL divergence does not necessarily satisfy the Lipschitz condition, traditional bound analysis cannot be suitable in our setting. 
Inspired by~\cite{PAC_loss_sum}, this paper defines a generalized bound based on the loss function.

\subsubsection{Preliminaries}
Here, we show some definitions and lemmas before demonstrating our theorem.
\begin{myDef}
\label{KTdef1}
    Given a measurable function $\ell$: $\mathcal{Y}^2 \to [0,\infty)$ with $\ell(y,y)=0$. Assume that $0 \le \ell \le C$, the generalization error of a hypothesis $h$ is:
    \begin{equation}
        \mathbf{R}_{\mathbb{D}}(h) = \mathbb{E}_{(X,Y) \sim \mathbb{D}}[\ell(h(X),Y)],
    \end{equation}
    and the empirical risk is:
    \begin{equation}
        \mathbf{R}_{\mathbb{\hat{D}}}(h) = \frac{1}{m} \sum_{i=1}^{m} \ell(h(x^i),y^i),
    \end{equation}
    which satisfies $\mathbb{E}_{\mathcal{D}} [\mathbf{R}_{\mathbb{\hat{D}}}(h)]$$=$$\mathbf{R}_{\mathbb{D}}(h)$, where the training dataset $\mathcal{D}=\{(x^i,y^i)\}_{i=1}^{m}$ and ${(x^i,y^i) \sim \mathbb{\hat{D}}}$.
\end{myDef}

\begin{myDef}
\label{KTdef2}
    A hypothesis $h$ is a function (\eg model), that associates the parameters $\theta$. Let $\mathcal{P}(\Theta)$ be the set of all probability distributions on the parameter set $\Theta$.  The probability measure ${\theta}$ depends on data samples with any possible dataset whose size is $m$ and $\theta \sim {\rho}$, where
   \begin{equation*}
    {\rho}: \bigcup_{i=1}^{n} (\mathcal{X}\times\mathcal{Y})^i \to \mathcal{P}(\Theta).
  \end{equation*}
\end{myDef}
 
\begin{lemma}[Hoeffding's Lemma]
\label{hoeffding}
        Suppose $U$ is a random independent variable valuing from an interval $[a, b]$, for any $t\in \mathbb{R}^{+}$,
        \begin{equation}
        \mathbb{E}\left[e^{t(U-\mathbb{E}[U])}\right] \le e^{\frac{(b-a)^2t^2}{8}}.
    \end{equation}
\end{lemma}
\begin{lemma}[Cramer-Chernoff Basis~\cite{PAC_oxford}]
\label{CCbasis}
    For any $t \in \mathbb{R}^{+}$ and any independent random variable $U^i$,
    \begin{equation}
      \mathbb{E}\left[ e^{t\sum_{i=1}^{m}(U^i-\mathbb{E}[U^i])}\right] = \prod_{i=1}^{m} \mathbb{E}\left[e^{t(U^i-\mathbb{E}[U^i])}\right].
   \end{equation}
\end{lemma}

\begin{lemma}[Donsker-Varadhan variational formula]
\label{DVvformula}
    For any measurable, bounded function $h:\Theta \to \mathbb{R}$, fix the prior probability measure $\pi \in \mathcal{P}(\Theta)$ with the parameter $\theta$ of $h$, 
    \begin{equation}
      \log \mathbb{E}_{\theta \sim \pi}[e^{h(\theta)}] = \underset{\rho \in \mathcal{P}(\Theta)}{\sup}\{\mathbb{E}_{\theta \sim \rho}[h(\theta)]-kl(\rho||\pi)\}.   
   \end{equation}
\end{lemma}

\begin{lemma}[Chernoff bound]
\label{chernoff}
    Given a random variable $U$ and $s \in \mathbb{R}$, for any $a>0$,
    \begin{equation*}
  \begin{split}
        \mathbf{Pr}(U>s) & = \mathbf{Pr}(e^{aU}>e^{as}) \\
        & \le \frac{\mathbb{E}(e^{aU})}{e^{as}},
  \end{split}
  \end{equation*}
  which is the exponential version of Markov inequality.
\end{lemma}

\begin{lemma}[Logit Inequality~\cite{KDbound}]
\label{softmax_bound}
   Given $(x^i, y^i)$ in the dataset $\mathcal{D}, (x^i, y^i) \sim \mathbb{\hat{D}}$, and $ h(x^i) \in \mathbb{R}^{n_c}, y^i \in \{1, \cdots, n_c\}$, 
   \begin{equation}
    \mathbf{1}\left[arg\max_{c'}h(x^i)_{c'} \ne y^i \right] \le 2(1-\phi_{\mathcal{T}}(h(x^i))_{y^i}),
  \end{equation}
 where $\phi$ is a softmax function with its temperature factor $\mathcal{T}$ and $(\cdot)_{c'}$ is the component value of $h(x^i)$ at the $c'$-th label.
According to~\cite{PerADA}, we further have:
  \begin{equation*}
 \begin{split}
      \mathbf{Pr}_{\mathbb{\hat{D}}}\left(\mathbf{1}\left[arg\max_{c'}h(x^i)_{c'} \ne y^i \right]\right) & \le \mathbb{E}_{\mathbb{\hat{D}}} \left[2(1-\phi_{\mathcal{T}}(h(x^i))_{y^i})\right] \\
      & = 2 - 2\mathbb{E}_{\mathbb{\hat{D}}} \left[\phi_{\mathcal{T}}(h(x^i))_{y^i} \right].
 \end{split}
 \end{equation*}
\end{lemma}
\begin{proof}
    If $arg\max_{c'}h(x^i)_{c'}=y^i$,  which means $\mathbf{1}\left[arg\max_{c'}h(x^i)_{c'} \ne y^i \right]=0$.  Due to $\phi_{\mathcal{T}}(h(x^i))_{c'} \in [0,1]$, we have $2(1-\phi_{\mathcal{T}}(h(x^i))_{y^i} \ge 0$. So, the lemma holds in this case.
    If $arg\max_{c'}h(x^i)_{c'} \ne y^i$, which means $\mathbf{1}\left[arg\max_{c'}h(x^i)_{c'} \ne y^i \right]=1$. 
    Suppose $arg\max_{c'}h(x^i)_{c'}=\hat{c}$, we have:
    \begin{align*}
    \phi_{\mathcal{T}}(h(x^i))_{y^i} & = \frac{exp({h(x^i)_{y^i}}/\mathcal{T})}{\sum_{a=1}^{n_c}exp({h(x^i)_{a}}/\mathcal{T})} \\
    & \le \frac{exp({h(x^i)_{y^i}}/\mathcal{T})}{exp({h(x^i)_{y^i}}/\mathcal{T}) + exp({h(x^i)_{\hat{c}}}/\mathcal{T})} \\
    & \le \frac{exp({h(x^i)_{y^i}}/\mathcal{T})}{exp({h(x^i)_{y^i}}/\mathcal{T}) + exp({h(x^i)_{y^i}}/\mathcal{T})} \\
    & = \frac{1}{2},
  \end{align*}
  because the smaller the denominator, the greater the value.
  Thus, $2(1-\phi_{\mathcal{T}}(h(x^i)_{y^i}) \ge 1 = \mathbf{1}\left[arg\max_{c'}h(x^i)_{c'} \ne y^i \right]$. 
  In conclusion, the logit inequality is holds.
\end{proof}

\subsubsection{Bound with local risks}
Before we prove the generalization error bound over logit distillation loss, we first derive the bound with risks over the local dataset and model.
Here, we modify Lemma~\ref{hoeffding} for easier proving.
\begin{corollary}[Hoeffding’s Inequality over Local Risks]
\label{PAC_Hoeffding}
    Given the empirical distribution $\mathbb{\hat{D}}$ from samples $\mathcal{D} = [(x^1, y^1), \cdots, (x^m, y^m)]$ and true distribution $\mathbb{D}$, for any $t \in \mathbb{R}^{+}$,
    \begin{equation}
     \mathbb{E}_{\mathcal{D}}[e^{tm(\mathbf{R}_{\mathbb{D}}(h)-\mathbf{R}_{\mathbb{\hat{D}}}(h))}] \le e^{\frac{mC^2t^2}{8}}.
  \end{equation}
\end{corollary}

\begin{proof}
    Combine Lemma~\ref{CCbasis} with Lemma~\ref{hoeffding}, suppose $\forall U^i\in [a^i,b^i]$, we have:
    \begin{equation*}
  \begin{split}
    \mathbb{E}\left[ e^{t\sum_{i=1}^{m}(U^i-\mathbb{E}[U^i])}\right] & = \prod_{i=1}^{m} \mathbb{E}\left[e^{t(U^i-\mathbb{E}[U^i])}\right] \\
    & \le \prod_{i=1}^{m} e^{\frac{(b^i-a^i)^2t^2}{8}}.
  \end{split}
  \end{equation*}
  Suppose  $\forall a^i, b^i$ satisfies $b^i-a^i \le b-a$, we get:
  \begin{equation*}
  \begin{split}
    \mathbb{E}\left[ e^{t\sum_{i=1}^{m}(U^i-\mathbb{E}[U^i])}\right] & \le \prod_{i=1}^{m} e^{\frac{(b^i-a^i)^2t^2}{8}} \\
    & = e^{\sum_{i=1}^{m}\frac{(b^i-a^i)^2t^2}{8}} \\
    & \le e^{\frac{m(b-a)^2t^2}{8}}.
  \end{split}
  \end{equation*}
 
 With $U^i=\mathbb{E}_{\mathcal{D}}[\ell(h(x^i),y^i)]-\ell(h(x^i),y^i), \mathcal{D}=\{(x^i, y^i)\}_{i=1}^{m} \sim \mathbb{\hat{D}}$ and Definition~\ref{KTdef1}, we derive that:
  \begin{align*}
    \sum_{i=1}^m(U^i-\mathbb{E}[U^i]) & = \sum_{i=1}^m (\mathbb{E}_{\mathcal{D}}[\ell(h(x^i),y^i)]-\ell(h(x^i),y^i) - \mathbb{E}[\mathbb{E}_{\mathcal{D}}[\ell(h(x^i),y^i)] -\ell(h(x^i),y^i)] ) \\
    & = \sum_{i=1}^m (\mathbb{E}_{\mathcal{D}}[\ell(h(x^i),y^i)]-\ell(h(x^i),y^i) - \mathbb{E}_{\mathcal{D}}[\ell(h(x^i),y^i)] + \mathbb{E}[\ell(h(x^i),y^i)]) \\
    & = \sum_{i=1}^m \mathbb{E}[\ell(h(x^i),y^i)]-\sum_{i=1}^m \ell(h(x^i),y^i) \\
    & = m[\mathbf{R}_{\mathbb{D}}(h)-\mathbf{R}_{\mathbb{\hat{D}}}(h)].
 \end{align*}
 According Definition~\ref{KTdef1}, $\ell \in [0,C]$, we consider $b-a=C$. 
 Apply the above equality to $\mathbb{E}\left[ e^{t\sum_{i=1}^{m}(U^i-\mathbb{E}[U^i])}\right]$ and consider the dataset ${\mathcal{D}}$, the proof ends.
\end{proof}

Then, following Corollary~\ref{PAC_Hoeffding} and Lemma~\ref{DVvformula}, we have the below theorem to connect the generalization bound with the risks over the local dataset and model.
\begin{theorem}
    For any $t \in \mathbb{R}^{+}$ and $\epsilon \in (0,1)$, 
    \begin{equation}
    \mathbf{Pr}\left[\forall \rho \in \mathcal{P}(\Theta), \mathbb{E}_{\theta \sim \rho}[\mathbf{R}_{\mathbb{D}}(h)] \le \mathbb{E}_{\theta \sim\rho}[\mathbf{R}_{\mathbb{\hat{D}}}(h)]+\frac{kl(\rho||\pi)-\log \epsilon}{tm} +\frac{t C^2}{8}\right] \ge 1-\epsilon.
  \end{equation}
\end{theorem}
\begin{remark}
  The parameter $\theta$ of $h$ is the full parameter, instead of the parameter of a classification head in the main text.
\end{remark}

\begin{proof}
    According to~\cite{PAC_loss_sum}, given the hypothesis $h$ (\ie the fine-tuned model from $h_0$) with its parameter $\theta \in \Theta$. Considering the influence of the prior model $h_0$ with its fixed distribution $\pi$, we integrate it into Corollary~\ref{PAC_Hoeffding} and exchange the integration and sample expectation by Fubini:
\begin{equation*}
   \mathbb{E}_{\mathcal{D}}\mathbb{E}_{\theta \sim \pi}[e^{tm[\mathbf{R}_{\mathbb{D}}(h)-\mathbf{R}_{\mathbb{\hat{D}}}(h)]}] \le e^{\frac{mC^2t^2}{8}}.   
\end{equation*}
 Thanks to Lemma~\ref{DVvformula}, the above bound can measure from prior model distribution $\pi$ to any model distribution $\rho$:
 \begin{equation*}
     \mathbb{E}_{\mathcal{D}}\left[e^{\underset{\rho \sim \mathcal{P}(\Theta)}{\sup} \{tm\mathbb{E}_{\theta \sim \rho}[\mathbf{R}_{\mathbb{D}}(h)-\mathbf{R}_{\mathbb{\hat{D}}}(h)]-kl(\rho||\pi)\}}\right] \le e^{\frac{mt^2C^2}{8}}.
\end{equation*}
 Then, divide both sides of the equation by $e^{\frac{mt^2C^2}{8}}$, it gets:
\begin{equation*}
     \mathbb{E}_{\mathcal{D}} \left[e^{\underset{\rho \sim \mathcal{P}(\Theta)}{\sup} \{tm\mathbb{E}_{\theta \sim \rho}[\mathbf{R}_{\mathbb{D}}(h)-\mathbf{R}_{\mathbb{\hat{D}}}(h)]-kl(\rho||\pi)\}-\frac{mt^2C^2}{8}}\right] \le 1.
\end{equation*}
Using Lemma~\ref{chernoff} with its $a=1$, fix $s>0$ and $\rho \in \mathcal{P}(\Theta)$,
\begin{equation*}
\begin{split}
    \mathbf{Pr} &\left[ \underset{\rho \sim \mathcal{P}(\Theta)}{\sup} \{ tm\mathbb{E}_{\theta \sim \rho} [\mathbf{R}_{\mathbb{D}}(h)-\mathbf{R}_{\mathbb{\hat{D}}}(h)]-kl(\rho||\pi)\}-\frac{mt^2C^2}{8} > s\right] \\
    & \le \mathbb{E}_{\mathcal{D}} \left[ e^{\underset{\rho \sim \mathcal{P}(\Theta)}{\sup} \{tm\mathbb{E}_{\theta \sim \rho}[\mathbf{R}_{\mathbb{D}}(h)-\mathbf{R}_{\mathbb{\hat{D}}}(h)]-kl(\rho||\pi)\}-\frac{mt^2C^2}{8}}\right]e^{-s} \\
    & \le e^{-s}.
\end{split}
\end{equation*}
So, $\exists \rho \in \mathcal{P}(\Theta)$,  let $\epsilon=e^{-s}$ to get:
\begin{equation*}
    \mathbf{Pr}\left[ \mathbb{E}_{\theta \sim \rho}[\mathbf{R}_{\mathbb{D}}(h)] > \mathbb{E}_{\theta \sim \rho} [\mathbf{R}_{\mathbb{\hat{D}}}(h)]+ \frac{kl(\rho||\pi)-\log \epsilon}{tm} + \frac{tC^2}{8}\right] \le \epsilon.
\end{equation*}
Thus, with probability at least $1-\epsilon$, for $\forall \rho \in \mathcal{P}(\Theta)$ and any hypothesis $h$ with parameter $\theta$,
\begin{equation*}
    \mathbb{E}_{\theta \sim \rho}[\mathbf{R}_{\mathbb{D}}(h)] \le \mathbb{E}_{\theta \sim \rho}[\mathbf{R}_{\mathbb{\hat{D}}}(h)] + \frac{kl(\rho||\pi)-\log \epsilon}{tm} + \frac{tC^2}{8}.
\end{equation*}
\end{proof}

Then, we consider the specific probability $\rho$ in the set of $\mathcal{P}(\Theta)$ according to~\cite{PAC_loss_sum}, 
\begin{equation}
    \forall \theta \in \Theta, \mathbf{R}_{\mathbb{D}}(h) \le \mathbf{R}_{\mathbb{\hat{D}}}(h) + \frac{-{\log \pi}-\log\epsilon}{tm} + \frac{tC^2}{8}.
\label{eq:boundwithrisks}
\end{equation}

\subsubsection{Bound with logit distillation loss}
After integrating the generalization error bound with the local risk, we further extend it to logit distillation loss.
Fix the prior probability $\pi$ of Eq. (\ref{eq:boundwithrisks}) and apply Definition~\ref{KTdef1}, it gives:
\begin{equation}
    \mathbf{R}_{\mathbb{D}}(h) \le \frac{1}{m}\sum_{i=1}^{m}\ell (h(x^i),y^i)+\frac{\lambda-\log \epsilon}{tm} + \frac{tC^2}{8},
\label{eq:lossbound_kd1}
\end{equation}
which $\lambda={-\log \pi}$. 
Notably, the above analysis are based on the independent variables. 
In FedHPL, the private samples among clients is independent in the \textit{IID} and \textit{Non-IID} data settings. 
In \textit{Dir} data setting, all local data are independent while not necessarily independent across clients.
However, the above analysis only needs variables in $\mathbf{R}_{\mathbb{\hat{D}}}(h)$ to be independent, that is to require samples in the local dataset $\mathcal{D}$ are independent.
So, the bound holds in all settings of FedHPL.
At the same time, the knowledge among clients is transmitted in the form of logits which are different even using the same image as the input because of different model parameters.

Now, we consider the generalization error over the KD loss $\ell^{kd}$ for client $k$ in Eq. (\ref{eq:global_objective}).
$\ell$ can be seen as the KD loss with its bound $C_k$ and mensurability.
Since clients only upload the correctly predicted logits, the weighted global client-specific logits for each class can be seen as the label space $\mathcal{Y}$.
For representing the global logits with $h_k$, we integrate the global per-class logit for client $k$ with local models and a weighted knowledge aggregation mechanism, then denoted as:
\begin{equation}
  \tilde{p}_{k,c}=\sum_{j=1}^{K}\tilde{\beta}_{k,j}\sum_{i=1}^{|\mathcal{\tilde{D}}_{j,c}|}p_{j,c}^{i}
  =\sum_{j=1}^{K} \tilde{\beta}_{k,j}\sum_{i=1}^{|\mathcal{\tilde{D}}_{j,c}|}h_j(x_{j,c}^i),
\label{eq:global_logits_kd}
\end{equation}
for all client $j \in \{1, \cdots, K\}$, where $h(x_{j,c}^i)$ represent the logit of label $c$ in client $j$.
Specially, we use $h_j(X_{j,c})$ to replace $\sum_{i=1}^{|\mathcal{\tilde{D}}_{j,c}|}h_j(x_{j,c}^i)$ for writing conciseness.

Then, we turn our attention to the loss function $\ell$ over KD loss.
With the definition of cross-entropy loss $\ell^{ce}(p,q)=-\sum_{c'=1}^{n_c}p_{c'}\log q_{c'}$, information entropy $I(p)=-\sum_{c'=1}^{n_c}p_{c'}\log p_{c'}$, and $KL(p||q) = \sum_{c'=1}^{n_c}p_{c'}\log \frac{p_{c'}}{q_{c'}}$, client $k$ computes the distillation distance (\ie KD loss $\ell^{kd}$) for each local sample $x_k^i$ with its local model $h_k$ (\ie a backbone $F_k$ with a classification model $H_k$ and trainable prompts $\mathsf{P}_k$) by:
\begin{equation}
\begin{split}
    \ell^{kd}(p||q) & = KL(\phi_\mathcal{T}(\sum_{j=1}^{K}\tilde{\beta}_{k,j} h_j(X_{j,c}))||\phi_\mathcal{T}(h_k(x_k^i))) \\
& = \sum_{c'=1}^{n_c} \left[\phi_{\mathcal{T}}(\sum_{j=1}^{K}\tilde{\beta}_{k,j} h_j(X_{j,c}))_{c'}\log\frac{\phi_{\mathcal{T}}(\sum_{j=1}^{K}\tilde{\beta}_{k,j} h_j(X_{j,c}))_{c'}}{\phi_{\mathcal{T}}(h_k(x_k^i))_{c'}}\right] \\
& = \sum_{c'=1}^{n_c} \left[\phi_{\mathcal{T}}(\sum_{j=1}^{K}\tilde{\beta}_{k,j} h_j(X_{j,c}))_{c'} \log (\phi_{\mathcal{T}}(\sum_{j=1}^{K}\tilde{\beta}_{k,j} h_j(X_{j,c}))_{c'})\right] \\
& - \sum_{c'=1}^{n_c} \left[\phi_{\mathcal{T}}(\sum_{j=1}^{K}\tilde{\beta}_{k,j} h_j(X_{j,c}))_{c'}\log (\phi_{\mathcal{T}}(h_k(x_k^i))_{c'})\right] \\
& = \ell^{ce}(\phi_{\mathcal{T}}(\sum_{j=1}^{K}\tilde{\beta}_{k,j} h_j(X_{j,c})), \phi_{\mathcal{T}}(h_k(x_k^i))) - I(\phi_{\mathcal{T}}(\sum_{j=1}^{K}\tilde{\beta}_{k,j} h_j(X_{j,c}))).
\end{split}
\label{CE_KD_eq1}
\end{equation}
Notably, $c'$ is the c-th component value in the local logit $p_k^i$ or global logit, whereas $c$ represents the category of image $x_k^i$ (\ie $c=y_k^i$) and $n_c$ is the number of labels.

Moreover, we use Lemma~\ref{softmax_bound} with the global logits for client $k$ to further analyze the error bound:
\begin{equation}
\begin{aligned}
    \mathbf{Pr}_{\mathbb{\hat{D}}}\left(\mathbf{1}\left[arg\max_{c'}h_k(x_k^i)_{c'} \ne y_k^i \right]\right) & \le 2- 2\mathbb{E}_{\mathbb{\hat{D}}} \left[\phi_{\mathcal{T}}(h_k(x_k^i))_{y_k^i} \right]\\
    & = 2- 2\mathbb{E}_{\mathbb{\hat{D}}} \left[\phi_{\mathcal{T}}(h_k(x_k^i))_{y_k^i} \right] \\
    & + 2\mathbb{E}_{\mathbb{\hat{D}}} \left[\phi_{\mathcal{T}}(\sum_{j=1}^{K}\tilde{\beta}_{k,j}h_j(X_{j,{y_k^i}}))_{y_k^i} \right]\\
    & - 2\mathbb{E}_{\mathbb{\hat{D}}} \left[\phi_{\mathcal{T}}(\sum_{j=1}^{K}\tilde{\beta}_{k,j}h_j(X_{j,{y_k^i}}))_{y_k^i} \right] \\
    & \le \underset{\text{aggregation error}}{\underbrace{2-2\mathbb{E}_{\mathbb{\hat{D}}} \left[\phi_{\mathcal{T}}(\sum_{j=1}^{K}\tilde{\beta}_{k,j}h_j(X_{j,{y_k^i}}))_{y_k^i} \right]}} \\
    & + \underset{\text{distillation similarity}}{\underbrace{2|\mathbb{E}_{\mathbb{\hat{D}}}[\phi_{\mathcal{T}}(\sum_{j=1}^{K}\tilde{\beta}_{k,j}h_j(X_{j,{y_k^i}}))_{y_k^i}]-\mathbb{E}_{\mathbb{\hat{D}}}[\phi_{\mathcal{T}}(h_k(x_k^i))_{y_k^i} ]|}}
\end{aligned}
\label{CE_KD_eq2}
\end{equation}

\begin{remark}
\label{KDremark}
Eq. (\ref{CE_KD_eq1}) indicates the additional information entropy that we have to transfer knowledge from $p$ (\ie global distribution) to $q$ (\ie local distribution), and Eq. (\ref{CE_KD_eq2}) shows that the model performance is related to global logits and the distribution similarity.
At the beginning of every global epoch, $I(\phi_{\mathcal{T}}(\sum_{j=1}^{K}\tilde{\beta}_{k,j} h_j(X_{j,{y_k^i}})))$ is fixed because of the already uploaded local logits to the central server and fixed weight factor $\tilde{\beta}_{k,j}$.
The more similar the distribution of $\sum_{j=1}^{K}\tilde{\beta}_{k,j} h_j(X_{j,{y_k^i}})$ and $h_k(x_k^i)_{y_k^i}$ are, the less additional information is required, further reducing the generalization error.
Moreover, if global aggregated logits become more precise, the probability of wrong prediction will decrease and further better guide local learning.
We also can adjust the local distribution closer to the global distribution by modifying the weight coefficient $\tilde{\beta}_{k,j}$ and uploading correctly predicted logits to the server to improve the aggregation performance.
\end{remark}

Finally, for any $\epsilon \in (0,1)$, with the probability at least $1-\epsilon$, we denote $\mathbf{R}_{\mathbb{{D}}}(h)$ in Eq. (\ref{eq:lossbound_kd1}) as $\mathbf{R}_{\mathbb{{D}}_k}^{kd}(h_k)$ for client $k$. 
Exploiting Eq. (\ref{eq:global_logits_kd}) with the number of local samples $|\mathcal{D}_k|$ and Eq. (\ref{CE_KD_eq1}) to Eq. (\ref{eq:lossbound_kd1}) for client $k$, the client-specific generalization error bound over global logit distillation is:
\begin{equation}
\begin{split}
    & \mathbf{R}_{\mathbb{D}_k}^{kd}(h_k) \le \frac{1}{|\mathcal{D}_k|}\sum_{i=1}^{|\mathcal{D}_k|} KL(\phi_{\mathcal{T}}(\sum_{j=1}^{K} \tilde{\beta}_{k,j} h_j(X_{j,{y_k^i}}))||\phi_{\mathcal{T}}(h_k(x_k^i)))+\frac{\lambda_k-\log\epsilon}{t|\mathcal{D}_k|} + \frac{tC_k^2}{8} \\
    & = \frac{1}{|\mathcal{D}_k|}\sum_{i=1}^{|\mathcal{D}_k|} \left[\ell^{ce}(\phi_{\mathcal{T}}(\sum_{j=1}^{K}\tilde{\beta}_{k,j} h_j(X_{j,{y_k^i}})), \phi_{\mathcal{T}}(h_k(x_k^i))) - I(\phi_{\mathcal{T}}(\sum_{j=1}^{K}\tilde{\beta}_{k,j} h_j(X_{j,{y_k^i}})))\right]\\
    & +\frac{\lambda_k-\log\epsilon}{t|\mathcal{D}_k|} + \frac{tC_k^2}{8},
\end{split}
\end{equation}
where we replace $m$ and $C$ with $|\mathcal{D}_k|$ and $C_k$, respectively.
The $\ell$ in Eq. (\ref{eq:lossbound_kd1}) is the loss in Eq. (\ref{CE_KD_eq1}).
Because the local logits have different labels, we use the corresponding label $y_k^i$ of $x_k^i$ to replace the $c$ in $h_j(X_{j,c})$ and denoted as $h_j(X_{j,{y_k^i}})$.
\begin{remark}
    In addition to distillation distance loss, the number of samples and the loss bound also influence the error. The lower the loss bound $C_k$ and the more samples $|\mathcal{D}_k|$ will improve generalization ability. The sum of weight coefficients $\tilde{\beta}_{k,j}$ is not required to equal 1. We can control $\tilde{\beta}_{k,j}$ to modify the global weighted logits and thus influence the distillation loss, further adjusting the error bound and improve the model performance.
\end{remark}
For simplification, we use $\sum_{j=1}^{K}\tilde{\beta}_{k,j} h_j(X_{j})$ and $h_k(X_{k})$ to represent all samples, and the generalization error bound over global logit distillation is denoted as:
\begin{align*}
 \mathbf{R}_{\mathbb{D}_k}^{kd}(h_k) & \le \ell^{ce}(\phi_{\mathcal{T}}(\sum_{j=1}^{K}\tilde{\beta}_{k,j} h_j(X_{j}), \phi_{\mathcal{T}}(h_k(X_k))) - I(\phi_{\mathcal{T}}(\sum_{j=1}^{K}\tilde{\beta}_{k,j} h_j(X_{j}))) \\
& +\frac{\lambda_k-\log\epsilon}{t|\mathcal{D}_k|} + \frac{tC_k^2}{8}.
\end{align*}

\subsection{Generalization bound}
Inspired by~\cite{PAC_sum}, we have $\mathbf{R}_{\mathbb{D}_k}(h_k) \le \sup \{\mathbf{R}_{\mathbb{D}_k}^{ce}(h_k), \mathbf{R}_{\mathbb{D}_k}^{kd}(h_k)\}= \mathbf{R}_{\mathbb{D}_k}^{ce}(h_k) + \mathbf{R}_{\mathbb{D}_k}^{kd}(h_k)$ for each client k with its local model $h_k$ and local dataset $\mathcal{D}_k = (X_k,Y_k)$.
Thus, it gets:
\begin{equation*}
\resizebox{0.95\textwidth}{!}{$
\begin{aligned}
    &\mathbf{R}_{\mathbb{D}_T}(h_k)  \le \mathbf{R}_{\hat{\mathbb{D}}_k}^{ce}(h_k) + \sqrt{\frac{kl(h_k||h_0)+\ln\sqrt{4|{\mathcal{D}}_k|}-\ln \epsilon}{2|{\mathcal{D}}_k|}} + \lambda_{\mathbb{D}_k, \mathbb{D}_T}(h_0)+ d_{\mathcal{C}_h}((\mathbb{D}_k)_X,(\mathbb{D}_T)_X)\\
    & + \ell^{ce}(\phi_{\mathcal{T}}(\sum_{j=1}^{K}\tilde{\beta}_{k,j} h_j(X_{j})), \phi_{\mathcal{T}}(h_k(X_{k}))) - I(\phi_{\mathcal{T}}(\sum_{j=1}^{K}\tilde{\beta}_{k,j} h_j(X_j)))+\frac{\lambda_k-\log \epsilon}{t|\mathcal{D}_k|} + \frac{tC_k^2}{8},
\end{aligned}
$}
\end{equation*}
where $\lambda_{{\mathbb{D}_k},{\mathbb{D}}_T}(h_0)=\mathbf{R}_{{\mathbb{D}}_k}(h_0)+\mathbf{R}_{{\mathbb{D}}_T}(h_0)$ and $\mathcal{C}_h=h\Delta\mathcal{H}$.
It is worth noting that in the process of global logit distillation, the weight factor $\gamma$ and the temperature $\mathcal{T}$ in the loss function $\ell_k^{kd}$ influence the generation bound by constraining the loss range $C_k$ and $\ell^{kd}$ more than directly affect $\mathbf{R}_{\mathbb{D}_k}^{kd}(h_k)$.

\section{Details of experimental settings}
\label{appendix_setting_details}
\subsection{Details of computational resources}
\label{appendix_com_resource}
The proposed FedHPL is implemented in PyTorch~\cite{PyTorch} 2.2.2 and NVIDIA GeForce RTX 4090 with CUDA version 12.4. 

\subsection{Details of dataset and model settings}
\label{appendix_model_dataset}
\textbf{Dataset setting.}
We illustrate the local data distribution of each client on benchmark datasets with different dataset settings in Figure~\ref{AppendixdatasetSetting}.
CIFAR10~\cite{CIFAR} consists of 10 classes, each of which contains 5,000 training images and 1,000 testing images.
CIFAR100~\cite{CIFAR} contains 100 classes, with 500 training images and 100 testing images per class.
SVHN~\cite{SVHN} is comprised of 73,257 training images and 26,032 testing images for the digital recognition task with 10 classes.
Especially, the amount of per-class data is different in SVHN.
For \textit{IID} (independent identical distribution) data setting, we randomly sample independent data from the entire dataset.
In the setting of data heterogeneity (imbalanced class distribution), we conduct two different statistical settings for each client by randomly sampling data according to the Dirichlet distribution.
For \textit{Dir} data setting, each client $k$ samples $q_{k,c} \sim Dir(\alpha_{k,c})$ for each class $c$ (${\textstyle \sum_{c=1}^{n_c}\alpha_{k,c}}=1$ in CIFAR10, ${\textstyle \sum_{k=1}^{K}\alpha_{k,c}}=1$ in CIFAR100 and SVHN and $\alpha_{k,c}$ is randomly generated instead of giving certain values), then randomly assigns $q_{k,c}$ proportion of samples from the benchmark dataset for each class $c$.
In other words, for each class $c$, the number of samples in class $c$ on client $k$ (\ie $|\mathcal{D}_{k,c}|$) is equal to $q_{k,c}*|\mathcal{D}_{k}|$.
In this case, different clients have overlapping samples.
For \textit{Non-IID} data setting, each client $k$ samples $q_{k,c} \sim Dir(\alpha_{k,c})$ (${\textstyle \sum_{k=1}^{N}\alpha_{k,c}}=1$), and only chooses the corresponding non-overlap proportion of each class.
In this setting, samples between clients are independent.
A higher $\alpha$ means more balanced data distribution.
In addition, clients can specify the minimum quantity that they could have in all settings.
We also investigate different $\alpha$ in the \textit{Non-IID} data setting in FedHPL over benchmark datasets.
For fair comparison, the settings of client models and the split of private dataset in all approaches are kept the same.

\begin{figure}[htbp]
    \centering
    \subfloat[CIFAR10-IID]{\includegraphics[width=.33\linewidth]{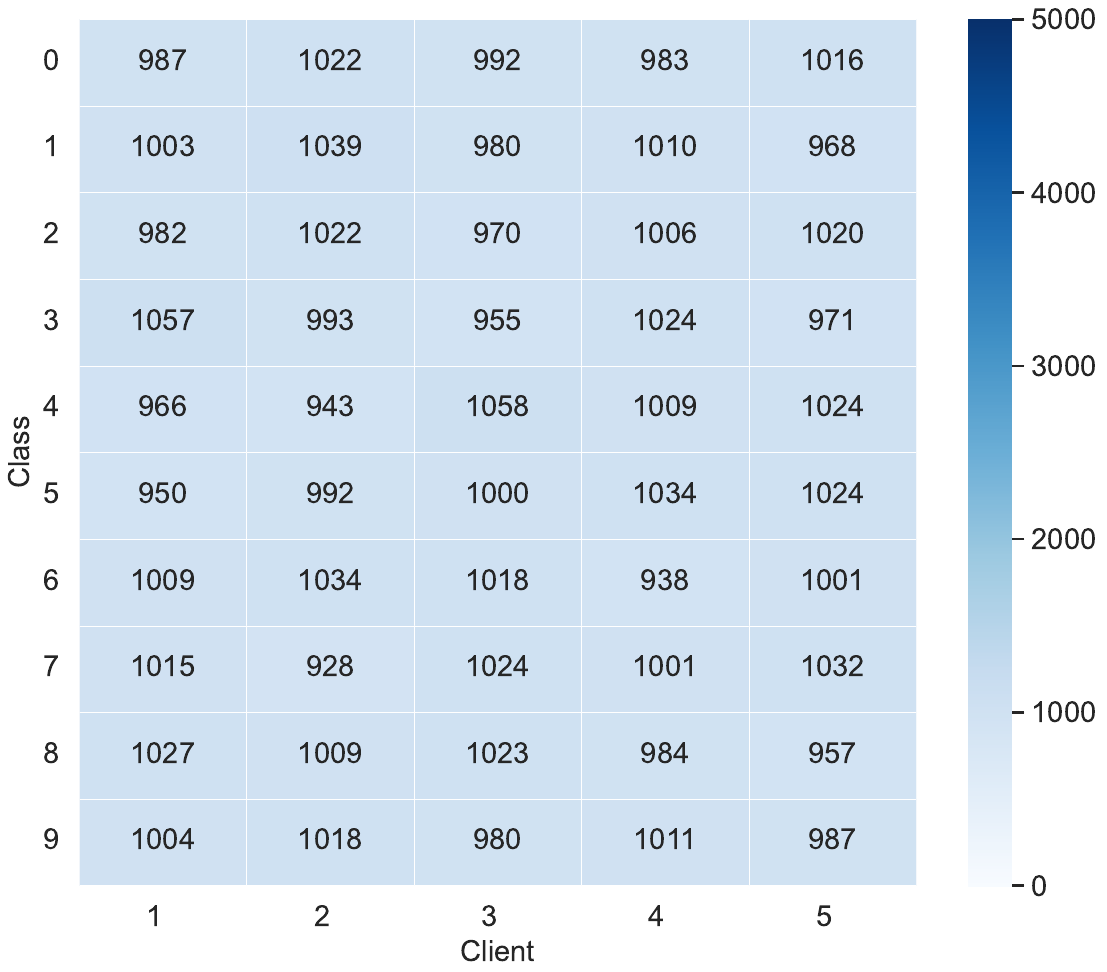}}
    \subfloat[CIFAR10-Dir]{\includegraphics[width=.33\linewidth]{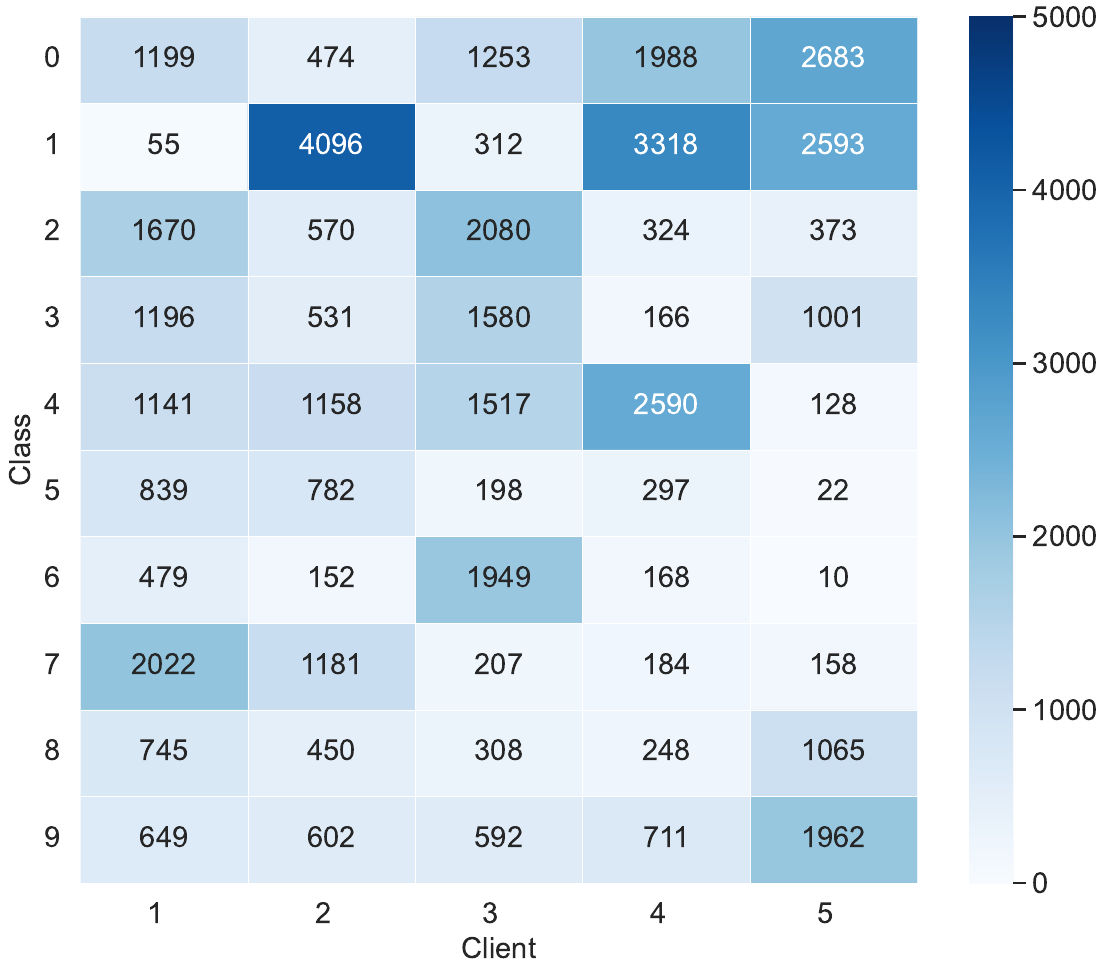}}
    \subfloat[CIFAR10-Non-IID]{\includegraphics[width=.33\linewidth]{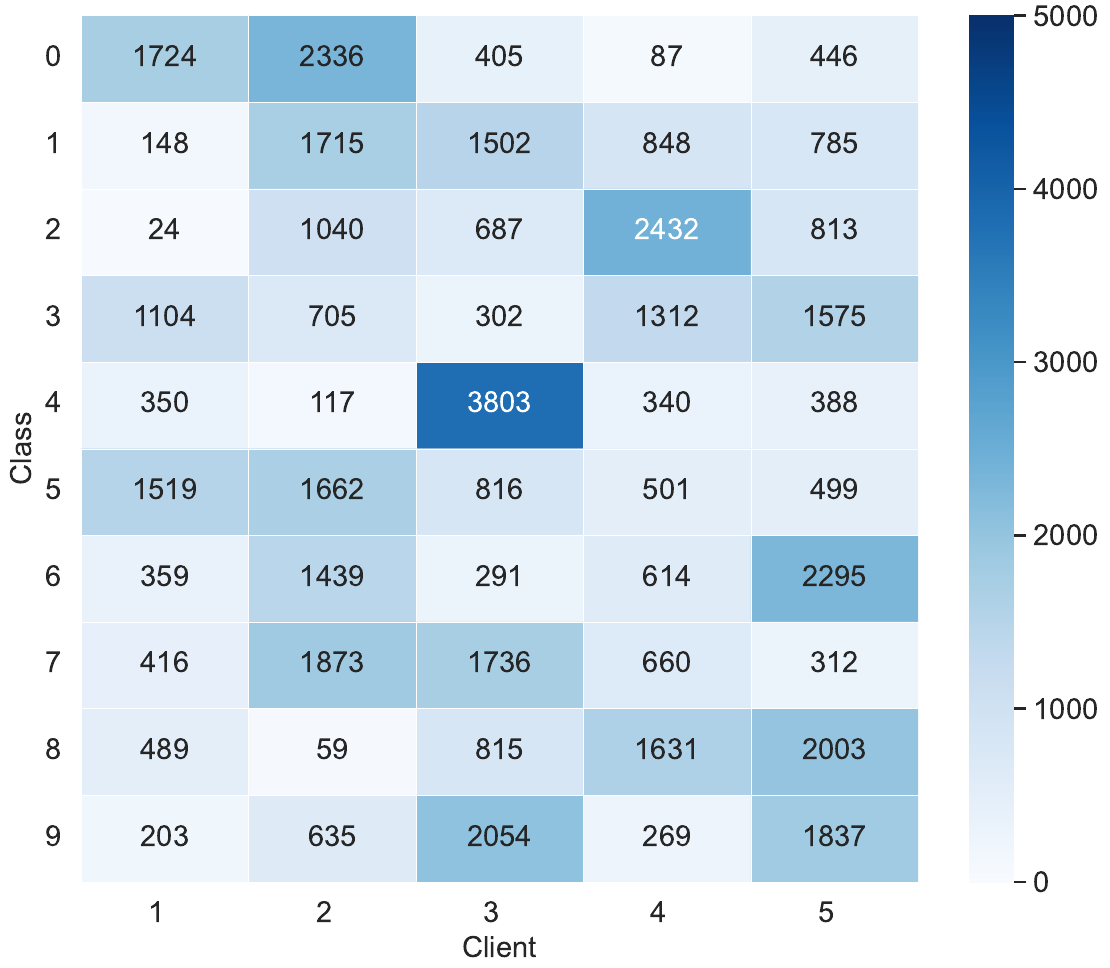}}\\
    \subfloat[CIFAR100-IID]{\includegraphics[width=.33\linewidth]{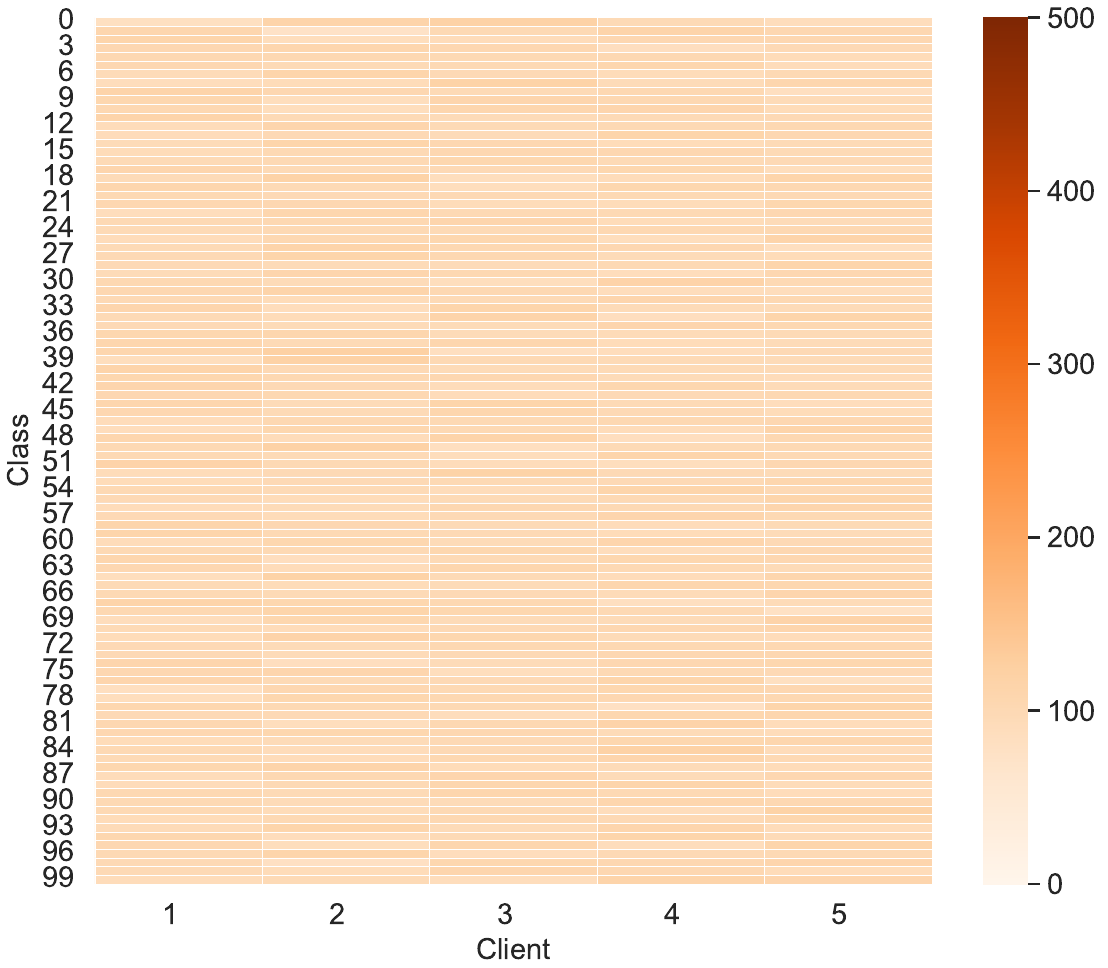}}
    \subfloat[CIFAR100-Dir]{\includegraphics[width=.33\linewidth]{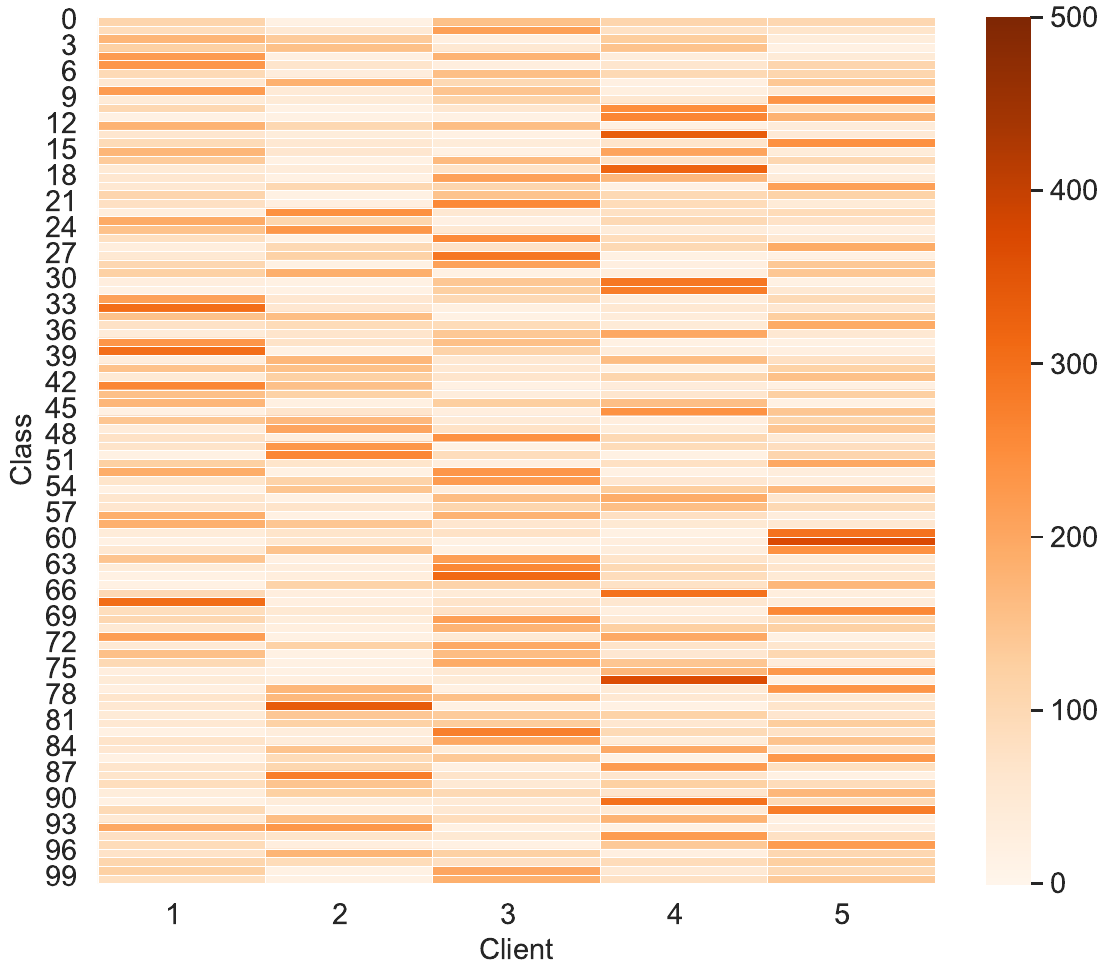}}
    \subfloat[CIFAR100-Non-IID]{\includegraphics[width=.33\linewidth]{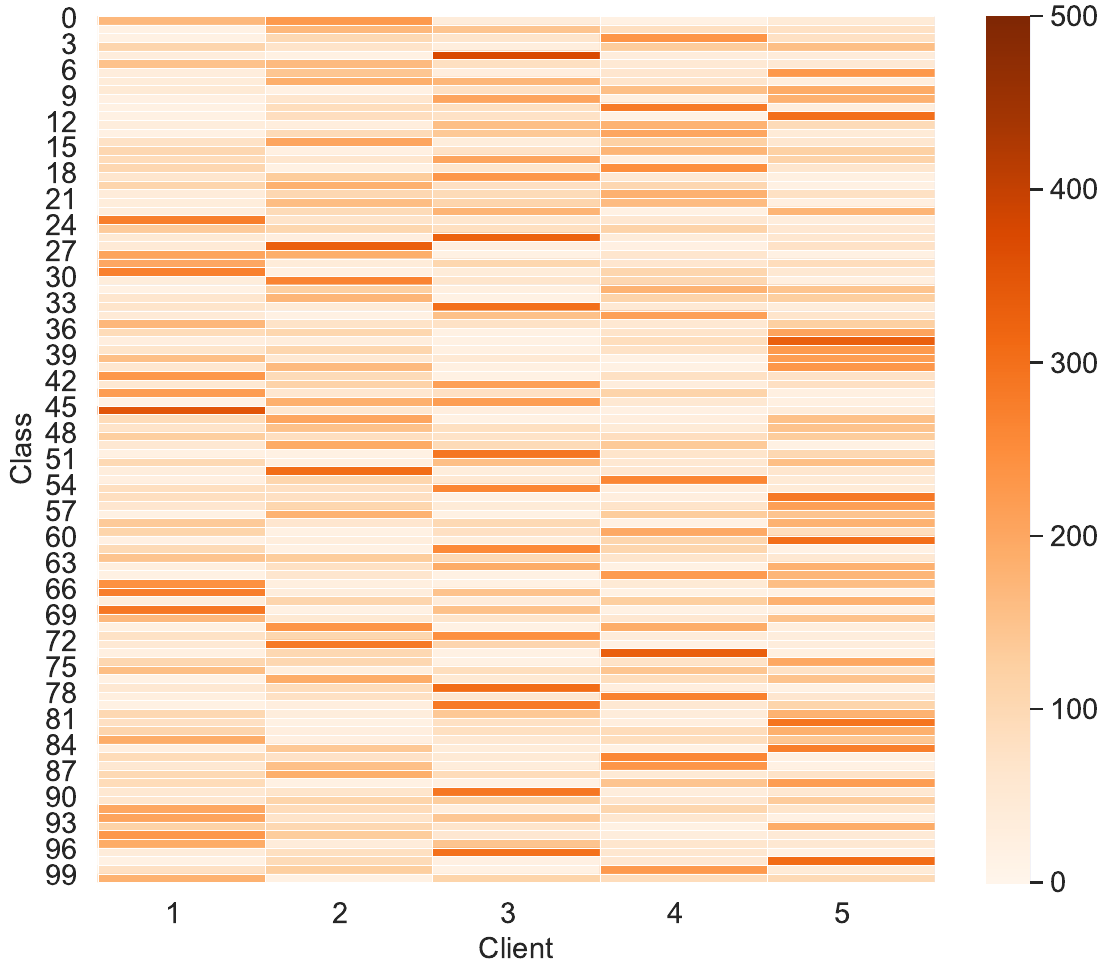}}\\
    \subfloat[SVHN-IID]{\includegraphics[width=.33\linewidth]{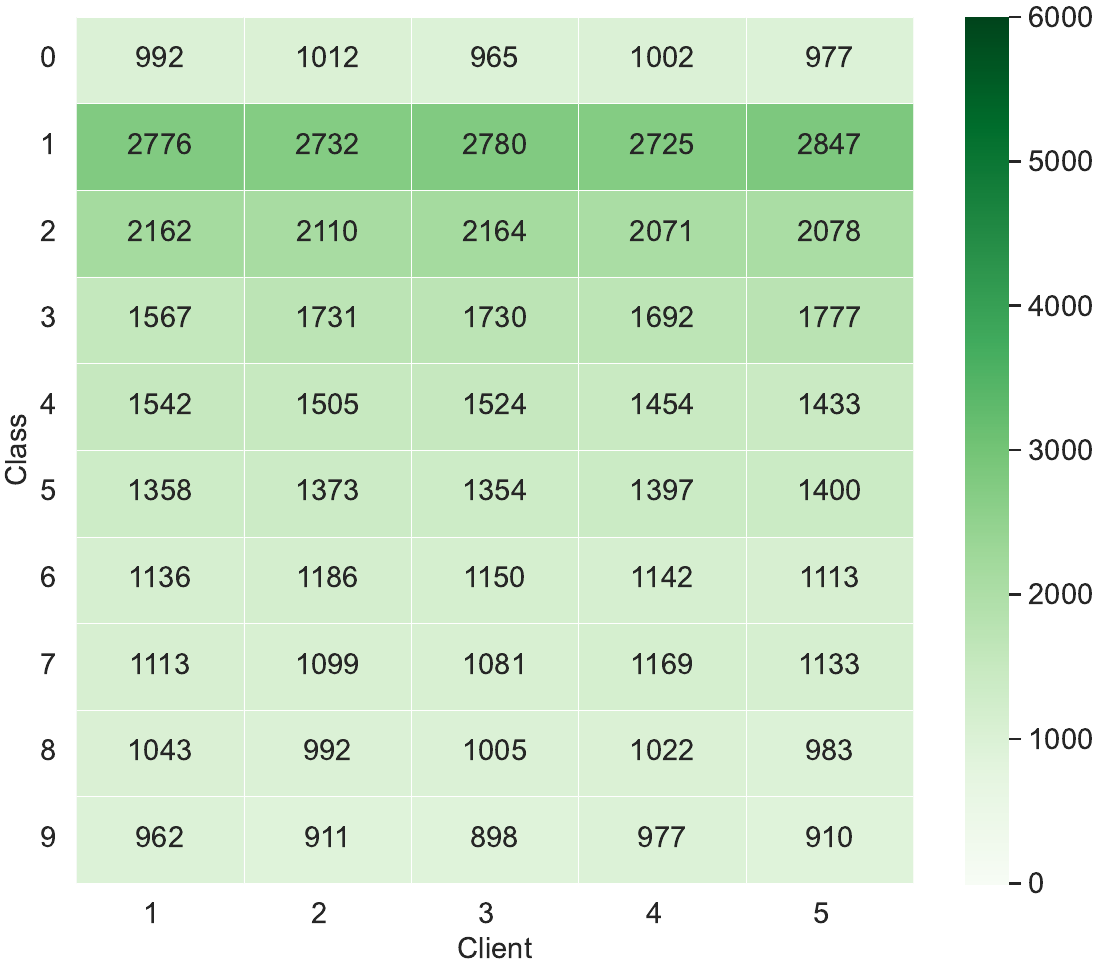}}
    \subfloat[SVHN-Dir]{\includegraphics[width=.33\linewidth]{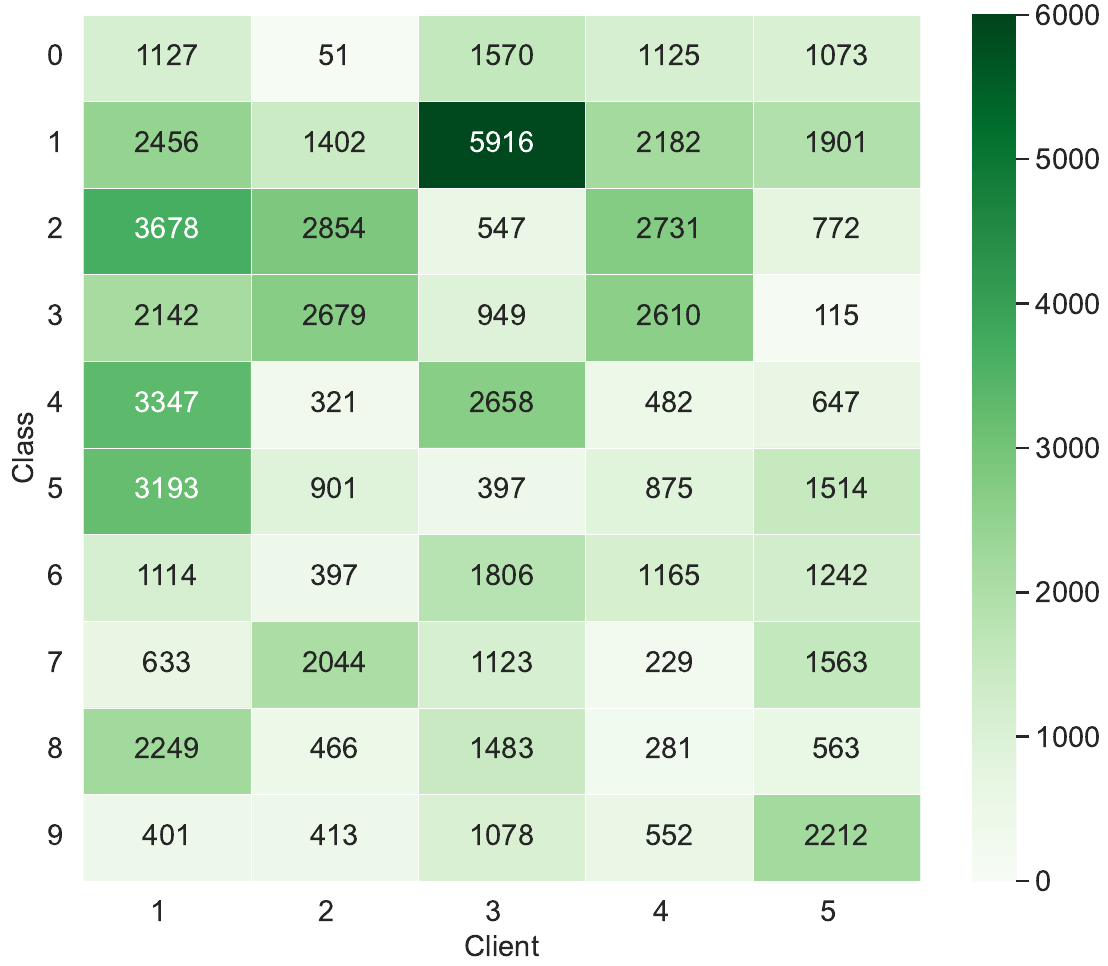}}
    \subfloat[SVHN-Non-IID]{\includegraphics[width=.33\linewidth]{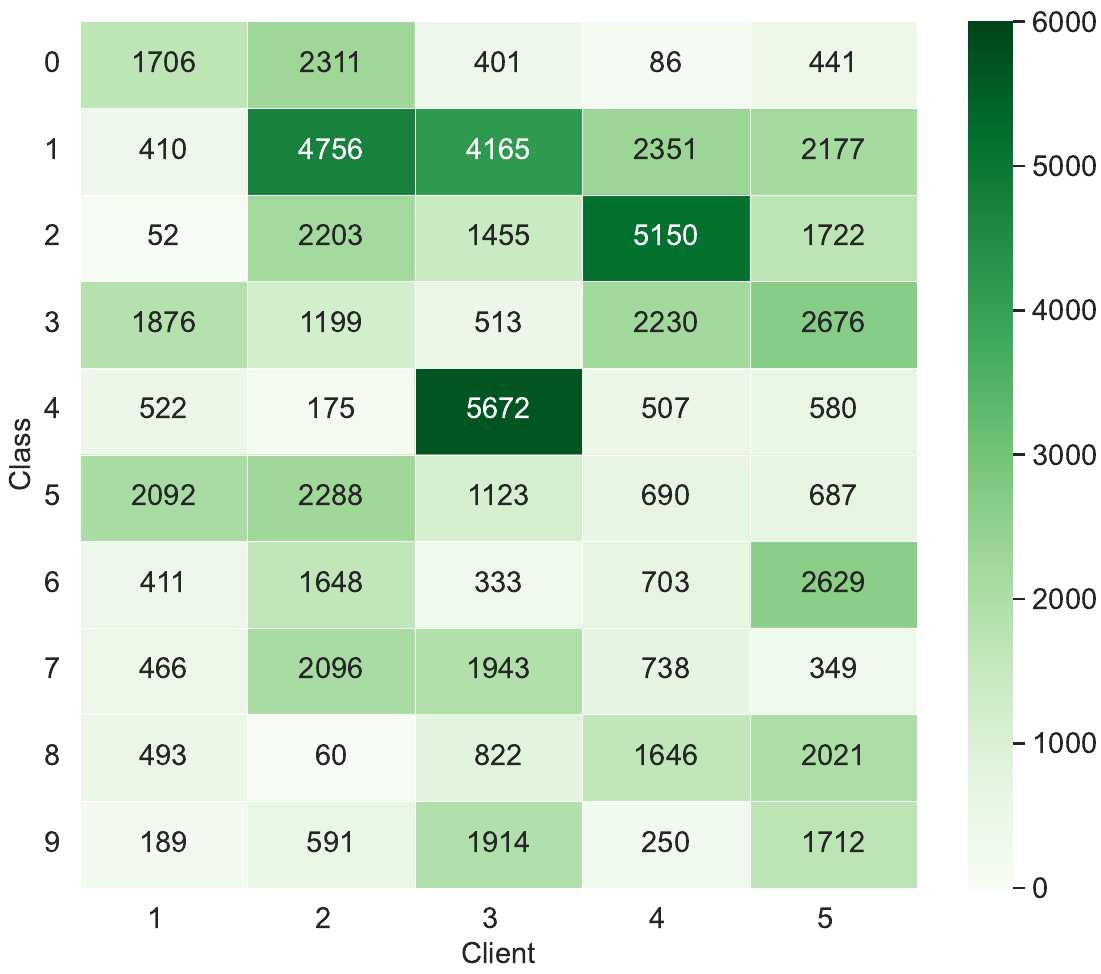}}\\
\caption{Heat maps for data sample distribution of each client on benchmark datasets. The subcaption `A-B' represents the experimental dataset and dataset setting.}
\label{AppendixdatasetSetting}
\end{figure}

\textbf{Model setting.}
We adopt ViT-B/16~\cite{ViT} as the pre-trained backbone over the experiments of the \textit{homogeneous model} setting.
Notably, we apply ResNet~\cite{ResNet} over FedHPL to fairly compare with other methods in the \textit{heterogeneous model} setting.
Then, we use ViT backbones for later ablation study and analysis.
The above backbones are all pre-trained on ImageNet-1k~\cite{ImageNet} which fine-tuned from ImageNet21k (image size: 224$\times$224) by supervised learning.
The specific model details are shown in Table~\ref{AppendixmodelSetting}.
Moreover, the training framework of FedHPL is based on FedGKT~\cite{FedGKT}, an open source in federated learning research with a distributed heterogeneous model environment.

\begin{table}[ht]
\caption{Client models in all model settings. All approaches in Table~\ref{tab:homo_model} apply the homogeneous model setting. FedHPL (CNN) and baselines in Table~\ref{tab:hete_model} use the heterogeneous ResNet model setting. FedHPL (ViT) in Table~\ref{tab:hete_model} uses the heterogeneous model setting (ViT backbones). Later experiments (ablation study in Figure~\ref{fig:ablation_theorem}, analysis in Table~\ref{tab:analysis_component}, and exploring experiments with FedHPL) use homogeneous and heterogeneous model settings (ViT backbones).}
\label{AppendixmodelSetting}
\centering
\resizebox{0.99\textwidth}{!}{
\begin{tabular}{c|c|c} 
\hline
\textbf{Model Setting}  & \textbf{Client models}   & \textbf{latent dimension} ($d_k$)       \\ 
\hline
homogeneous            & [ViT-B/16, ViT-B/16, ViT-B/16, ViT-B/16, ViT-B/16]  & [768, 768, 768, 768, 768]    \\
heterogeneous ResNet   & [ResNet18, ResNet34, ResNet50, ResNet34, ResNet18]  & [512, 512, 2048, 512, 512]\\
heterogeneous          & [ViT-S/16, ViT-B/16, ViT-L/16, ViT-B/16, ViT-S/16]  & [384, 768, 1024, 768, 384]\\
\hline
\end{tabular}
}
\label{tab:model_hete_settings}
\end{table}

\subsection{Details of implementation}
In the main text, we mainly describe the implementation of ViT backbones in FedHPL, and we use 10 prompts for the ResNet backbones in FedHPL.
We run 100 rounds for the heterogeneous ResNet setting in FedHPL, while in homogeneous and heterogeneous model settings (ViT backbones), we run 10 rounds in CIFAR10 and 15 rounds for CIFAR100 and SVHN datasets.
For all baselines, we run 100 global epochs for better comparison.
Note that in the evaluation phase of FedBABU~\cite{FedBABU} and the last local epoch of FedRep~\cite{FedRep}, we train 10 local epochs according to their original settings.
In pFedPT~\cite{pFedPT}, we use the original setting and ViT models.
In pFedPG~\cite{pFedPG}, we use 10 prompts with the supervised pre-trained backbones and 5 local client training epochs according to the original paper.
In FedMD~\cite{FedMD}, we perform 50 epochs of pre-training and model initialization before model tuning and execute 5 local training in each global round.

\subsection{Details of advanced baselines}
\label{Appendix2}
\textbf{Baselines in homogeneous models.}
\textbf{FedAVG}~\cite{FedAvg} proposes a communication-efficient distributed deep learning framework, which can achieve parameter aggregation by uploading updates of local models and then constructing a global model for decentralized clients.
\textbf{FedProx}~\cite{FedProx} applies a proximal term to adapt the model drift caused by multiple local updates in statistically heterogeneous datasets.
\textbf{SCAFFOLD}~\cite{SCAFFOLD} uses control variates in local training to correct for the local drift from the global models and further improve the model performance.
\textbf{FedBABU}~\cite{FedBABU} divides a local model into a feature extractor and a classification head and only updates the extractor during the training phase with a randomly initialized head which is further fine-tuned in the evaluation process.
\textbf{FedRep}~\cite{FedRep} shares the parameters of extractors for aggregation and loads the global parameters with private data to train the head.
\textbf{pFedPT}~\cite{pFedPT} firstly leverages a trainable personalized prompt generator to capture the local distribution through visual prompts with a frozen backbone, then trains the backbone with the frozen generator and sends backbone parameters for information aggregation.
\textbf{pFedPG}~\cite{pFedPG} trains the local prompt generator with a frozen pre-trained model by visual prompt tuning and generates global prompts by observing local optimization directions.

\textbf{Baselines in heterogeneous models.}
\textbf{FedGen}~\cite{FedGEN} addresses heterogeneous federated learning through a lightweight generator, which ensembles user information in a data-free manner and further regulates local training for improving generalization ability.  
\textbf{FedGH}~\cite{FedGH} proposes a model-heterogeneous FL framework by training a private feature extractor to generate local representations.
Then clients transfer them to a central server for training a global classification head and acquiring the global distribution.
\textbf{FedProto}~\cite{FedProto} uses the federated prototype learning framework instead of gradients to improve heterogeneous tolerance by regularizing the local model with aggregated local prototypes from different users.
\textbf{FedTGP}~\cite{FedTGP} introduces an adaptive-margin-enhanced contrastive learning mechanism to train class-wise prototypes among clients and a server, further improving the adaptation in the model heterogeneity environment.
\textbf{FedMD}~\cite{FedMD} simply leverages the knowledge distillation algorithm to achieve information sharing between different local models and a central server.
\textbf{FedHE}~\cite{FedHE} applies knowledge distillation based on logit vectors to train various heterogeneous models and reduces communication overheads.

\section{Details of experimental results}
\subsection{Parameter analysis}
\label{Appendixparameters}
We first briefly analyze the model parameters of the above FL algorithms from the perspective of trainable parameters and uploaded parameters.

\subsubsection{Homogeneous model experiments} 
\label{Appendixparameterhomo}
\textbf{Trainable parameters.}
The entire model parameters are trained at the local client side in \textbf{\textit{loss-based}} federated learning methods and FedRep, while FedBABU exploits some special training tricks for better performance in data heterogeneity. 
It trains the backbone during the training phase and keeps the classification head frozen while updating either the head or the entire model (we prefer the entire model) during the evaluation phase.
Additionally, \textbf{\textit{VPT-based}} federated learning methods train a generator for generating prompts which can capture the data distribution in the training phase.
Especially, pFedPT inserts prompts into the input space (\ie images) instead of the model and only freezes the backbone in stage 1 to train a prompt generator and train the entire model parameter with the frozen generator in stage 2 with the original input space (\ie 32$\times$32 pixels).
In pFedPG, only inserted prompts and classification head could be trained when the parameter of backbone is frozen after loading the pre-trained foundation model.
In summary, Table~\ref{tab:parameter_vit1} and Table~\ref{tab:parameter_vit2} show the number of trainable parameters over the above algorithms and our FedHPL.

\begin{table}
\caption{Trainable parameters (M) of backbones and entire models (the backbone with a linear head) in original ViT-S/16, ViT-B/16, and ViT-L/16 with 3 prompts. We also count the parameter of ViT in pFedPT and the generator in pFedPT and pFedPG, denoted as ViT, G1, and G2. Notably, pFedPG only updates prompts (10 prompts) and head parameters on the client side whereas the prompt generator is trained on the server side. So, the full trainable parameters in pFedPT (F1) consist of ViT and G1, while the parameters in pFedPG (F2) consist of G2, prompts, and a head.}
\label{tab:parameter_vit1}
\centering
\resizebox{0.98\textwidth}{!}{
    \begin{tabular}{c|cc|ccc|cc} 
    \hline
    Dataset        & Backbone (S/B/L)   & Entire Model (S/B/L) & ViT    & G1 (K)  & F1    & G2     & F2   \\
    \hline
    CIFAR10/SVHN   & 20.66/81.82/289.25 & 20.57/81.83/289.26   & 1.026  & 1.3125  & 1.028 & 1.512  & 1.526 \\
    CIFAR100       & 20.66/81.82/289.25 & 20.70/81.90/289.35   & 1.037  & 1.3125  & 1.039 & 1.512  & 1.592 \\
    \hline
    \end{tabular}
}
\end{table}

\begin{table}
\caption{Trainable parameters (K) in FedHPL with different insertion position, which consists of prompts and a classification head. $n$ represents the number of prompts for each backbone layer. For example, if $n=3$ and the insertion style changes from VPT-shallow to VPT-deep, the trainable parameter over the CIFAR10 and SVHN dataset changes from 4.89K (VPT-shallow) to 17.26K (VPT-deep) in ViT-S, 9.76K to 34.51K in ViT-B and 13.01K to 82.01K in ViT-L.}
\label{tab:parameter_vit2}
\centering
\resizebox{0.97\textwidth}{!}{
    \begin{tabular}{c|ccc|ccc} 
    \hline
    \multirow{2}{*}{Dataset} & \multicolumn{3}{c|}{VPT-shallow}   & \multicolumn{3}{c}{VPT-deep}                                                              \\ 
    \cline{2-7}
                             & ViT-S          & ViT-B          & ViT-L      & ViT-S        & ViT-B          
    & ViT-L     \\ 
    \hline
    CIFAR10/SVHN             & 3.76 + 0.375n  & 7.51 + 0.75n   & 10.01 + n  & 3.76 + 4.5n   & 7.51 + 9n   & 10.01 + 24n \\
    CIFAR100                 & 37.60 + 0.375n & 75.10 + 0.75n  & 100.10 + n & 37.60 + 4.5n  & 75.10 + 9n   & 100.10 + 24n \\
    \hline
    \end{tabular}
}
\end{table}

\textbf{Uploaded parameters.} 
For every global round for each client, the uploaded parameters are related to model parameters.
Specifically, FedAVG, FedProx, and SCAFFOLD upload entire model parameters to the central server, and SCAFFOLD needs to upload about twice the model parameters due to extra control variables.
FedBABU and FedRep only require clients to upload the backbone, saving communication costs for head parameters.
pFedPT only uploads the parameters of ViT and does not need to upload the generator G1.
pFedPG only uploads the update direction of prompts and further decreases communication overheads, only 7.5K parameters need to be uploaded for each client in every global round.
The communication overhead of FedHPL is related to the number of classes and correct logits and is further compressed by averaging the local logits of each category.
We analyze the communication overheads in Appendix~\ref{appendix_communication}.

\subsubsection{Heterogeneous model experiments}
\label{Appendixparameterhete}
We firstly explain why we choose the CNN backbones for the \textit{heterogeneous model setting}.
In most methods for model heterogeneity, they investigate the framework effectiveness in the CNN heterogeneous setting.
For example, FedGH varies the number of filters and the dimension in the CNN model and FedProto only considers the number heterogeneity of output channels in the convolutional layers.
Only FedTGP considers the Transformer architecture in their code but does not mention it in their original paper.
For a fair comparison with these SOTA methods, we apply our methods to CNN.

\textbf{Trainable parameters.}
In \textbf{\textit{model-based}}, \textbf{\textit{prototype-based}}, and \textbf{\textit{distillation-based}} federated learning algorithms, all model parameters are required to be trained while our FedHPL only needs to update prompts and a linear classification head.
The parameters of ResNet (\ie trainable parameters of baselines) are shown in Table~\ref{tab:parameter_resnet1}.
Notably, FedGen and FedTGP train an extra model at the server side, which are the knowledge generator and prototype generator and we show them in Table~\ref{tab:upload-hete-paras}.
It can be observed that baselines have different numbers of parameters because they apply distinct projection layers for the parameter aggregation in the model heterogeneous setting.
For example, a mapping layer in FedProto for aggregating prototypes of different latent dimensions, and a linear layer for aligning representations in FedGH.
The trainable parameters of FedHPL are shown in Table~\ref{tab:hete-FedHPL-para-ResNet}.
It can be seen that the parameters are increased with the image size and the number of classes.

\textbf{Uploaded parameters.}
In FedGen, clients upload head parameters and labels $|\mathcal{D}_k|$ while the server distributes the trained generator, global head, and label space.
In FedGH, clients upload the per-class average representations obtained from the output of a feature extractor, then the server uses them as inputs to train a global classification head and further distributes the head to each client.
In FedHE, only the per-class average logits are required to upload, and then the server aggregates and distributes the global per-class logits.
In FedMD, clients upload predicted logits on a shared public dataset to the server, and then the server aggregates and distributes the global logits.
For FedProto and FedTGP, clients and a server exchange the local prototypes and global prototypes, while the former exploits an averaging parameter method and the latter applies a generator to produce global prototypes.
In addition, the methods of uploading per-class parameters also upload the label space for distinguish labels.
The uploaded parameters with extra model architectures of baselines are illustrated in Table~\ref{tab:upload-hete-paras}.

\begin{table}[ht]
\caption{Entire model parameters (M) of ResNet in the heterogeneous model setting. We also count the model parameters of FedHPL with different image sizes (32 $\times$ 32 and 224 $\times$ 224).}
\label{tab:parameter_resnet1}
\centering
    \begin{tabular}{c|ccc|ccc} 
    \hline
    \multirow{2}{*}{Method} & \multicolumn{3}{c|}{CIFAR10/SVHN}                            & \multicolumn{3}{c}{CIFAR100}                   \\ 
    \cline{2-7}
                             & ResNet18    & ResNet34   & ResNet50      & ResNet18      & ResNet34      & ResNet50       \\ 
    \hline
    FedGen                   & 10.664      & 20.304     & 22.424        & 10.708        & 20.348        & 22.468         \\
    FedGH                    & 10.907      & 20.547     & 23.417        & 10.951        & 20.591        & 23.461         \\
    FedProto                 & 11.664      & 21.304     & 23.424        & 11.708        & 21.348        & 23.468         \\
    FedTGP                   & 10.664      & 20.304     & 22.424        & 10.708        & 20.348        & 22.468         \\
    FedMD                    & 10.661      & 20.301     & 22.451        & 10.705        & 20.345        & 22.627         \\
    FedHE                    & 10.656      & 20.296     & 22.431        & 10.700        & 20.340        & 22.607         \\
    $\text{FedHPL}^{32}$     & 10.665      & 20.305     & 22.440        & 10.709        & 20.349        & 22.616  \\
    $\text{FedHPL}^{224}$    & 10.671      & 20.311     & 22.446        & 10.715        & 20.355        & 22.622  \\
    \hline
    \end{tabular}
\end{table}

\begin{table}[ht]
\caption{Trainable parameters of FedHPL in ResNet18, ResNet34, and ResNet50 with different image sizes (32/224). The trainable parameters in ResNet18 and ResNet34 are the same. We exhibit them in normal font, whereas we display the parameters of ResNet50 in \textbf{bold} font. Prompts inserted to the left and right of input images are denoted as `prompt$\_$lr'. Prompts inserted to the top and bottom of input images are denoted as `prompt$\_$tb'. The parameters of the head are independent of the image size. We also count the sum of trainable parameters with the different image sizes.}
\label{tab:hete-FedHPL-para-ResNet}
\centering
\resizebox{0.98\textwidth}{!}{
\begin{tabular}{c|ccccc} 
\hline
Dataset      & prompt\_lr & prompt\_tb & head (K) & $\text{sum}^\text{32}$ (K) & $\text{sum}^\text{224}$ (K)\\ 
\hline
CIFAR10/SVHN & $\text{576}^\text{32}$/$\text{4032}^\text{224}$ &$\text{684}^\text{32}$/$\text{4140}^\text{224}$ & 5.01/\textbf{20.01}     & 6.24/\textbf{21.24}  & 12.99/\textbf{27.99}\\
CIFAR100     & $\text{576}^\text{32}$/$\text{4032}^\text{224}$ &$\text{684}^\text{32}$/$\text{4140}^\text{224}$ & 50.10/\textbf{200.10}   & 51.33/\textbf{201.33}& 58.08/\textbf{208.08}\\
\hline
\end{tabular}
}
\end{table}

\begin{table}[ht]
\caption{Uploaded parameters of each client. 
The generator parameter of FedGen and the prototype generator in FedTGP are denoted as $G$ and $G_p$.
The server ($H_{g}$) in FedGH needs to distribute the global head parameter.
The output dimension in FedMD is the number of the private classes $n_c$ and public classes $n_p$ ($n_p$ is usually equal to 10).
The communication of FedHPL can further compress by averaging the local logits on the category and we analyze it in Appendix~\ref{appendix_communication}.
}
\label{tab:upload-hete-paras}
\centering
\begin{tabular}{c|c|c|l} 
\hline
Method         & CIFAR10/SVHN             & CIFAR100                  & Remarks (CIFAR10/SVHN, CIFAR100)                                                  \\ 
\hline
FedGen   & 5130+$|\mathcal{D}_k|$& 51300+$|\mathcal{D}_k|$& $G$: 0.507M / 0.551M         \\
FedGH    & 5130& 51300& $H_{g}$: 5130 / 51300          \\
FedProto & 5130& 51300& ($n_c$) $\times$ per-class average prototypes + $n_c$               \\
FedTGP   & 5130& 51300& $G_p$: 0.506M / 0.550M     \\
FedMD    & 21$|\mathcal{D}_p|$& 111$|\mathcal{D}_p|$& ($n_c + n_p + 1$) $\times$ public samples ($|\mathcal{D}_p|$)\\
FedHE    & 110& 10100& per-class average logits       \\      
FedHPL   & 11$|\mathcal{\tilde{D}}_k|$& 101$|\mathcal{\tilde{D}}_k|$& weighted logits      \\
FedHPL+  & 110& 10100& after compression by category      \\
\hline
\end{tabular}
\end{table}

\subsection{Analysis of communication cost}
\label{appendix_communication}
In this subsection, we simply analyze the communication cost among clients from the number of uploaded parameters in a global communication round.

\textbf{Uploaded parameters in baselines.}
In approaches for the \textit{homogeneous model} setting, the communication cost is related to models.
We have explained this in detail in \textbf{Uploaded Parameters} of Appendix~\ref{Appendixparameterhomo} and provided some specific values in Table~\ref{tab:parameter_vit1}. 
In methods for the \textit{heterogeneous model} setting, the communication overhead can refer to \textbf{Uploaded Parameters} of Appendix~\ref{Appendixparameterhete}. 

\textbf{Uploaded parameters in FedHPL.}
We first analyze from a theoretical perspective.
Given all uploaded logits $\{p_{j}^i\}_{i=1}^{|\mathcal{\tilde{D}}_j|}$ from client $j$, the server aggregates logits for each client $k$:
\begin{equation}
\begin{split}
\tilde{p}_{k,c} &=\frac{\sum_{j=1}^{K}\beta_{k,j}\sum_{\forall (p_{j}^{i},y_{j}^{i})\in \vec{p}_j, y_{j}^{i}=c}p_{j}^{i}}{1+\sum_{j=1}^{K}\beta_{k,j}|{\mathcal{D}}_{j,c}|}=\sum_{j=1}^{K}\tilde{\beta}_{k,j}\sum_{i=1}^{|\mathcal{\tilde{D}}_{j,c}|}p_{j,c}^{i} \\
& =\sum_{j=1}^{K}\tilde{\beta}_{k,j}|\tilde{\mathcal{D}}_{j,c}|\bar{p}_{j,c}
=\sum_{j=1}^{K}\tilde{\beta}_{k,j,c}\bar{p}_{j,c}, 
\end{split}
\label{eq:appendix_logit}
\end{equation}
where $\tilde{\beta}_{k,j,c}=\tilde{\beta}_{k,j}|\tilde{\mathcal{D}}_{j,c}|$ represents the weight coefficient for client $j$ to client $k$ in class $c$ and $\bar{p}_{j,c}=\sum_{i=1}^{|\mathcal{\tilde{D}}_{j,c}|}p_{j,c}^{i}/|\mathcal{\tilde{D}}_{j,c}|$.
The size of the average logit $\bar{p}_{j,c} \in \mathbb{R}^{n_c}$ is only related to the number of classes $n_c$ and is independent of the number of correct logits $|\mathcal{\tilde{D}}_{j,c}|$.
Thus, clients in FedHPL can reduce the communication cost by averaging local logits by label and uploading the per-class knowledge $\{\bar{p}_{j,c}\}_{c=1}^{n_c}$ with the corresponding count value $|\mathcal{\tilde{D}}_{j,c}|_{c=1}^{n_c}$.
Furthermore, clients have fewer convergence rounds compared with baselines, further reducing the communication cost.
Next, we demonstrate the effectiveness of the average uploading mechanism (upload the per-class average logits $\{\bar{p}_{j,c}\}_{c=1}^{n_c}$ with extra count values $\{|\mathcal{\tilde{D}}_{j,c}|\}_{c=1}^{n_c}$)  from the experimental perspective.

As shown in Table~\ref{tab:comm_sim}, we exhibit the logit communication cost among all clients over the first and the last global round in FedHPL with the mechanism of uploading all correct logits. 
The communication cost among all clients in FedHPL with the average uploading mechanism is a constant, which is shown in the Table caption.
We can observe that the logit communication cost can compress to 1\% on CIFAR100 dataset and 0.1\% on CIFAR10 and SVHN datasets with the average uploading mechanism compared to the original mechanism.
It is also evident that the compression degree will gradually increase with the improvement of model performance because the number of correct logits will increase and further raise the overhead in the mechanism of uploading all correct logits, while the cost still remains constant in the average uploading technique because the number of labels is fixed.
Furthermore, in order to verify whether Eq. (\ref{eq:appendix_logit}) holds and whether the average uploading mechanism is equivalent to the original uploading mechanism, we perform experiments on whether using the same local logits can generate the same global logits under these two uploading methods.
The results show that the cosine similarity of all global logits between the two mechanisms is 1, which proves that the two methods produce consistent global per-class logits with the same local logits.

Then we test the model performance of the two uploading mechanisms in the same testing environment.
As illustrated in Table~\ref{tab:comm_acc}, the test accuracy of clients of the two methods is basically the same in the same testing environment, further proving our theory.
The slight difference is due to the random model training on each batch with the SGD optimizer.
So, clients can reduce the communication cost with the average uploading mechanism while maintaining the comparable performance to the original uploading mechanism which uploads all correctly predicted logits.

\begin{table}[ht]
\caption{The logit communication cost (M) in the first and last global round over all clients when uploading all correct logits. The communication cost with the average uploading mechanism is a constant (CIFAR10: 0.54K; CIFAR100: 49.32K; SVHN: 0.54K).}
\label{tab:comm_sim}
\centering
\subfloat[homogeneous Model]{
\resizebox{0.49\textwidth}{!}{
    \begin{tabular}{c|cc|cc|cc} 
    \hline
    \multirow{2}{*}{Data}& \multicolumn{2}{c|}{CIFAR10}                                                                  & \multicolumn{2}{c|}{CIFAR100}                                                                  & \multicolumn{2}{c}{SVHN}                                                                        \\ 
    \cline{2-7}
                   & First  & Last   & First   & Last   & First  & Last    \\ 
    \hline
    IID            &  0.445&   0.521&  4.085&  4.759& 0.688&0.749\\
    Dir            &  0.506&   0.524&  4.015&  4.741& 0.696&0.751\\
    Non-IID        &  0.510&   0.523&  3.455&  4.711& 0.702&0.749\\
    \hline
    \end{tabular}
}}
\subfloat[heterogeneous Model]{
\resizebox{0.49\textwidth}{!}{
    \begin{tabular}{c|cc|cc|cc} 
    \hline
    \multirow{2}{*}{Data}& \multicolumn{2}{c|}{CIFAR10}                                                                  & \multicolumn{2}{c|}{CIFAR100}                                                                  & \multicolumn{2}{c}{SVHN}                                                                        \\ 
    \cline{2-7}
                   & First  & Last   & First   & Last   & First  & Last    \\ 
    \hline
    IID            &  0.496&   0.524&  4.198&  4.744& 0.583&0.737\\
    Dir            &  0.514&   0.524&  3.629&  4.726& 0.622&0.736\\
    Non-IID        &  0.531&   0.524&  4.185&  4.733& 0.514&0.739\\
    \hline
    \end{tabular}
}}
\end{table}
\vspace{-1em}
\begin{table}[ht]
\caption{The test average accuracy (\%) among clients in FedHPL with ViT backbones. `Average' represents the average uploading mechanism, while `All' represents the original uploading mechanism that uploading all qualified logits without compression.}
\label{tab:comm_acc}
\centering
\subfloat[homogeneous Model]{
\resizebox{0.49\textwidth}{!}{
    \begin{tabular}{c|cc|cc|cc} 
    \hline
    \multirow{2}{*}{Data} & \multicolumn{2}{c|}{CIFAR10}   & \multicolumn{2}{c|}{CIFAR100}   & \multicolumn{2}{c}{SVHN}  \\ 
    \cline{2-7}
                      & Average  & All    & Average  & All    & Average  & All \\ 
    \hline
    IID           & 97.96    &        97.84&          89.66&        89.66&          95.51& 95.49\\
    Dir          & 96.76& 96.64& 86.92    & 86.91  & 93.21& 93.12\\
    Non-IID              & 96.76&        96.85& 86.65&        86.68& 91.88    & 91.37 \\
    \hline
    \end{tabular}
}}
\subfloat[heterogeneous Model]{
\resizebox{0.49\textwidth}{!}{
    \begin{tabular}{c|cc|cc|cc} 
    \hline
    \multirow{2}{*}{Data}& \multicolumn{2}{c|}{CIFAR10}   & \multicolumn{2}{c|}{CIFAR100}   & \multicolumn{2}{c}{SVHN}  \\ 
    \cline{2-7}
                      & Average  & All    & Average  & All    & Average  & All \\ 
    \hline
    IID& 96.93    &        96.68& 88.86&        88.78& 94.32&       94.38\\
    Dir& 96.06& 96.23& 85.88    & 85.87  & 91.49& 91.40\\
    Non-IID& 96.03& 95.38& 85.24& 85.46& 89.40    & 89.62 \\
    \hline
    \end{tabular}
}}
\end{table}

\subsection{Details of model performance}
\label{Appendix_table_details}
\textbf{The details of Table~\ref{tab:homo_model}.}
Figure~\ref{fig:homo-model-label-acc} reports the per-class average test accuracy among clients in all datasets over the \textit{homogeneous model} setting.
Because the central server in \textbf{\textit{personalized}} and \textbf{\textit{VPT-based}} federated learning baselines aggregates parameters for assisting clients to train private local models instead of generating a global model like \textbf{\textit{loss-based}} federated learning methods, we use the average test accuracy of all clients to represent the global performance.
In \textbf{\textit{loss-based}} federated learning methods, we test the accuracy of the global model.
The paper follows the configuration throughout all experiments when clients train their distinct local models with the global knowledge from the server instead of generating a global model.
We can observe that pFedPG and our FedHPL are obviously better than others on model performance, but our method can further accommodate the \textit{heterogeneous model} setting while pFedPG cannot.
From the result on CIFAR100 dataset, it can be inferred that the pre-trained parameters can provide a strong representation to local models in complex multi-class downstream tasks, thus significantly improving the model performance.

\begin{figure}[htbp]
    \centering
    \subfloat[CIFAR10-IID]{\includegraphics[width=.32\linewidth]{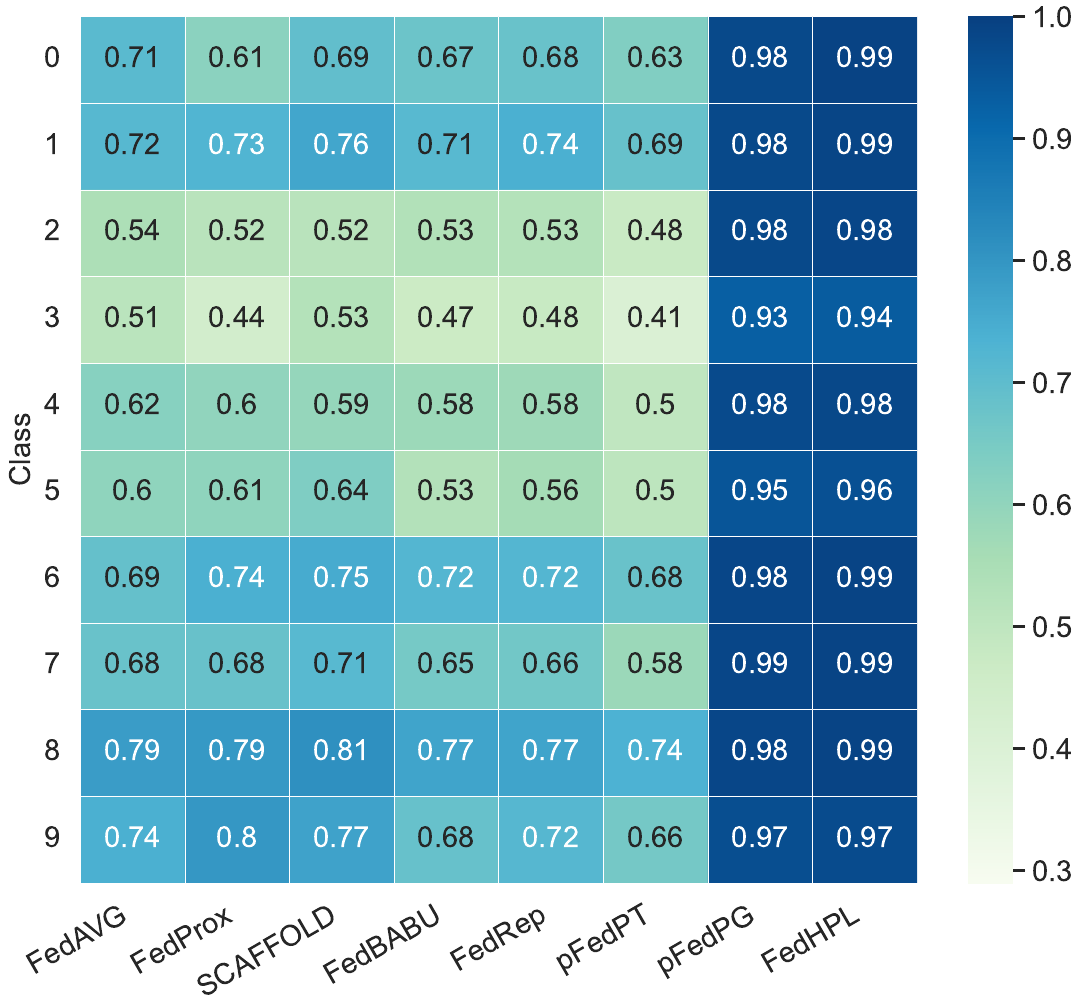}}
    \subfloat[CIFAR10-Dir]{\includegraphics[width=.32\linewidth]{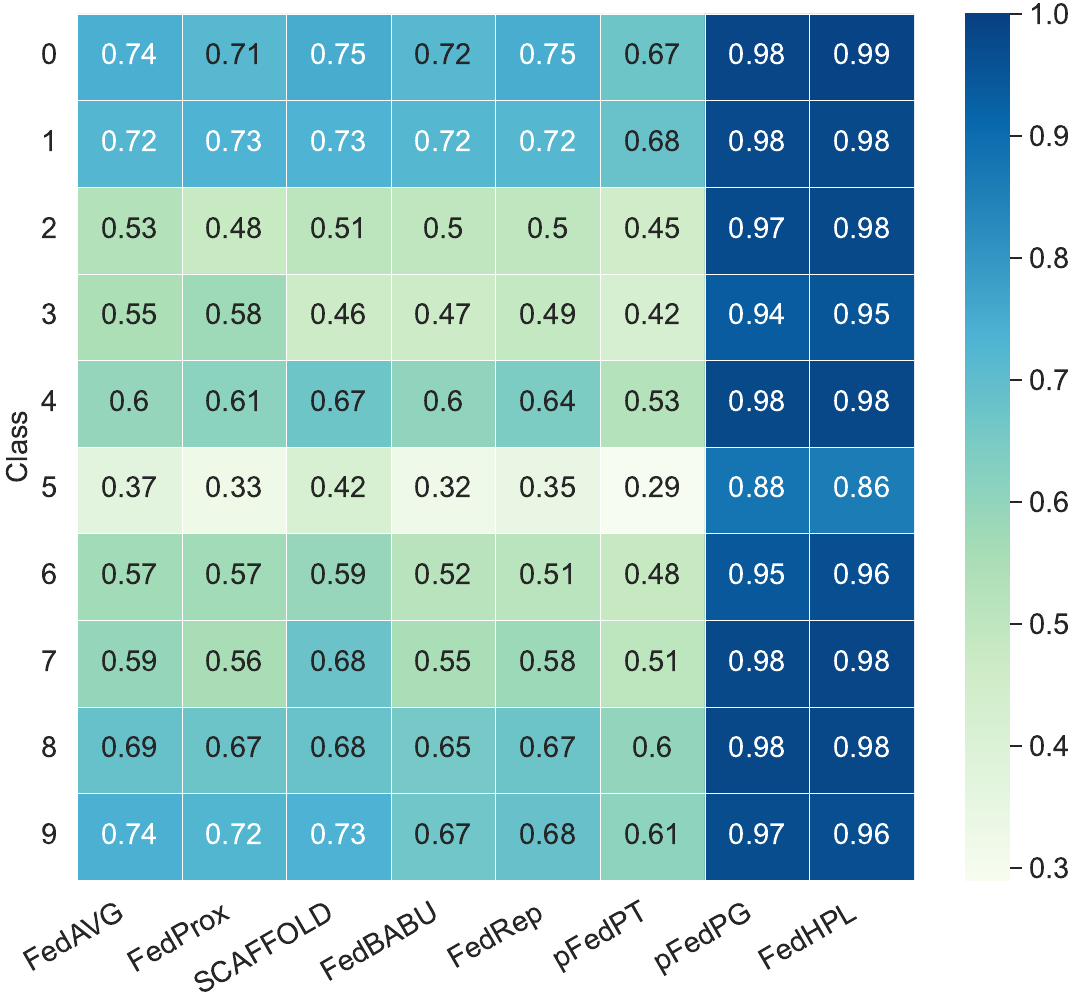}}
    \subfloat[CIFAR10-Non-IID]{\includegraphics[width=.32\linewidth]{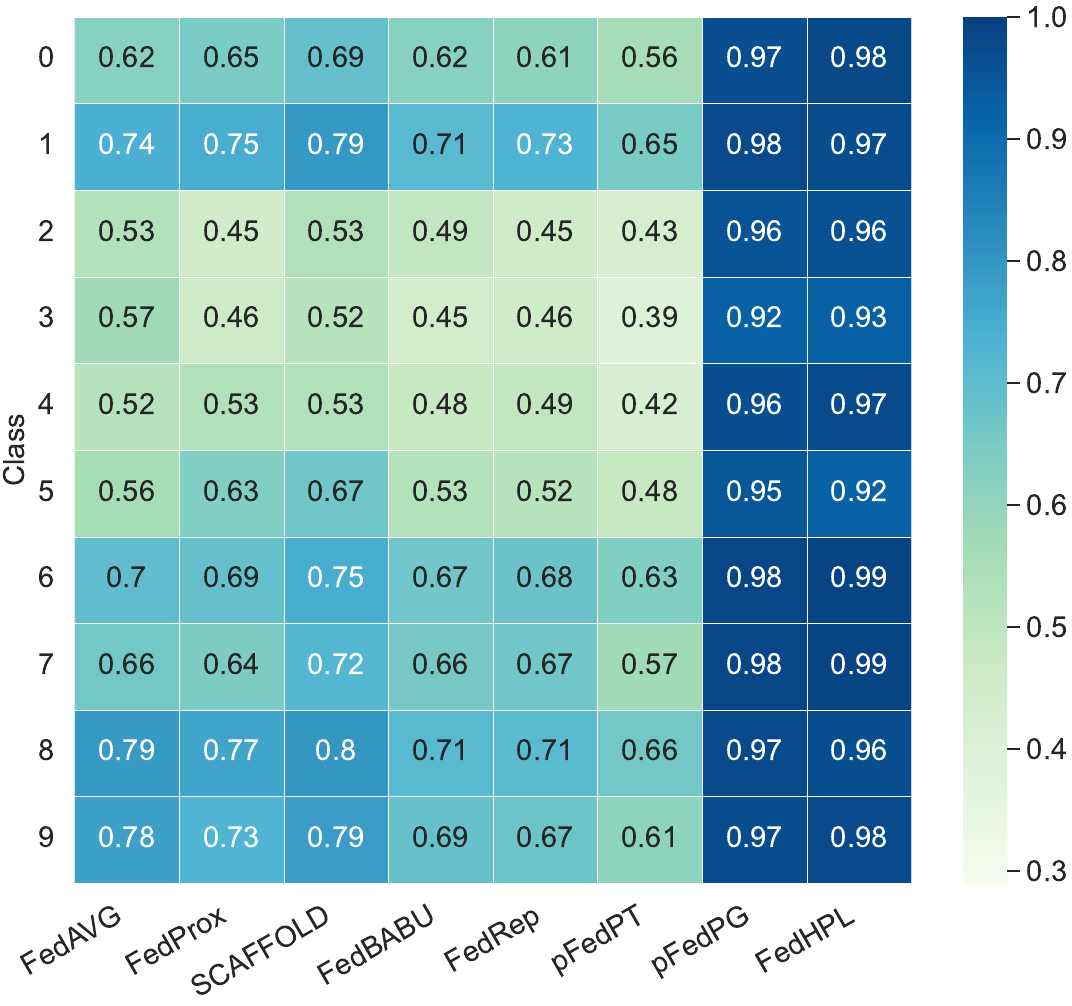}}\\
    \subfloat[CIFAR100-IID]{\includegraphics[width=.32\linewidth]{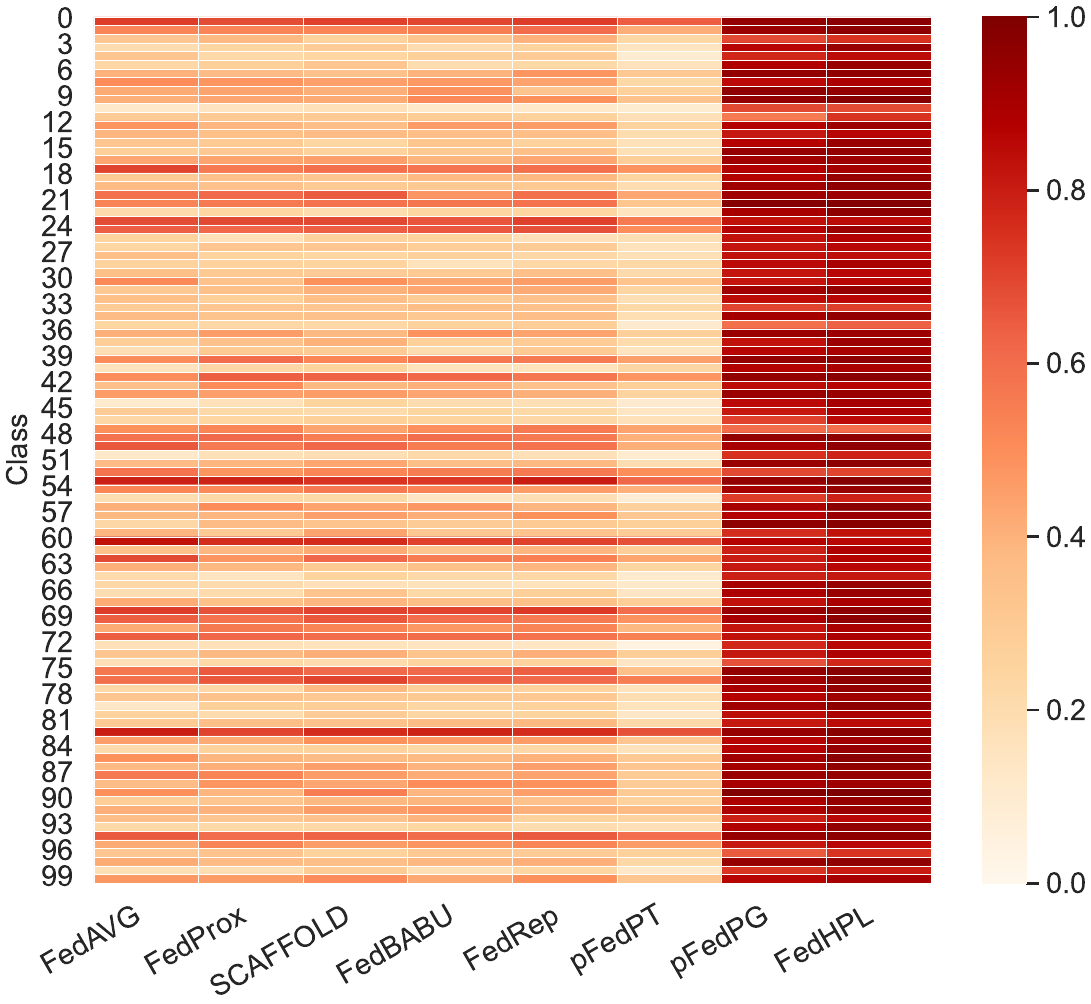}}
    \subfloat[CIFAR100-Dir]{\includegraphics[width=.32\linewidth]{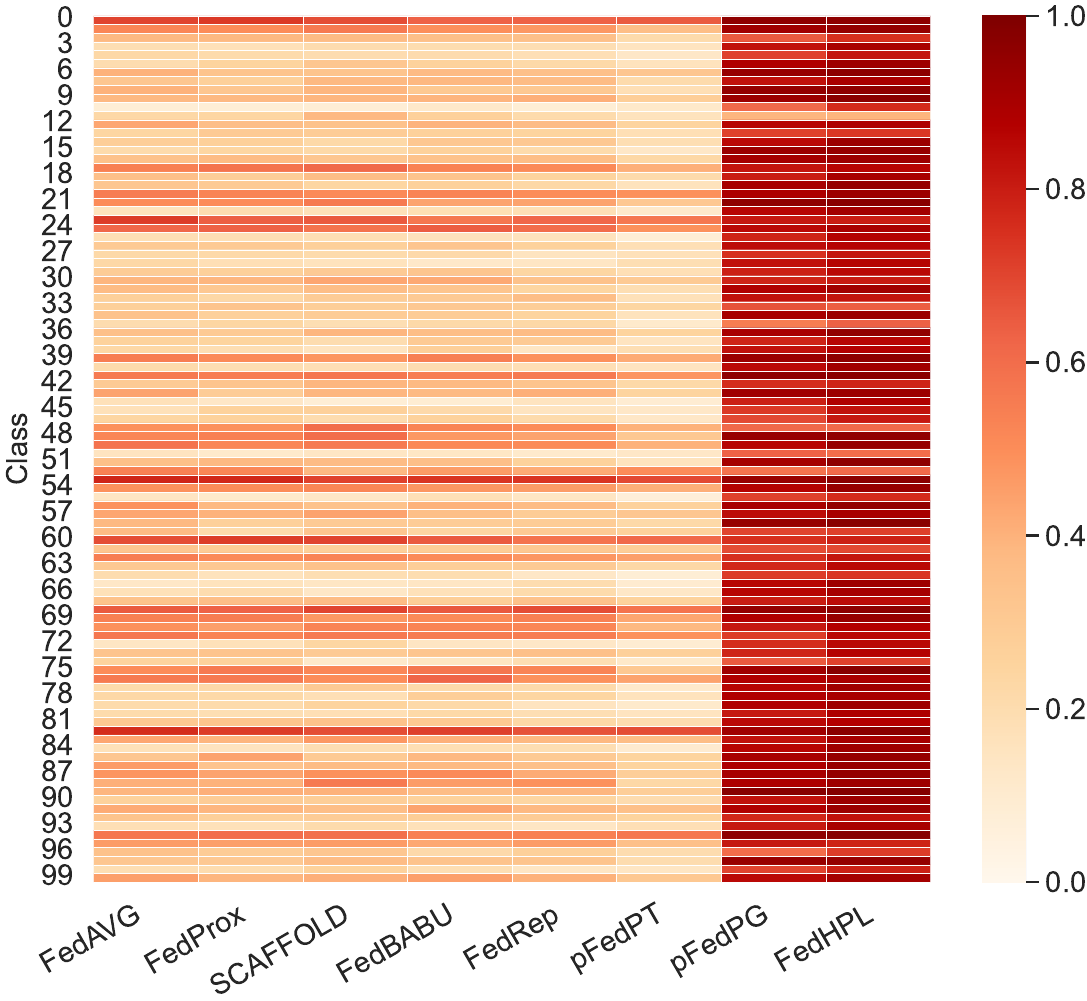}}
    \subfloat[CIFAR100-Non-IID]{\includegraphics[width=.32\linewidth]{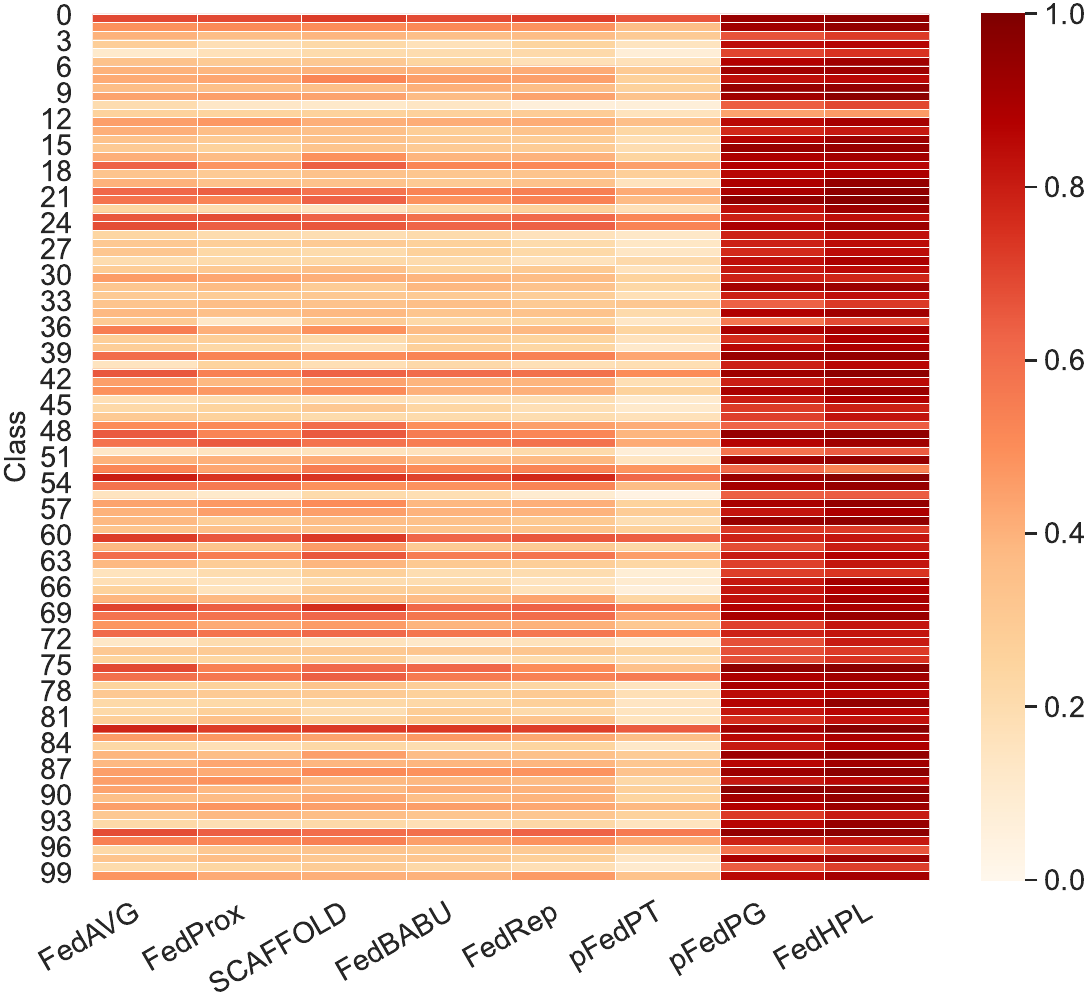}}\\
    \subfloat[SVHN-IID]{\includegraphics[width=.32\linewidth]{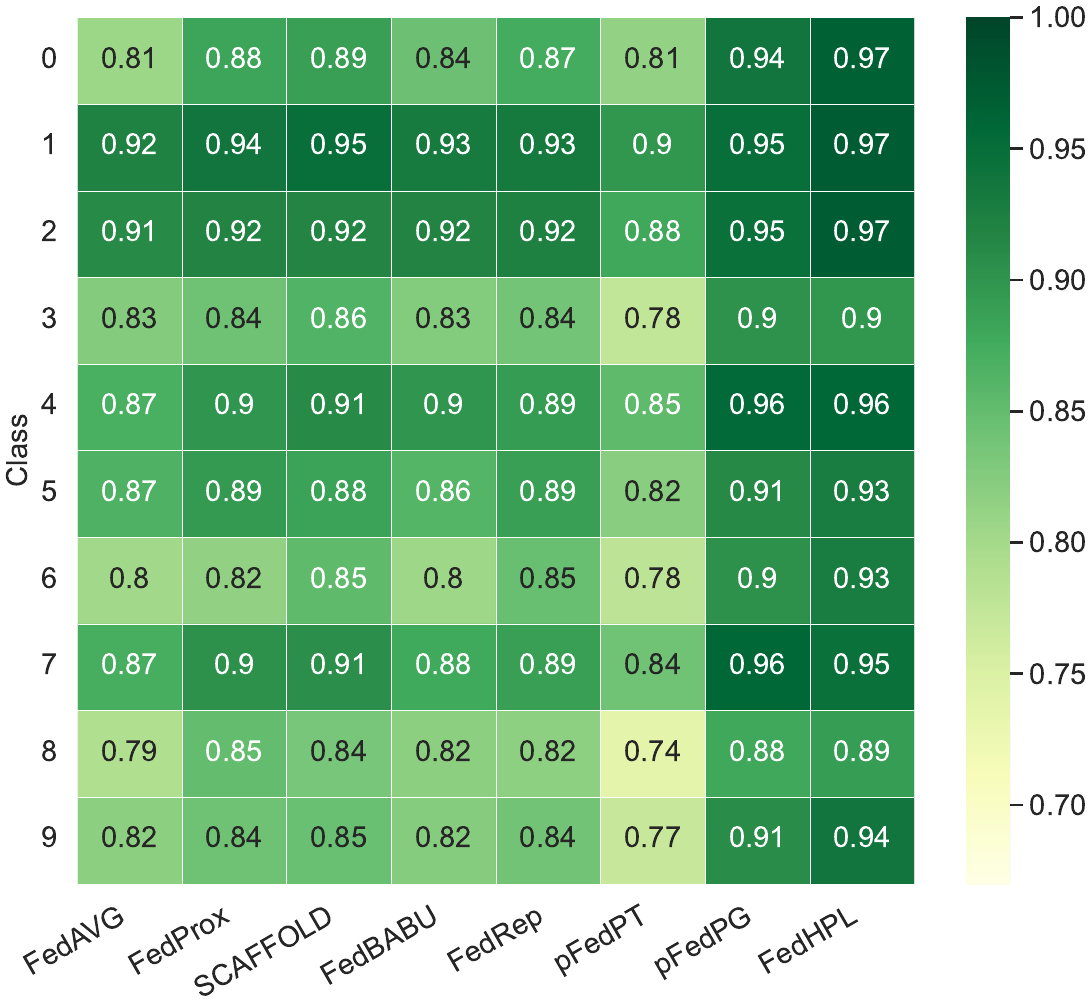}}
    \subfloat[SVHN-Dir]{\includegraphics[width=.32\linewidth]{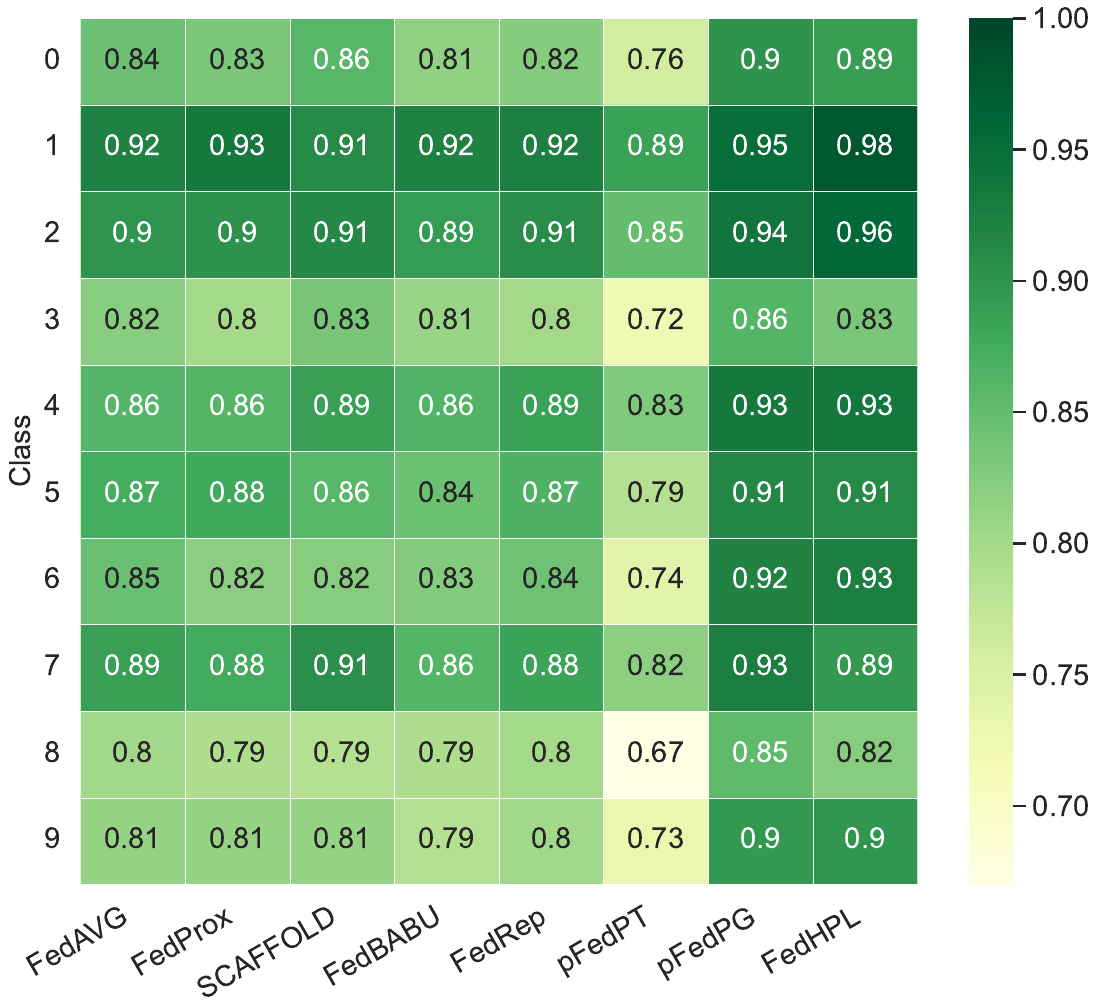}}
    \subfloat[SVHN-Non-IID]{\includegraphics[width=.32\linewidth]{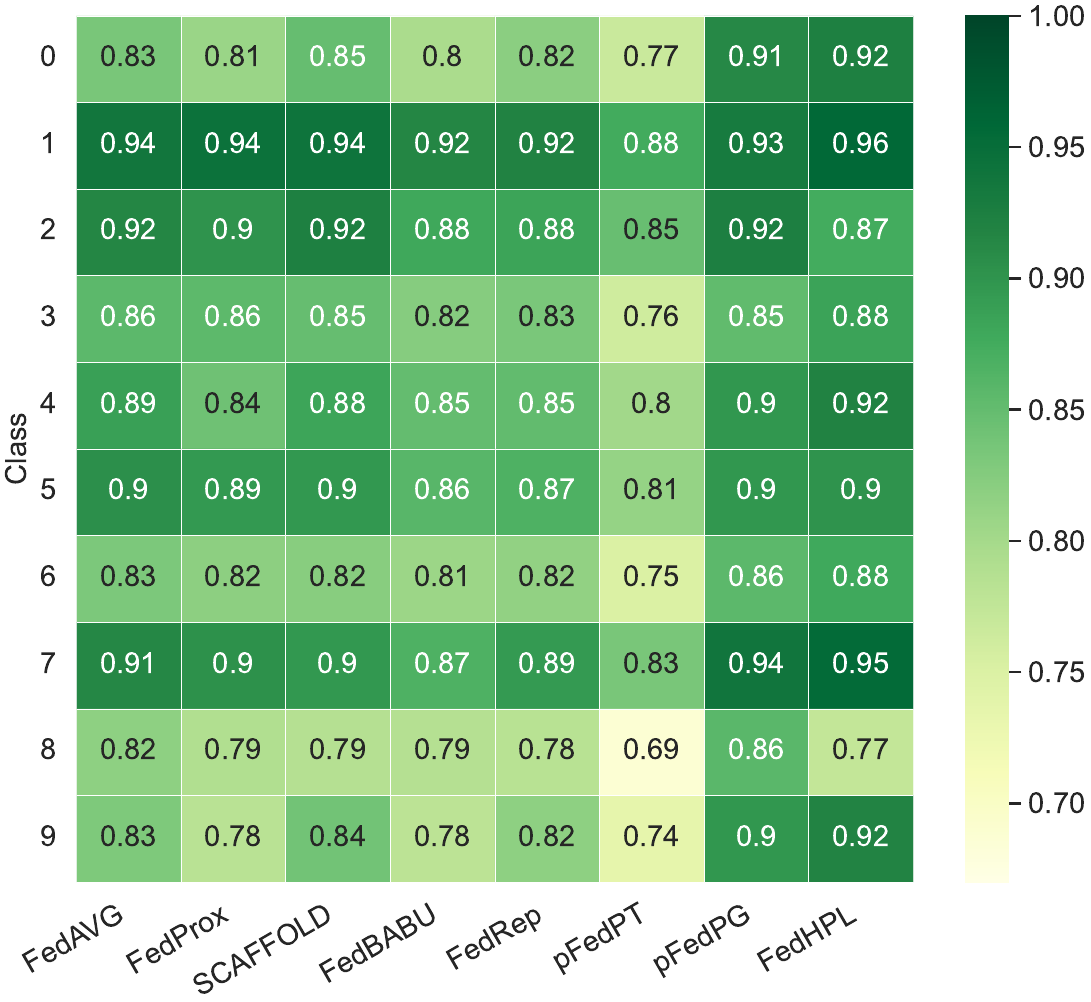}}\\
\caption{The per-class average test accuracy (\%) among all clients with baselines and FedHPL in the homogeneous model setting. The caption `A-B' represents the benchmark dataset and dataset setting. In order to visually demonstrate the test accuracy, we adjust the lowest test accuracy from 0 to 29\% in CIFAR10 and 0 to 67\% over SVHN.}
\label{fig:homo-model-label-acc}
\end{figure}

\textbf{The details of Table~\ref{tab:hete_model}.}
As illustrated in Figure~\ref{fig:hete-model-label-acc}, the per-class average test accuracy among clients in FedHPL is higher than other baselines over CIFAR10 and CIFAR100, especially for CIFAR100 which has more classes with fewer per-class samples.
We also notice that \textbf{our FedHPL with ResNet over the SVHN dataset is not ideal} because VPT cannot show the strong representation ability in a small CNN backbone according to~\cite{VPT}.
Moreover, we think the performance degradation may also be related to the SVHN dataset itself.
Prompt tuning is padded around the original image with pixels in FedHPL over ResNet and prompts cannot capture the core information well, leading to bad test accuracy.
Notably, there exists \textbf{significant test performance difference} between our experiments and the original results of some approaches (\ie FedGen, FedGH, FedProto, and FedTGP) under the entire test dataset, and we perform some experiments to analyze the reason.

\begin{figure}[htbp]
    \centering
    \subfloat[CIFAR10-IID]{\includegraphics[width=.32\linewidth]{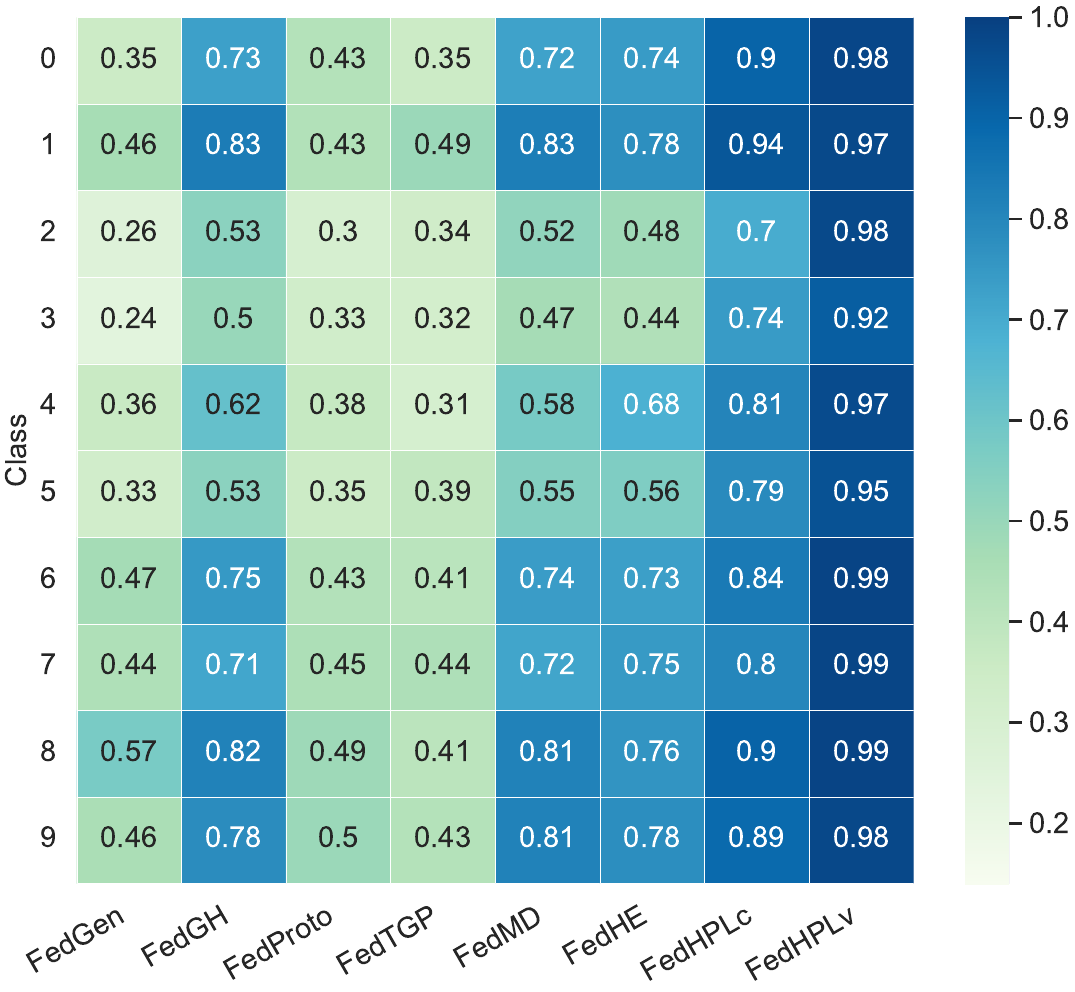}}
    \subfloat[CIFAR10-Dir]{\includegraphics[width=.32\linewidth]{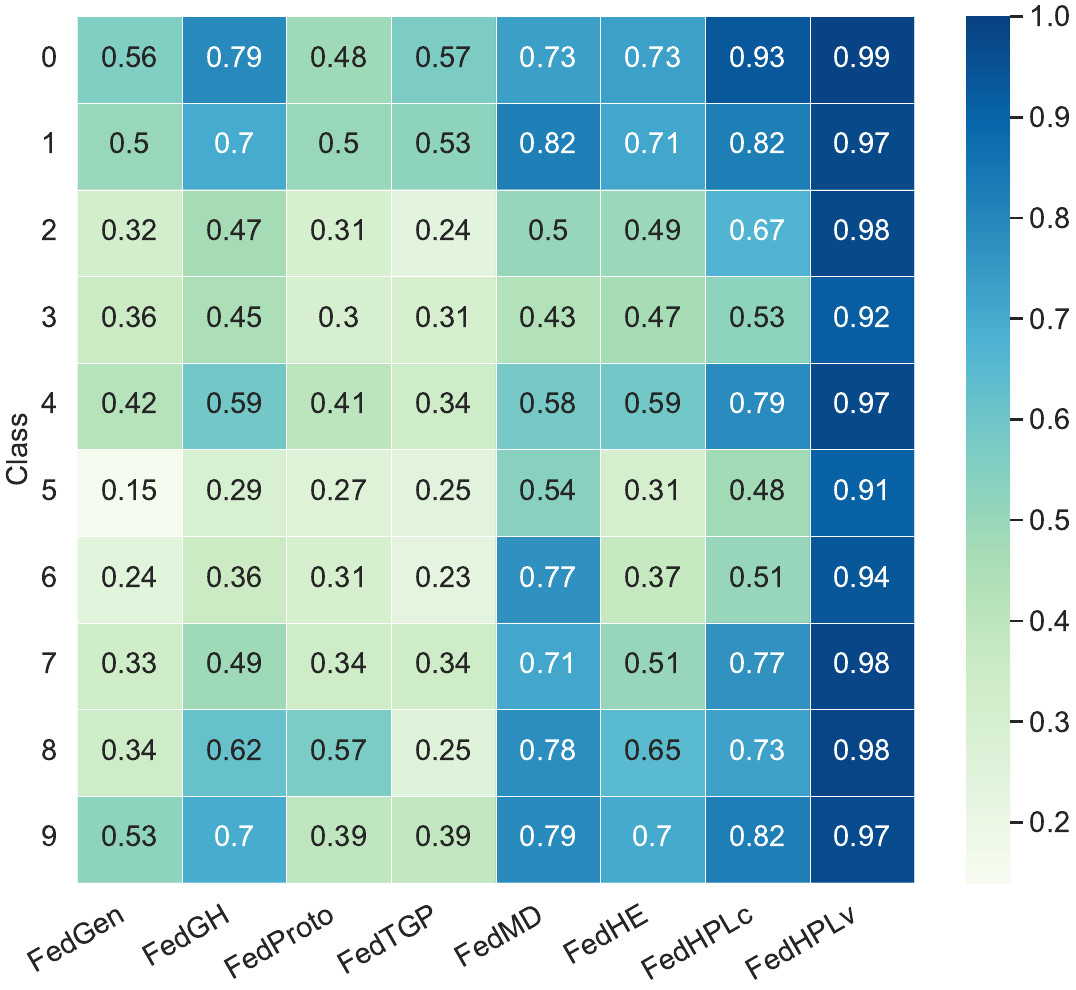}}
    \subfloat[CIFAR10-Non-IID]{\includegraphics[width=.32\linewidth]{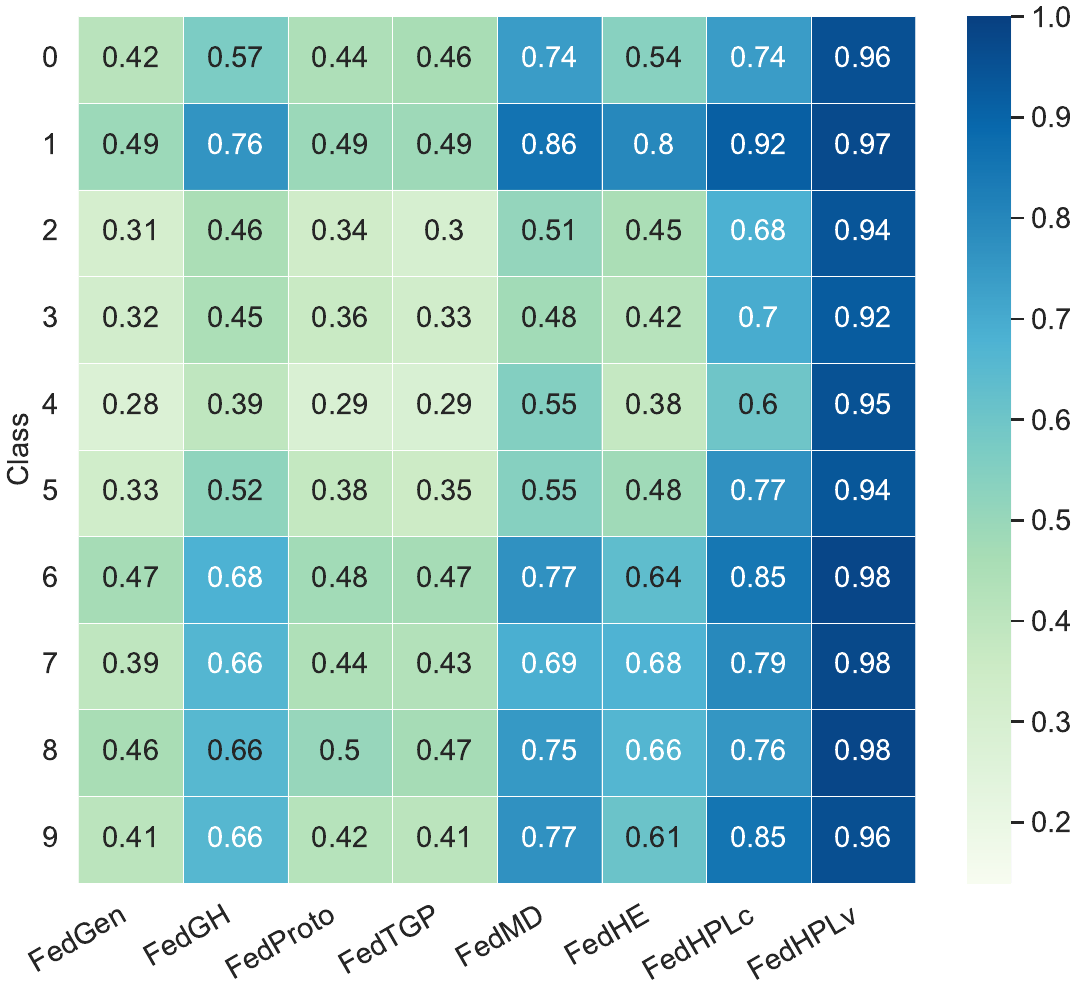}}\\
    \subfloat[CIFAR100-IID]{\includegraphics[width=.32\linewidth]{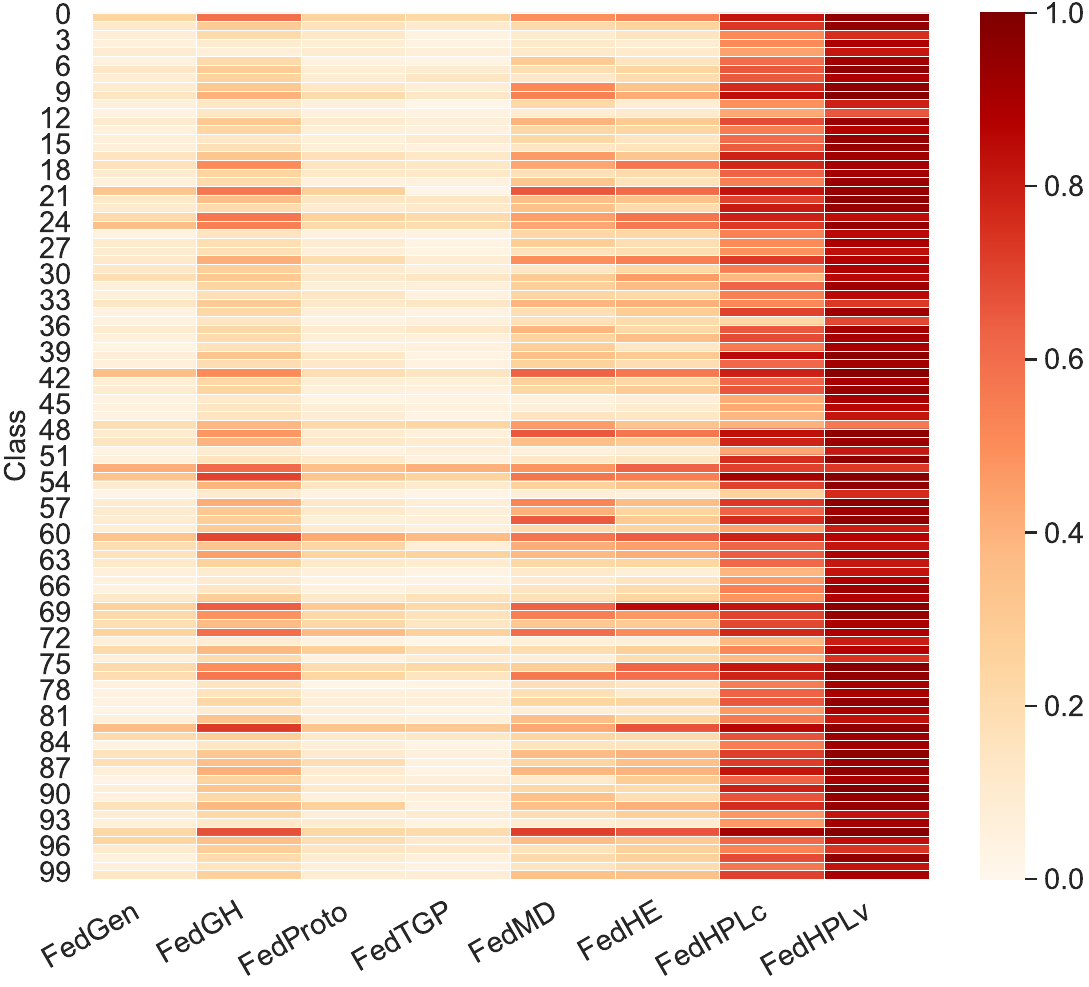}}
    \subfloat[CIFAR100-Dir]{\includegraphics[width=.32\linewidth]{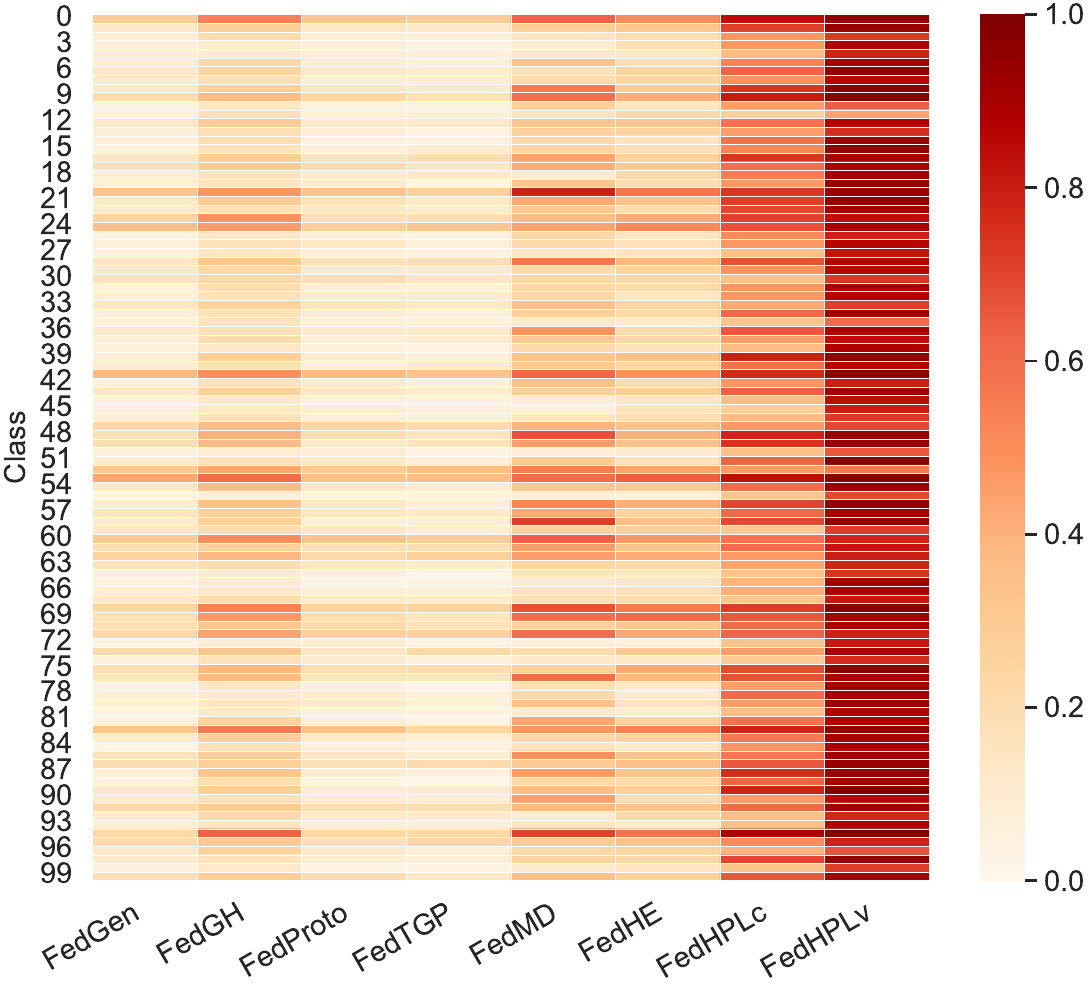}}
    \subfloat[CIFAR100-Non-IID]{\includegraphics[width=.32\linewidth]{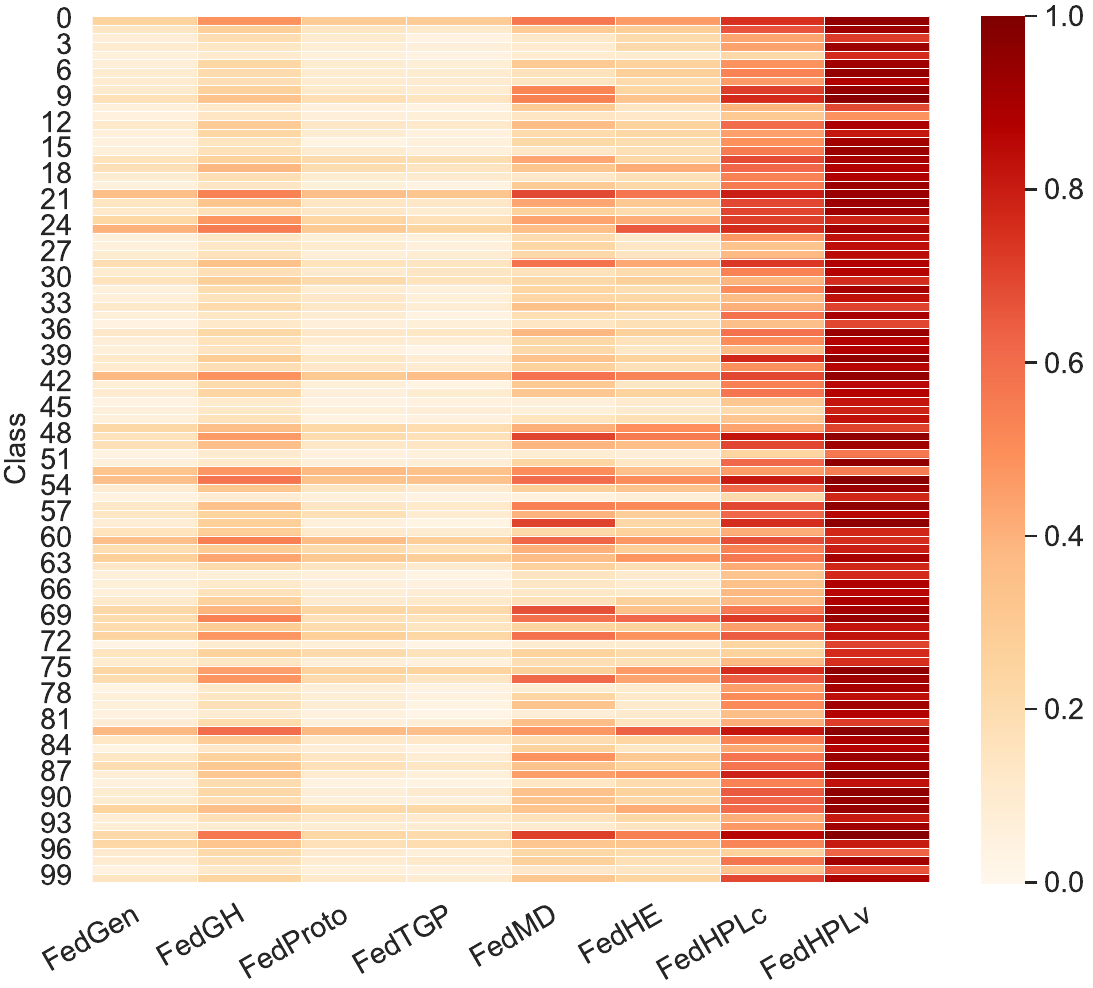}}\\
    \subfloat[SVHN-IID]{\includegraphics[width=.32\linewidth]{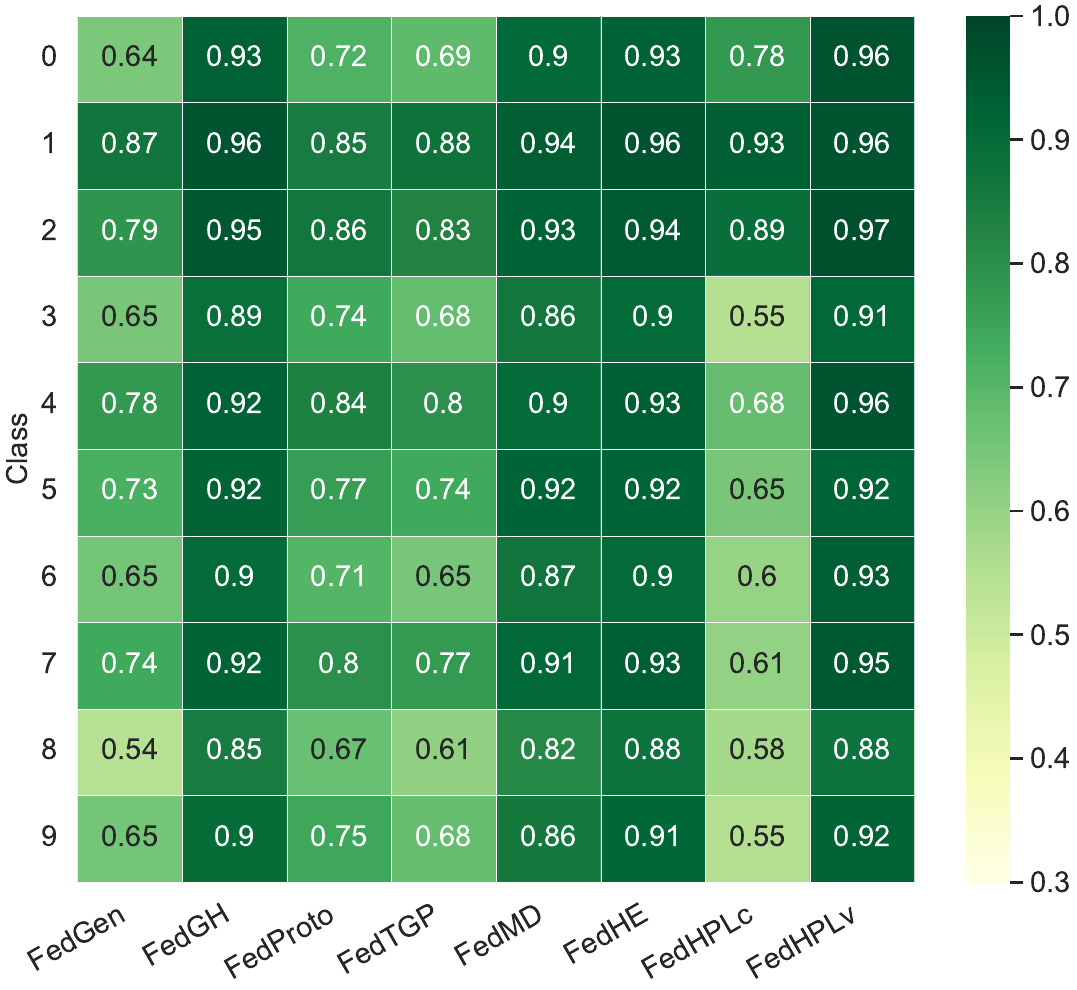}}
    \subfloat[SVHN-Dir]{\includegraphics[width=.32\linewidth]{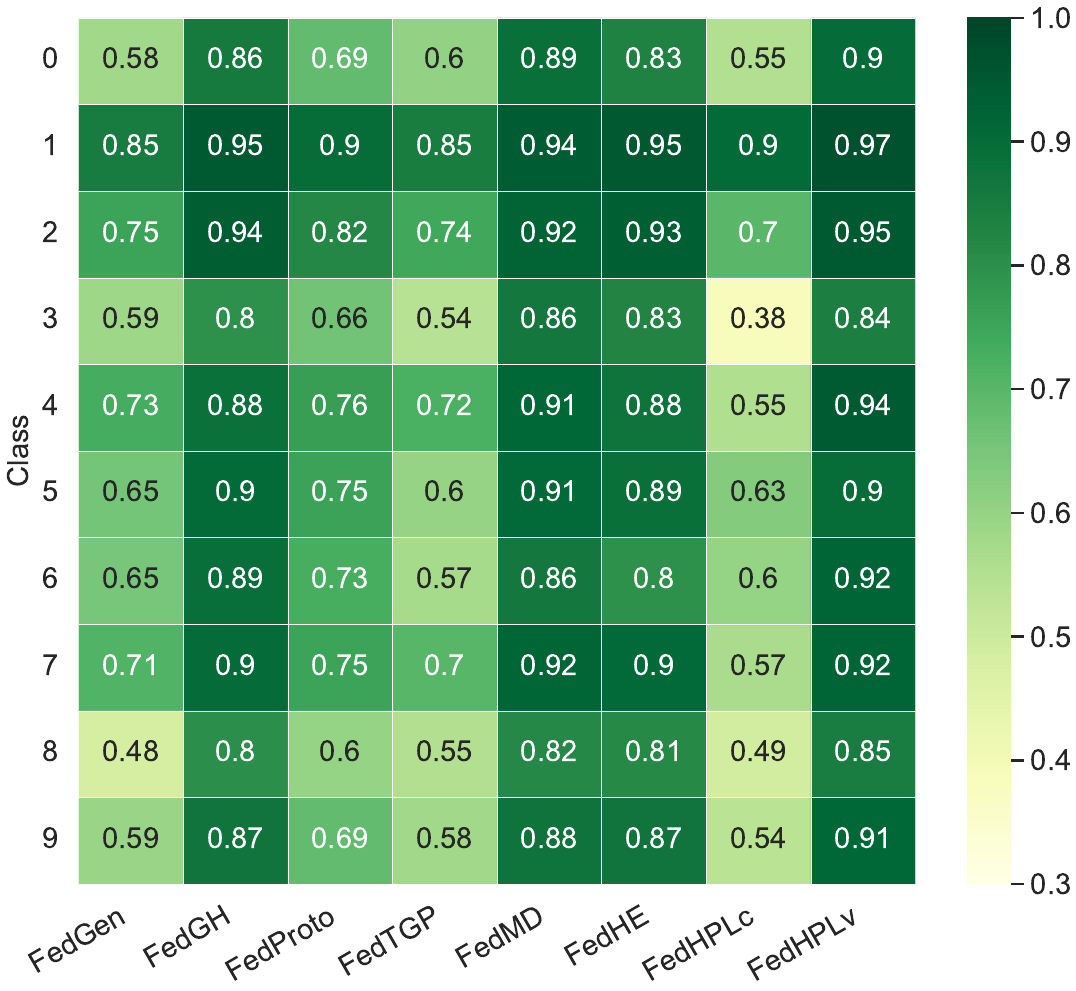}}
    \subfloat[SVHN-Non-IID]{\includegraphics[width=.32\linewidth]{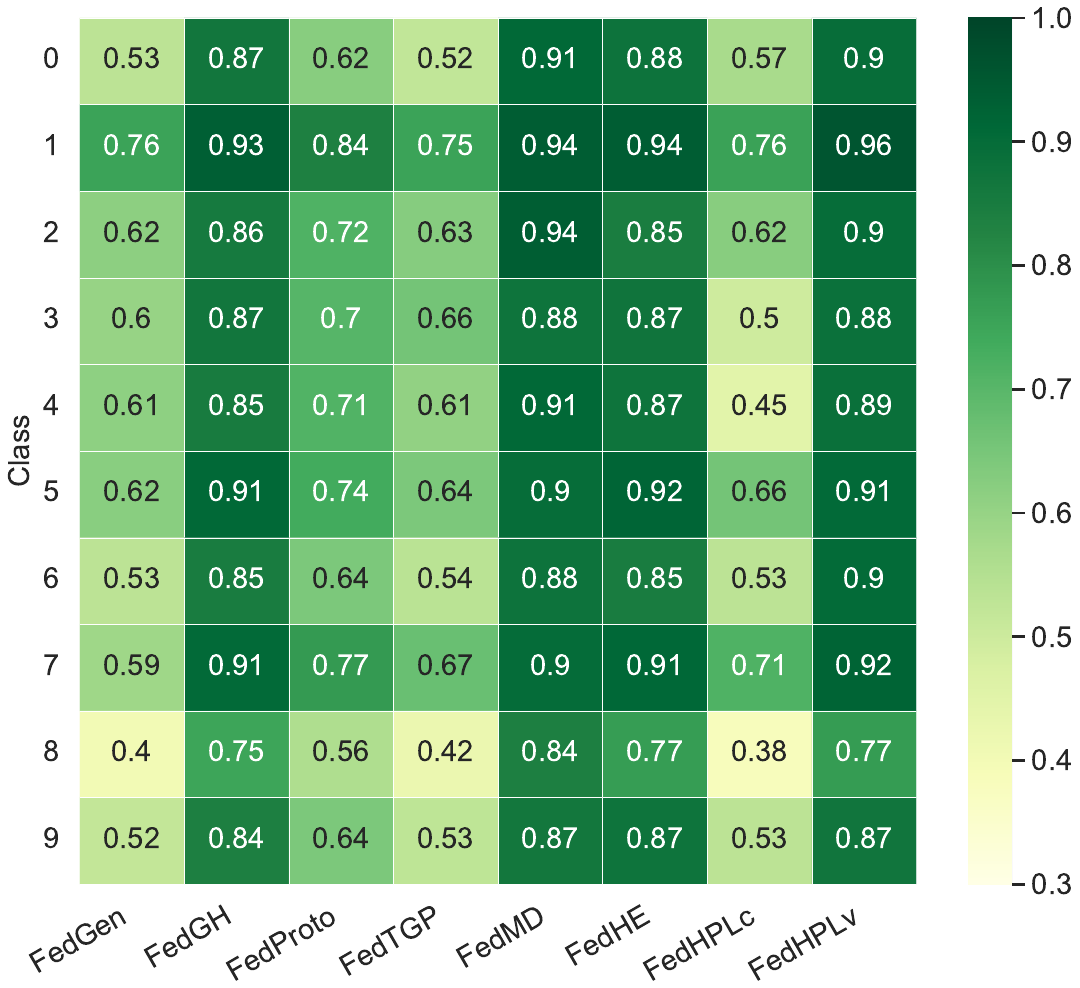}}\\
\caption{The per-class average test accuracy (\%) among clients with baselines and FedHPL in the heterogeneous model setting. In order to visually demonstrate the test accuracy, we adjust the lowest test accuracy from 0 to 14\% in CIFAR10 and 0 to 30\% over SVHN. FedHPLc and FedHPLv represent FedHPL (CNN) and FedHPL (ViT) in Table~\ref{tab:hete_model}.}
\label{fig:hete-model-label-acc}
\end{figure}

\textbf{The performance difference comes from the testing methods.}
We consider that the performance difference between the original paper of these baselines and our experiments is caused by the different testing method.
We use the entire test dataset with all labels while the original methods split the training dataset of benchmark datasets into clients' training set and clients' test set according to the same class space and only test performance on the private test set.
Thus, we apply their original partitioning techniques and split different training and test datasets.
Specifically, we firstly treat the training dataset of benchmark datasets as the private data of clients and split per-class samples into training sets (80\%) and test sets (20\%).
Then, we distribute samples with \{2, 4, 6, 8, 10\} classes into each client on the CIFAR10 and SVHN datasets, and allocate samples with \{10, 30, 50, 70, 100\} classes into each client on the CIFAR100 dataset, respectively. 
Figure~\ref{fig:labels_baselines_CIFAR10}, Figure~\ref{fig:labels_baselines_CIFAR100}, and Figure~\ref{fig:labels_baselines_SVHN} exhibit the test accuracy of each client and the average test accuracy in FedGen, FedGH, FedProto, and FedTGP over the heterogeneous ResNet model setting. 
It reveals that the test accuracy usually rises with the decrease in the number and category of test samples, suggesting that the partition of the test dataset is the reason that causes the performance difference.
However, the generalization performance of the model still has certain flaws.

\textbf{The details of selecting ViT as the pre-trained backbone.}
Figure~\ref{fig:res_vit_local_our} shows the per-class average client accuracy in different heterogeneous model settings.
We can observe that the performance of ResNet is lower than the performance of ViT backbones.
Additionally, it can be observed that the model utility with the original image size of 32 $\times$ 32 is not ideal because the pre-trained foundation models were trained on ImageNet with the input size of 224 $\times$ 224. 
In summary, this demonstrates that the ViT is a better choice for the pre-trained backbone.

\begin{figure}[htbp]
    \centering
    \subfloat[FedGen-IID]{\includegraphics[width=.24\linewidth]{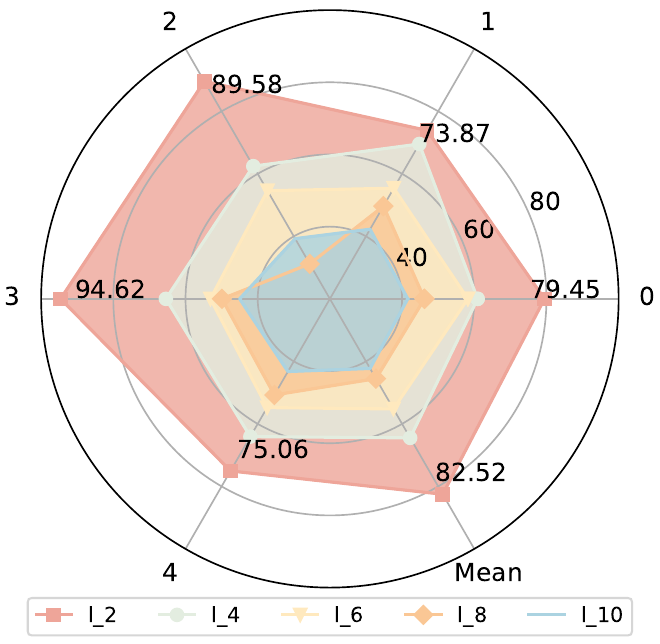}}
    \subfloat[FedGH-IID]{\includegraphics[width=.24\linewidth]{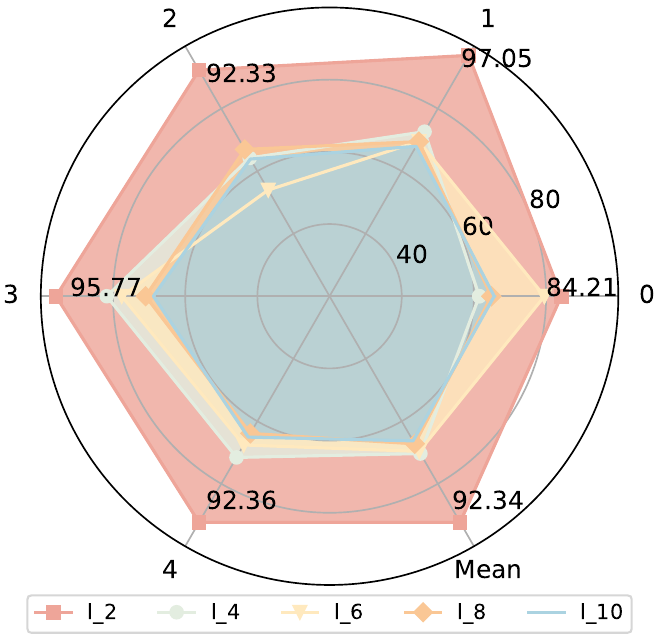}}
    \subfloat[FedProto-IID]{\includegraphics[width=.24\linewidth]{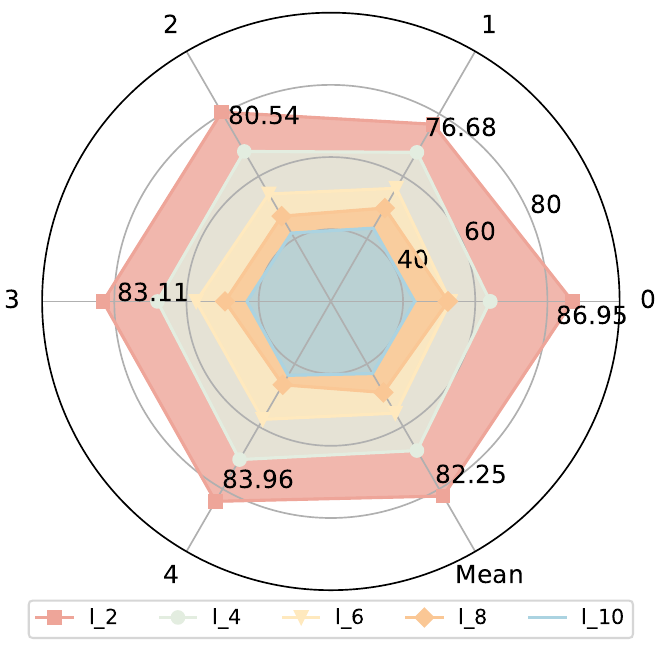}}
    \subfloat[FedTGP-IID]{\includegraphics[width=.24\linewidth]{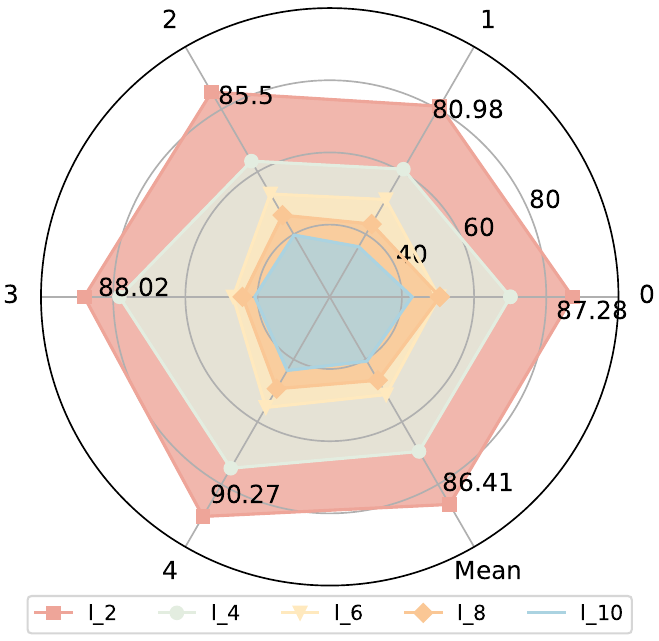}}\\
    \subfloat[FedGen-Dir]{\includegraphics[width=.24\linewidth]{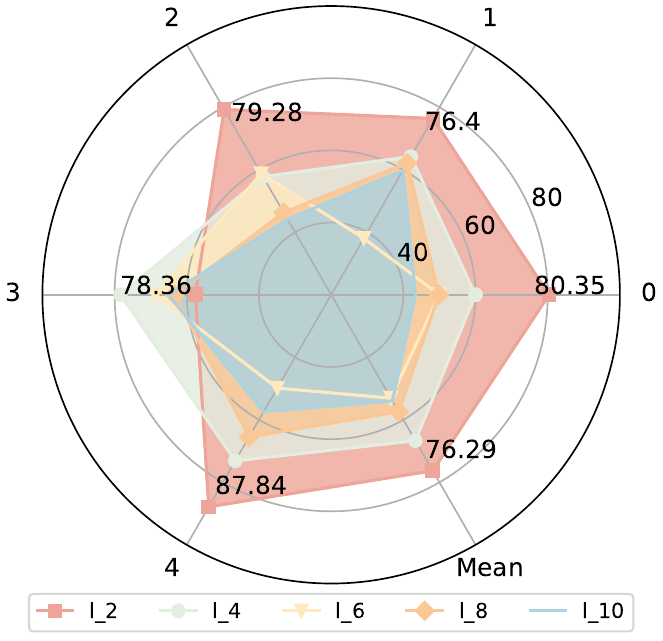}}
    \subfloat[FedGH-Dir]{\includegraphics[width=.24\linewidth]{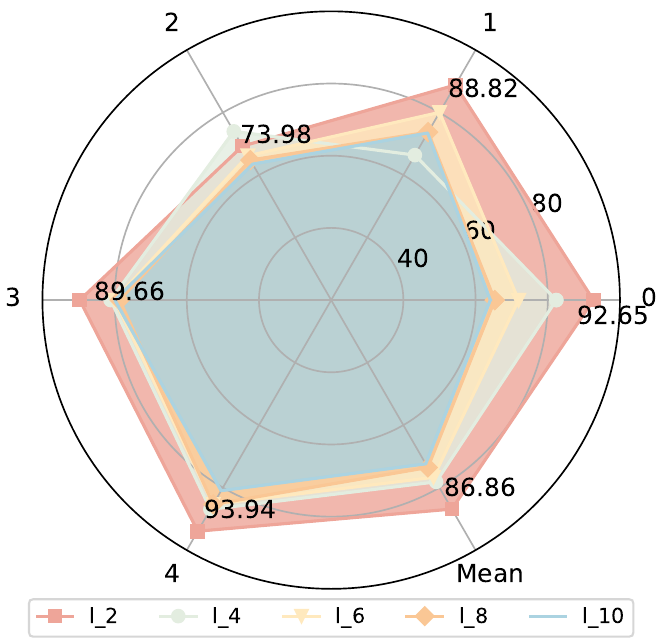}}
    \subfloat[FedProto-Dir]{\includegraphics[width=.24\linewidth]{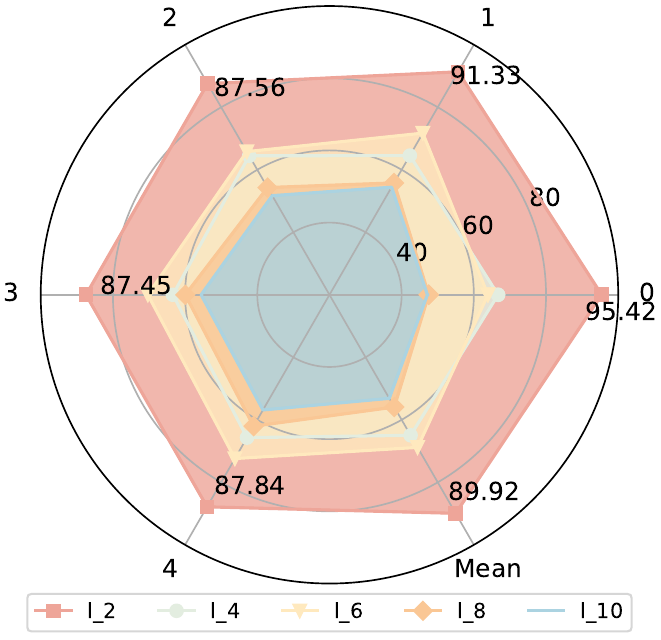}}
    \subfloat[FedTGP-Dir]{\includegraphics[width=.24\linewidth]{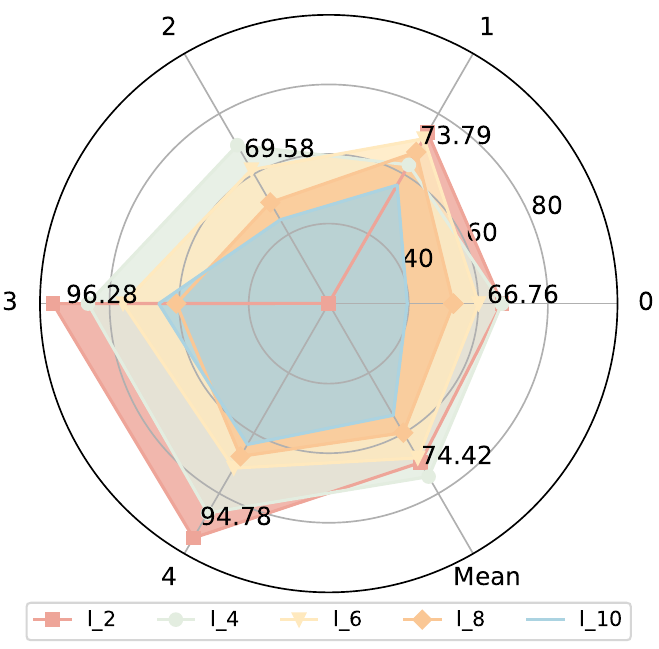}}\\
    \subfloat[FedGen-Non-IID]{\includegraphics[width=.24\linewidth]{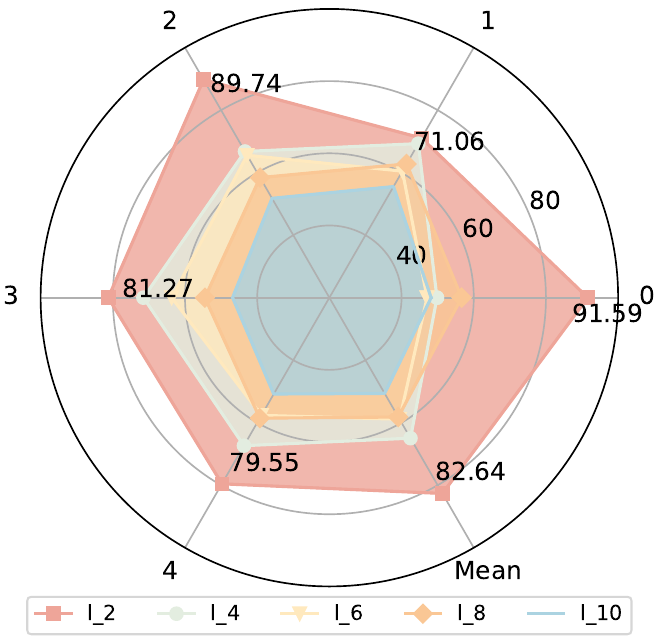}}
    \subfloat[FedGH-Non-IID]{\includegraphics[width=.24\linewidth]{figures/Leida_label_percentages/Leida_CIFAR10/CIFAR10_FedGH_Dir_label_percentage.pdf}}
    \subfloat[FedProto-Non-IID]{\includegraphics[width=.24\linewidth]{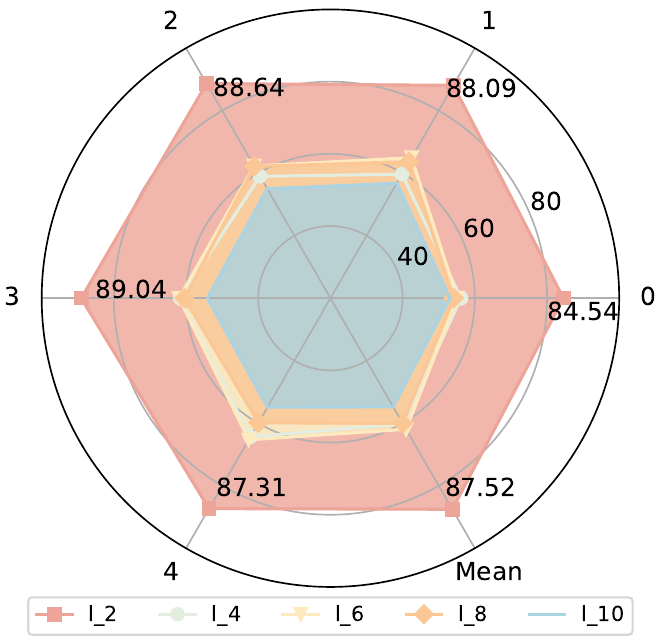}}
    \subfloat[FedTGP-Non-IID]{\includegraphics[width=.24\linewidth]{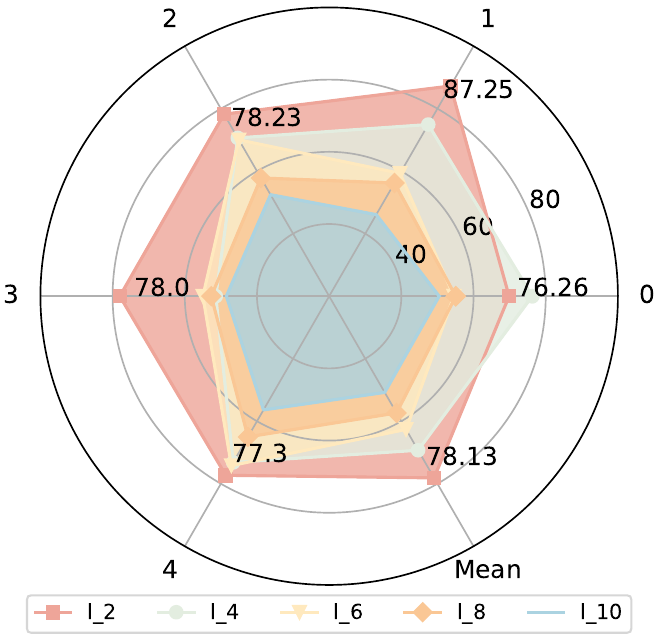}}\\
\caption{Per-client and average test accuracy (\%) on CIFAR10 in the heterogeneous ResNet setting.}
\label{fig:labels_baselines_CIFAR10}
\end{figure}

\begin{figure}[htbp]
    \centering
    \subfloat[FedGen-IID]{\includegraphics[width=.24\linewidth]{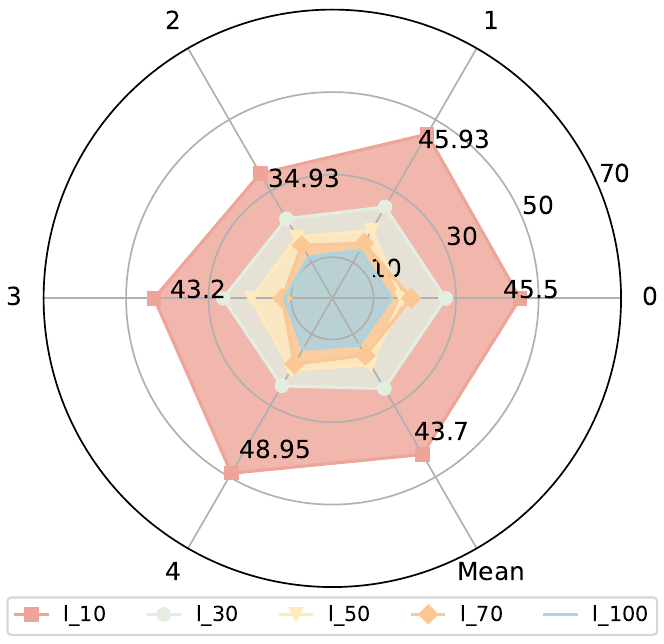}}
    \subfloat[FedGH-IID]{\includegraphics[width=.24\linewidth]{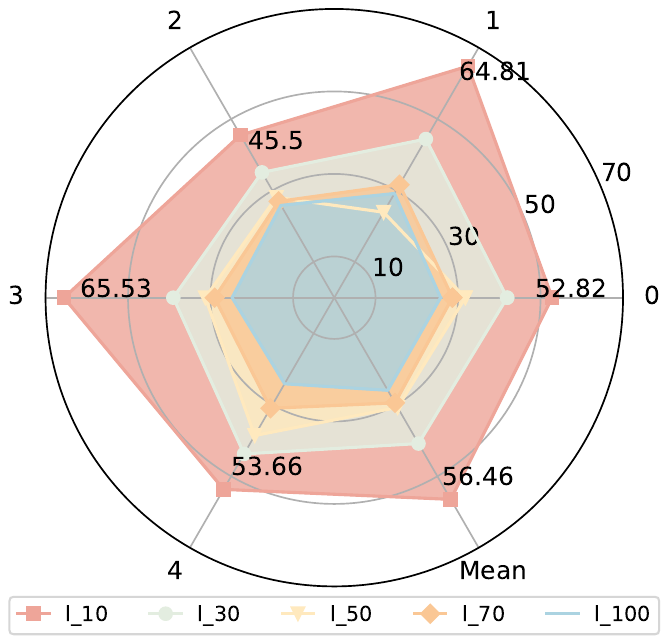}}
    \subfloat[FedProto-IID]{\includegraphics[width=.24\linewidth]{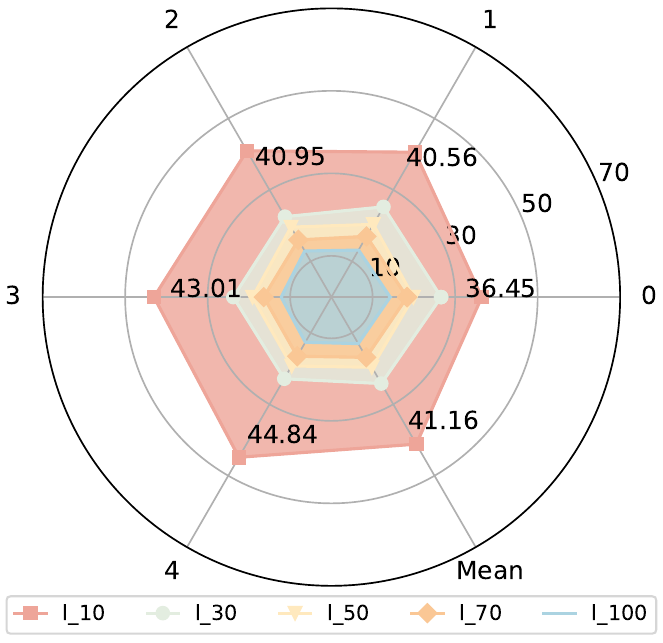}}
    \subfloat[FedTGP-IID]{\includegraphics[width=.24\linewidth]{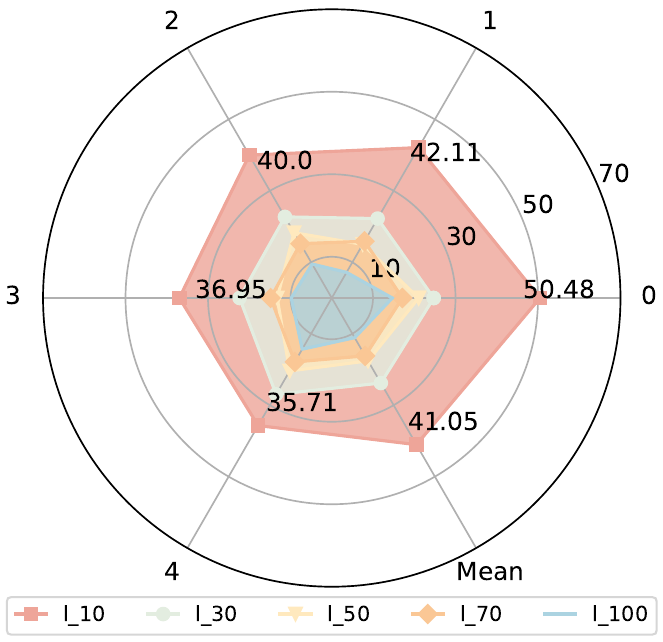}}\\
    \subfloat[FedGen-Dir]{\includegraphics[width=.24\linewidth]{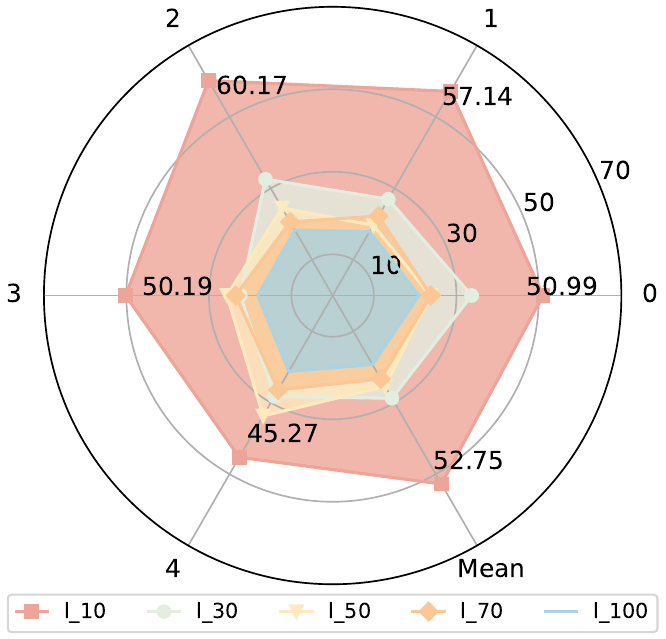}}
    \subfloat[FedGH-Dir]{\includegraphics[width=.24\linewidth]{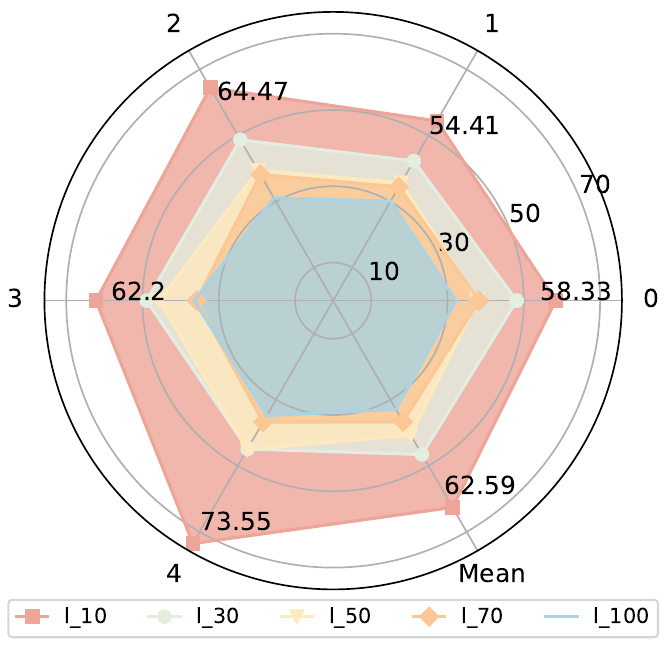}}
    \subfloat[FedProto-Dir]{\includegraphics[width=.24\linewidth]{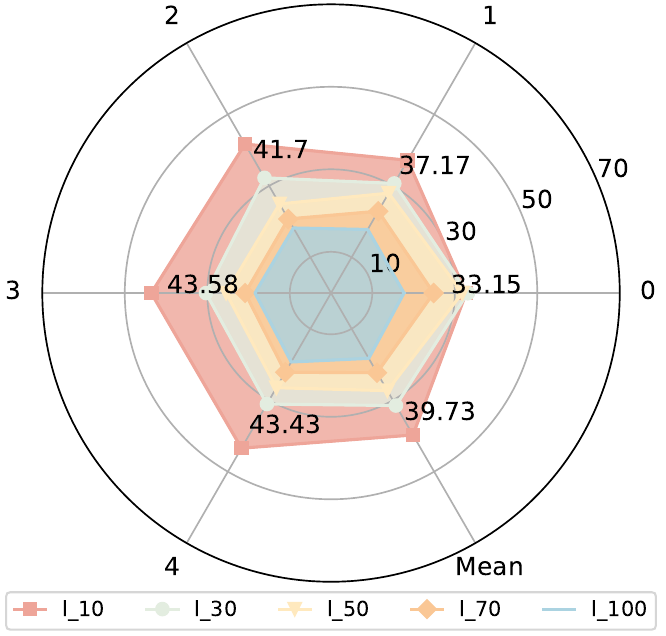}}
    \subfloat[FedTGP-Dir]{\includegraphics[width=.24\linewidth]{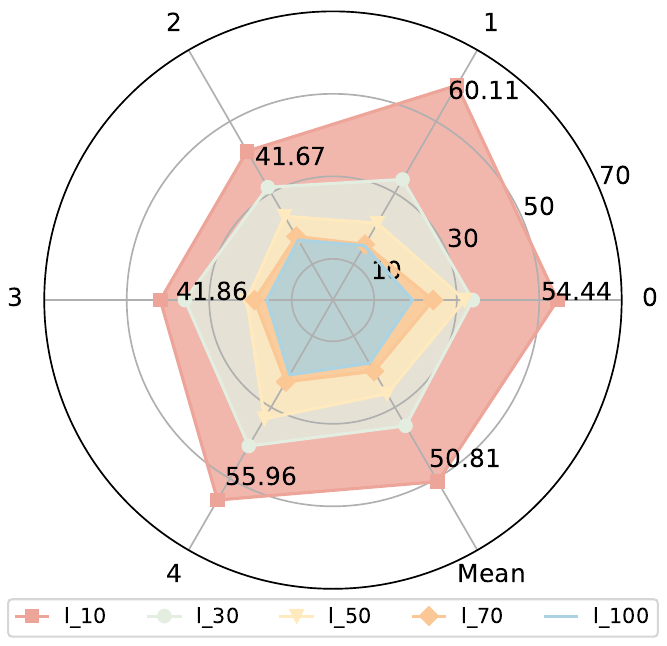}}\\
    \subfloat[FedGen-Non-IID]{\includegraphics[width=.24\linewidth]{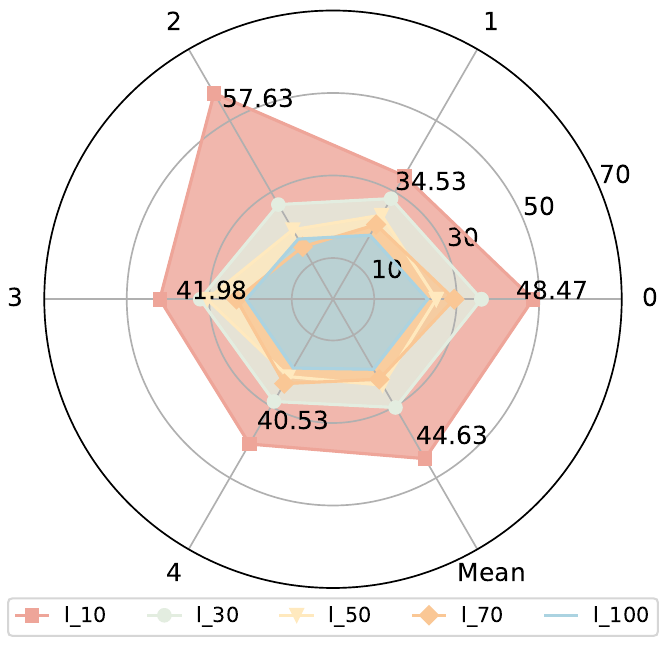}}
    \subfloat[FedGH-Non-IID]{\includegraphics[width=.24\linewidth]{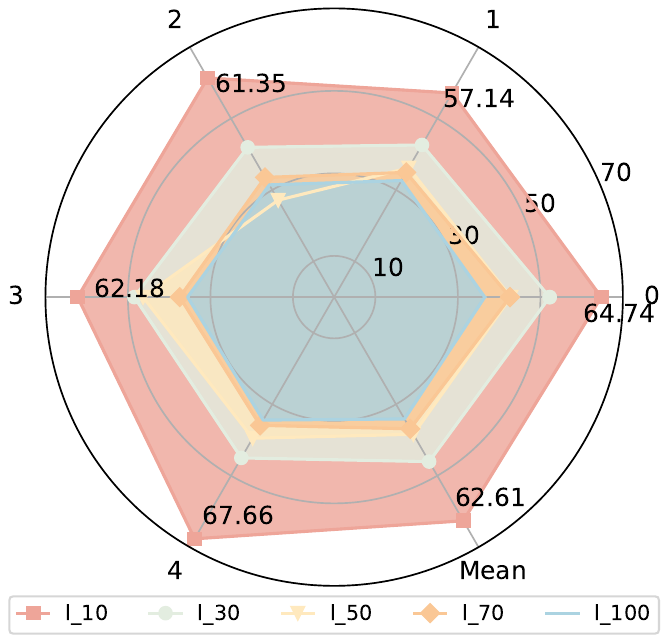}}
    \subfloat[FedProto-Non-IID]{\includegraphics[width=.24\linewidth]{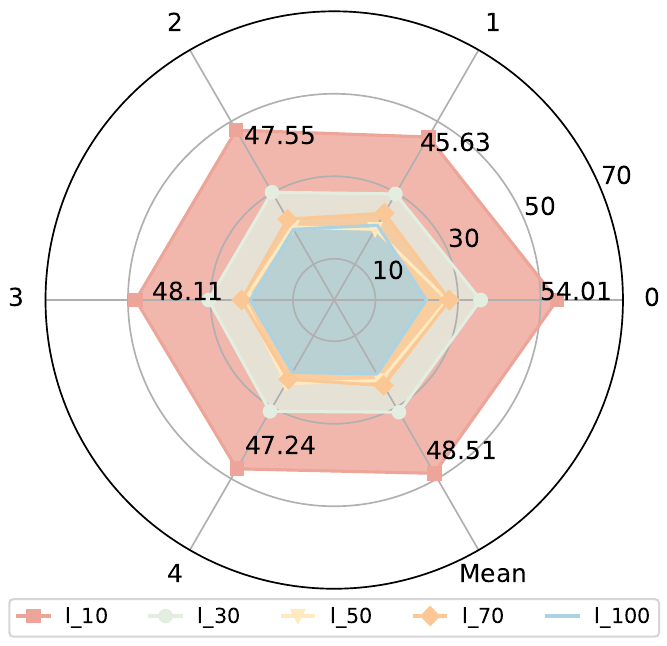}}
    \subfloat[FedTGP-Non-IID]{\includegraphics[width=.24\linewidth]{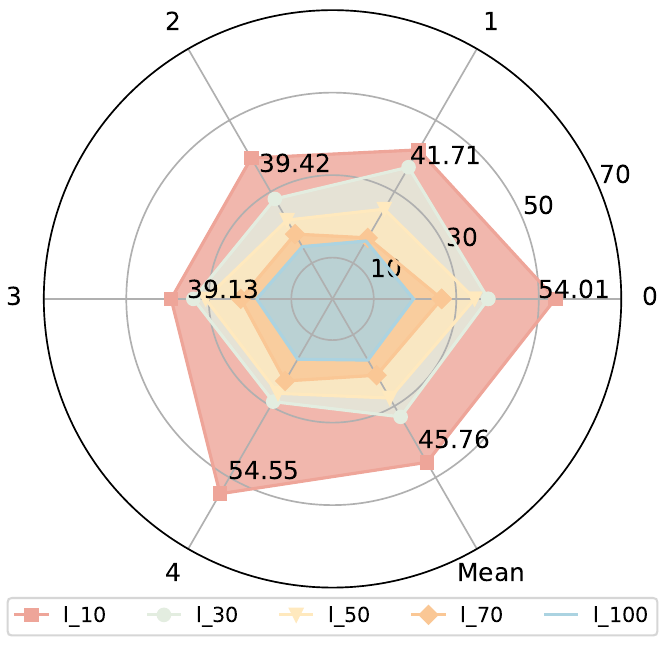}}\\
\caption{Per-client and average test accuracy (\%) on CIFAR100 in the heterogeneous ResNet setting.}
\label{fig:labels_baselines_CIFAR100}
\end{figure}

\begin{figure}[htbp]
    \centering
    \subfloat[FedGen-IID]{\includegraphics[width=.24\linewidth]{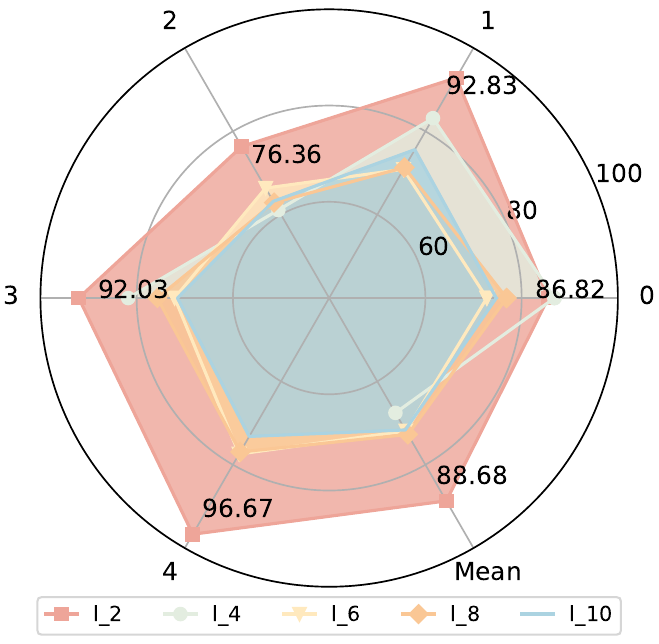}}
    \subfloat[FedGH-IID]{\includegraphics[width=.24\linewidth]{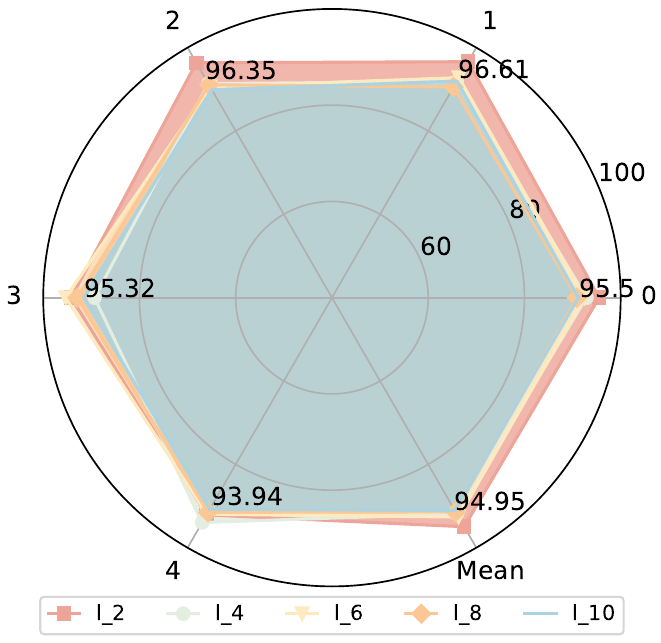}}
    \subfloat[FedProto-IID]{\includegraphics[width=.24\linewidth]{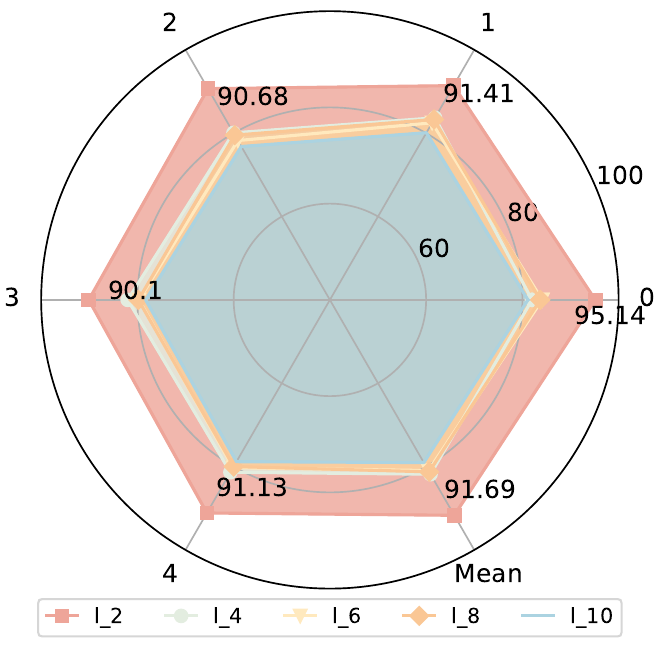}}
    \subfloat[FedTGP-IID]{\includegraphics[width=.24\linewidth]{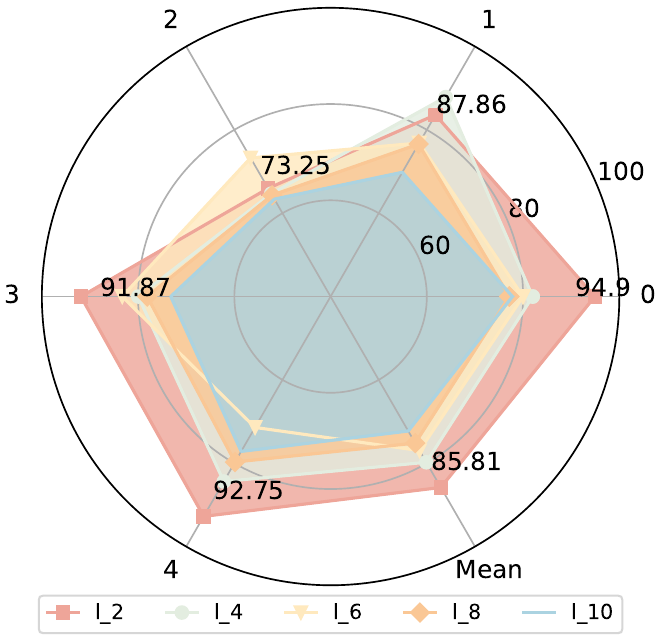}}\\
    \subfloat[FedGen-Dir]{\includegraphics[width=.24\linewidth]{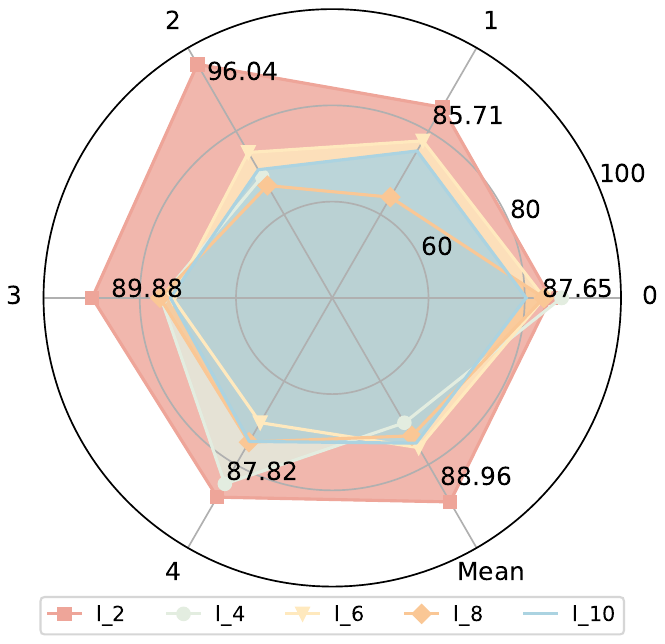}}
    \subfloat[FedGH-Dir]{\includegraphics[width=.24\linewidth]{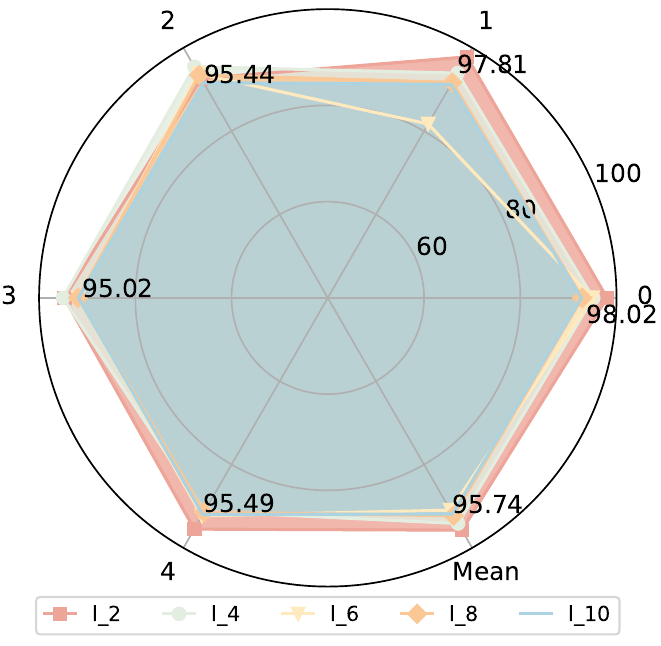}}
    \subfloat[FedProto-Dir]{\includegraphics[width=.24\linewidth]{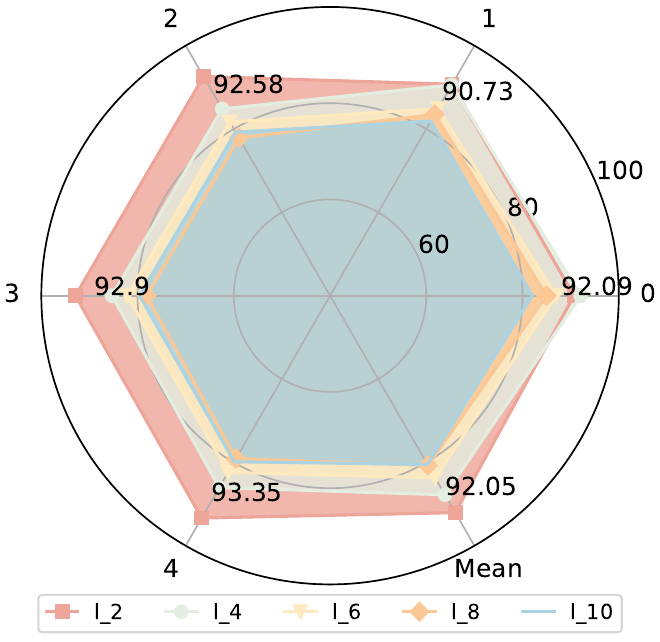}}
    \subfloat[FedTGP-Dir]{\includegraphics[width=.24\linewidth]{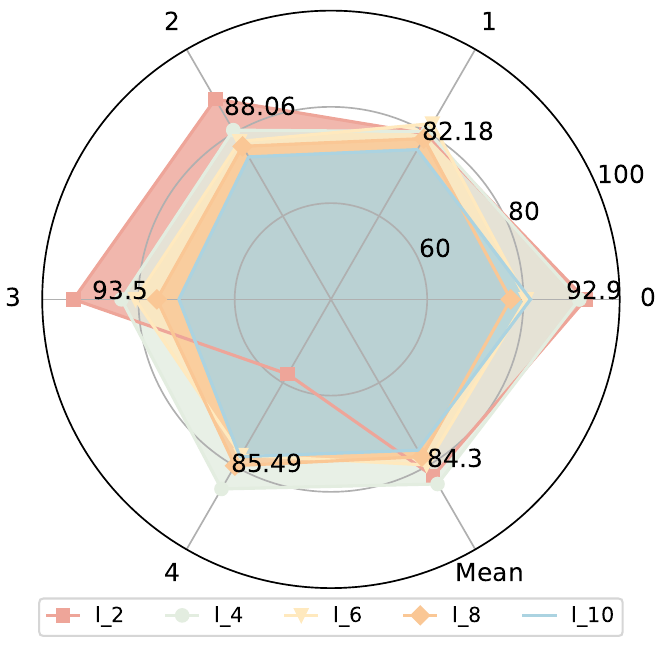}}\\
    \subfloat[FedGen-Non-IID]{\includegraphics[width=.24\linewidth]{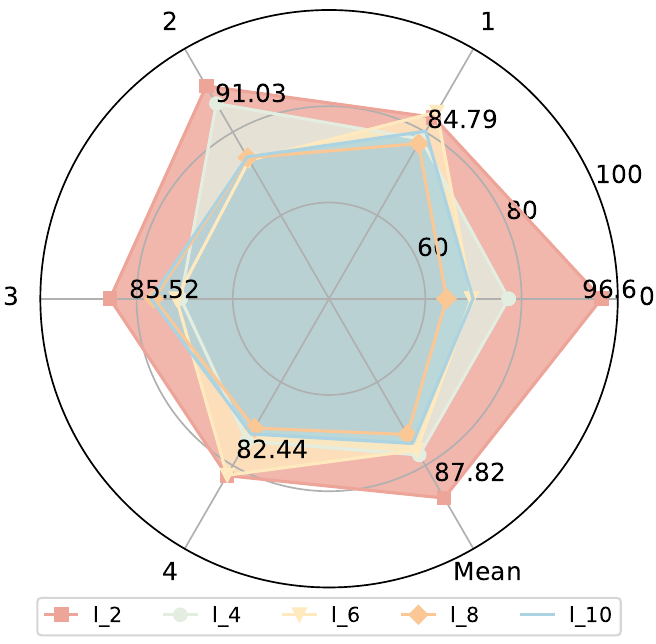}}
    \subfloat[FedGH-Non-IID]{\includegraphics[width=.24\linewidth]{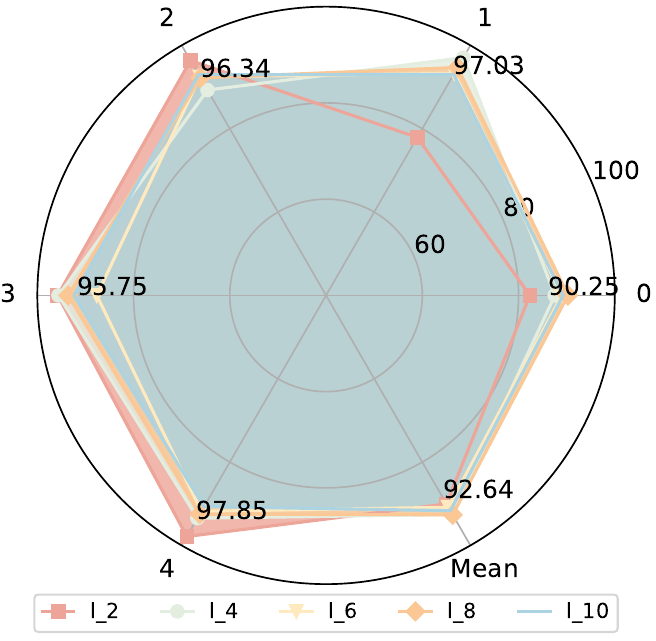}}
    \subfloat[FedProto-Non-IID]{\includegraphics[width=.24\linewidth]{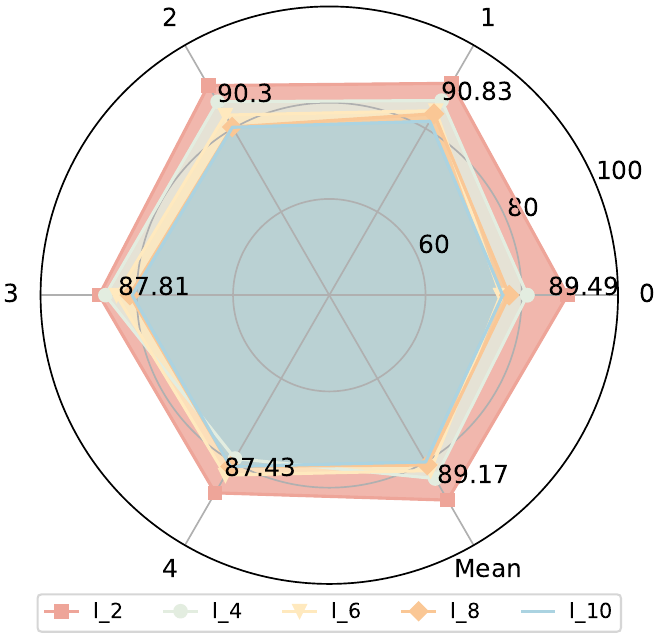}}
    \subfloat[FedTGP-Non-IID]{\includegraphics[width=.24\linewidth]{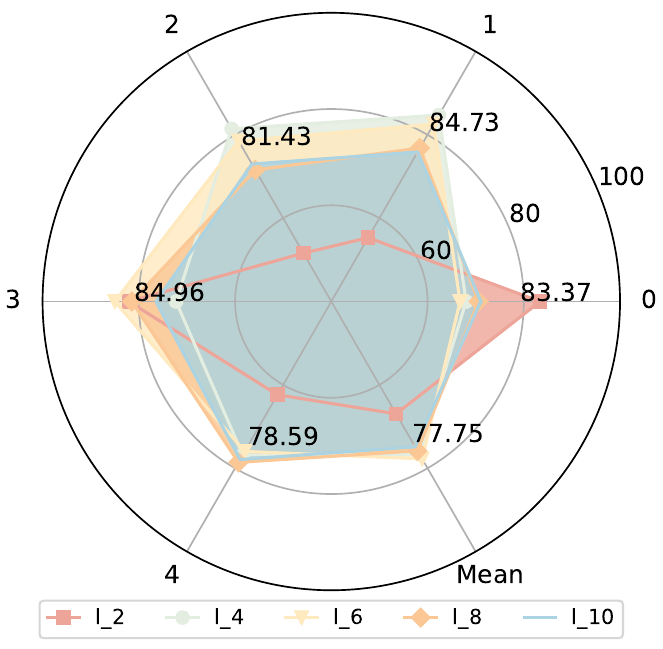}}\\
\caption{Per-client and average test accuracy (\%) on SVHN in the heterogeneous ResNet setting.}
\label{fig:labels_baselines_SVHN}
\end{figure}

\begin{figure}
    \centering
    \subfloat[IID]{\includegraphics[width=.33\linewidth]{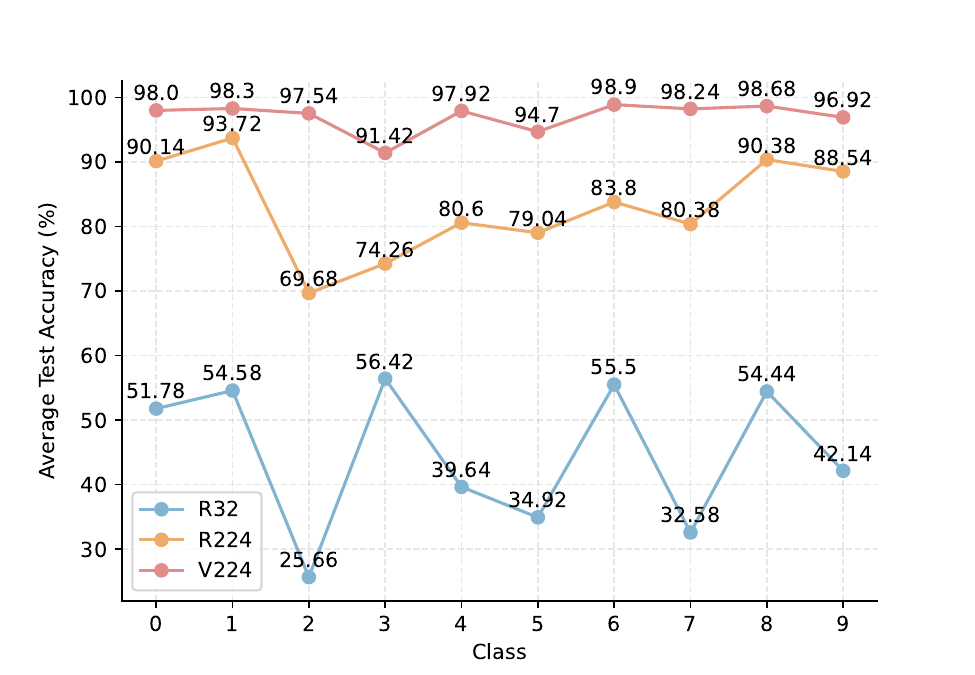}}
    \subfloat[Dir]{\includegraphics[width=.33\linewidth]{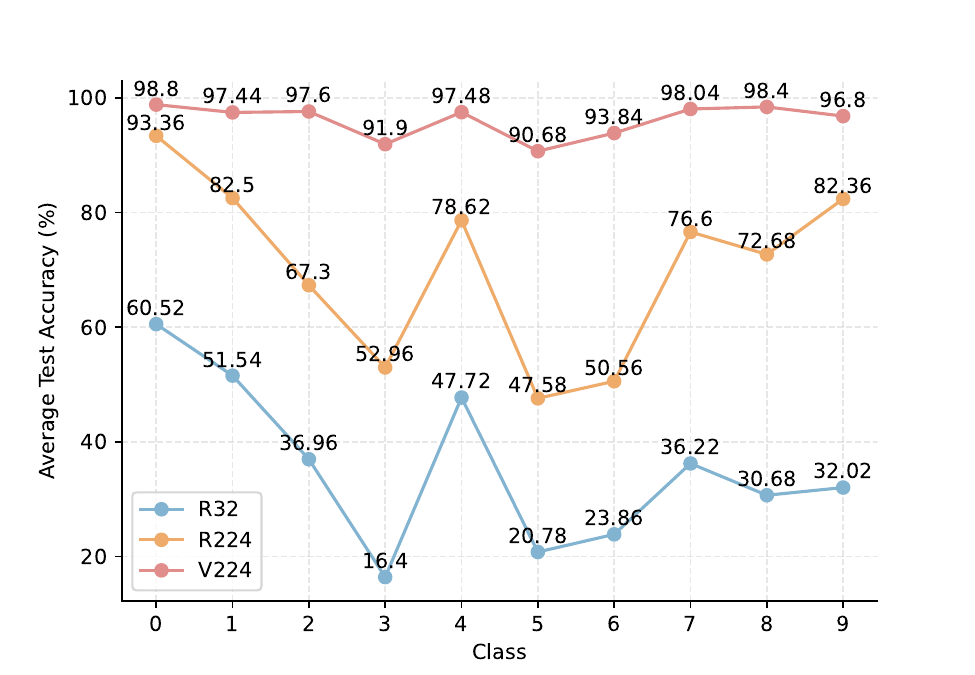}}
    \subfloat[Non-IID]{\includegraphics[width=.33\linewidth]{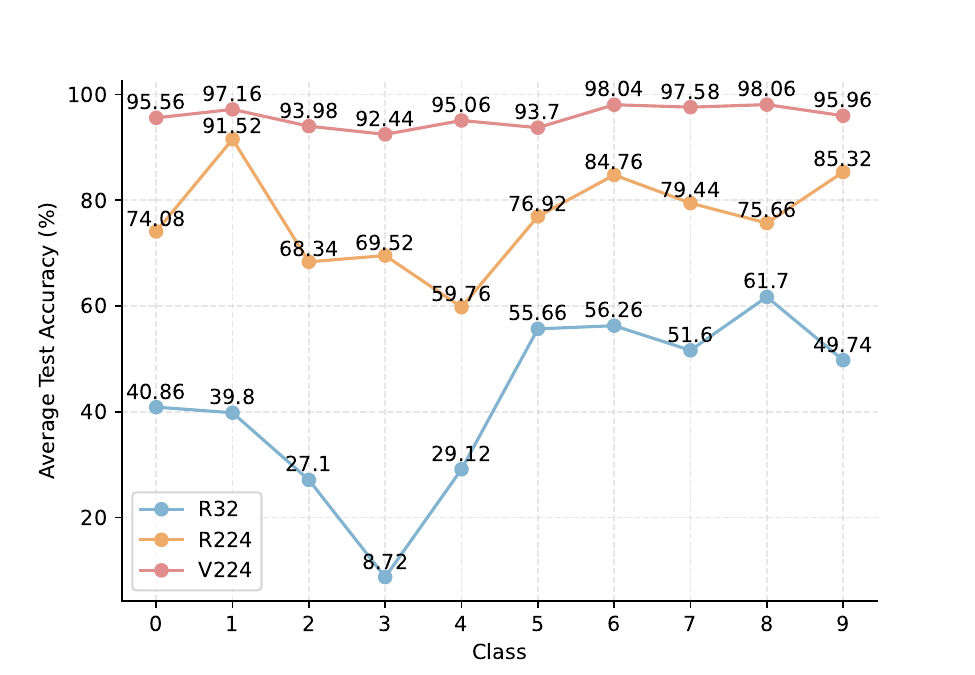}}
\caption{The average test accuracy of each class on the CIFAR10 dataset with FedHPL. `R' and `V' denote the heterogeneous ResNet model and heterogeneous model (ViT) setting. The latter digit is the size of an input image.}
\label{fig:res_vit_local_our}
\end{figure}

\newpage
\subsection{Effect of the training style of pre-trained backbones}
\label{appendix:backbone_mocov3}
Next, we turn our attention to the pre-trained backbone and explore the effect of the pre-training style of foundation models on model performance in FedHPL.
In addition to the pre-trained parameters which are trained with supervised learning, we also initialize the backbone with self-supervised pre-training on ImageNet1k without labels.
Notably, if there is no global logit for a certain label (\ie all clients predict wrong), client $k$ can replace $\tilde{p}_{k,c}$ with the local logit $p_k^i$.
This situation is because self-supervised pre-trained model parameters need more rounds and training time to adapt downstream tasks. 
So, in the initial global epochs, some labels may not be correctly classified, especially for the datasets with a lot of classes (\eg CIFAR100) and limited local samples.
It can be seen from Table~\ref{tab:moco-v3} that the average test accuracy of FedHPL with the pre-trained model by self-supervised learning is worse than the accuracy of FedHPL with the pre-trained model by supervised learning (as shown in Table~\ref{tab:homo_model} and Table~\ref{tab:hete_model}).
However, the test accuracy in FedHPL with the self-supervised pre-trained models is still better than other baselines on the CIFAR10 and CIFAR100 datasets.
Because pFedPG uses the pre-trained backbone trained by supervised learning, the performance is slightly lower than it.
Furthermore, the average test accuracy on the SVHN dataset is not as ideal as in supervised learning.
But it still represents a comparable performance to other algorithms.

\begin{table}
\caption{The average test accuracy (\%) in FedHPL with the ViT backbone pre-trained by a self-supervised learning method: MoCo-v3~\cite{MoCo-v3}. We perform 50 global rounds with one local epoch over all datasets for better adapt downstream tasks. Model setting can refer to Table~\ref{AppendixmodelSetting}.}
\label{tab:moco-v3}
\centering
\resizebox{0.99\textwidth}{!}{
    \begin{tabular}{c|ccc|ccc|ccc} 
    \hline
    \multirow{2}{*}{Model Setting} & \multicolumn{3}{c|}{CIFAR10} & \multicolumn{3}{c|}{CIFAR100} & \multicolumn{3}{c}{SVHN}  \\ 
    \cline{2-10}
                         & IID   & Dir   & Non-IID       & IID   & Dir   & Non-IID        & IID   & Dir   & Non-IID    \\ 
    \hline
    Homogeneous Model    & 96.50 & 91.56 & 93.94        & 80.22 & 62.37 & 61.19         & 91.73 & 88.22 & 85.94     \\
    Heterogeneous Model  & 93.93 & 87.71 & 89.70        & 71.63 & 57.34 & 56.64         & 90.57 & 87.11 & 85.26\\
    \hline
    \end{tabular}
}
\end{table}

\subsection{Details of ablation study and analysis}
\label{appendix_analysis_details}

\subsubsection{Effect of prompt length and insertion position}
\label{Appendix_prompt}
From the detailed comparison of prompt length and insertion position presented in Figure~\ref{fig:ablation_prompt_length_insert}, we can see that the model performance in VPT-deep is evidently higher than of VPT-shallow with less accuracy variance.
We also notice that the average model accuracy is robust in VPT-deep and we choose VPT-deep with $n=3$ prompts in each backbone layer for faster training, less trainable parameters, and better performance.

\begin{figure}[htbp]
    \centering
    \subfloat[VPT-shallow]{\includegraphics[width=.49\linewidth]{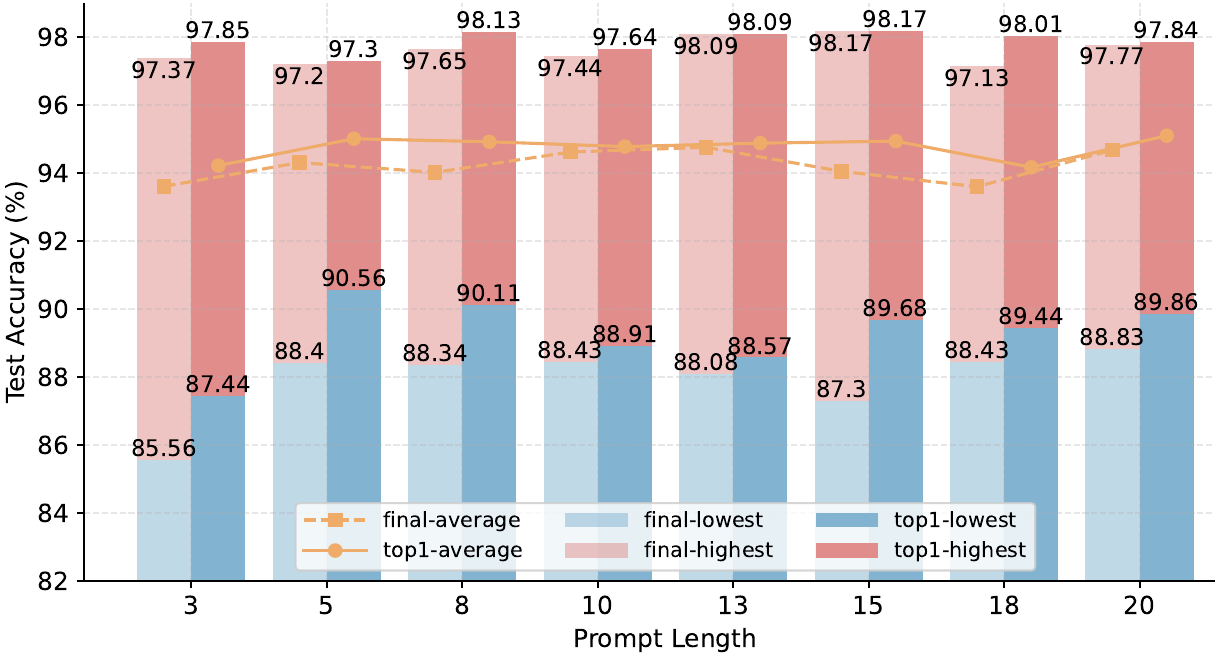}}
    \subfloat[VPT-deep]{\includegraphics[width=.49\linewidth]{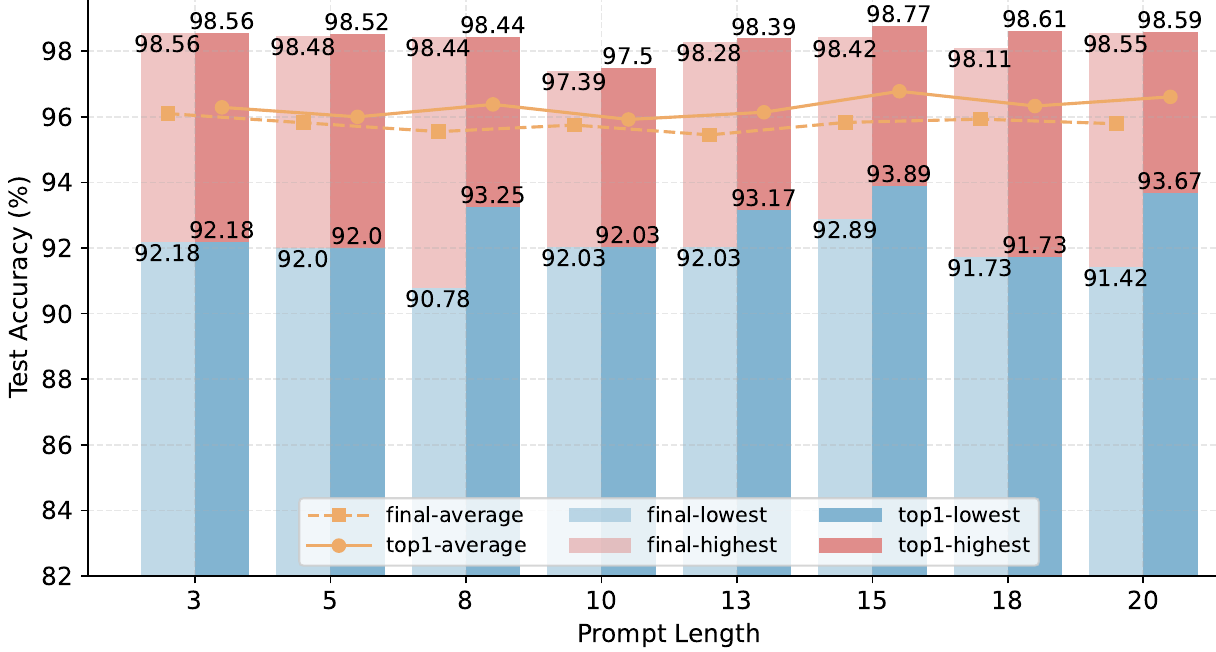}}
\caption{Ablation study on prompt length and insertion position on the CIFAR10 dataset in the \textbf{Dir} data and heterogeneous model (ViT) setting. `A-B' represents the epoch (final: the test accuracy in the last epoch; top1: the best test accuracy) and client accuracy (lowest: the lowest test accuracy of clients; average: the average test accuracy among all clients; highest: the highest test accuracy of clients). The left bar represents `final' and the right bar represents `top1'.}
\label{fig:ablation_prompt_length_insert}
\end{figure}

\subsubsection{Sensitivity to the number of involved training samples}
We next explore the sensitivity to the percentage of samples on the CIFAR10 dataset in the \textit{IID} data and \textit{homogeneous model} setting for clear comparison without other factors' influence.
The percentage of samples can reflect the effect of the number of samples $|\mathcal{D}_k|$ on the model performance.
We test the model accuracy of FedHPL with only local prompt tuning and no collaborative learning (\ie no global logit distillation).
Then we show the results over different percentages of involved local training samples in Figure~\ref{fig:ablation_theorem}.
We also compare the results with FedHPL over full training with local prompt tuning and global logit distillation.
The lowest test accuracy, average test accuracy, and the highest test accuracy of clients are shown in Figure~\ref{fig:ablation_theorem}.

\subsubsection{Necessity of weighted aggregation}
Furthermore, due to the fact that global logits are aggregated of local logits and further guiding the local training. 
Then clients generate the next global round of logits with the guidance of weighted logits.
The interaction inevitably causes slightly training oscillations and we agree that it will be an interesting future work to investigate that how to alleviate the unstable training over the situation.

\subsubsection{Effect of data heterogeneity}
In addition to the quantity of local samples, we also investigate the effect of data heterogeneity on model performance.
Especially, Dir($\alpha$) in the \textit{Non-IID} data setting has no overlap samples and $\alpha$ can control the data heterogeneity.
A smaller $\alpha$ corresponds to a more imbalanced data distribution and increases the data heterogeneity.
We set the minimum sample percentage for clients, which is 1\% in CIFAR10, 10\% in CIFAR100 (because the quantity of per-class samples is small), and 0.5\% in SVHN.
As shown in Table~\ref{tab:mocov3-alpha}, the convergence speed in FedHPL with the self-supervised pre-trained backbone is slow and the model performance has a degradation compared to supervised learning.
Moreover, the performance over the CIFAR10 and CIFAR100 datasets only has a slight degradation as the degree of data heterogeneity increases (\ie $\alpha$ reduces), further demonstrating the effectiveness of FedHPL in addressing data heterogeneity.
However, the performance in the SVHN dataset with $\alpha=0.1$ is not ideal because the number of local samples is small (\eg only 27 samples in label 7 across 4 clients).
In future work, we will handle the issue of model performance degradation caused by too few local samples.

\begin{table}[ht]
\caption{The average test accuracy (\%) among clients with different backbones trained by supervised learning and self-supervised learning (MoCo-v3). We run 10 global rounds on CIFAR10 and 15 global rounds on the CIFAR100 and SVHN datasets with the `Supervised' backbone. We run the same epochs with MoCo-v3 and denoted as MoCo-v3 (1) whereas MoCo-v3 (2) represents the 50 global training rounds. The local epoch is 1. The model setting can refer to Table~\ref{AppendixmodelSetting}.}
\label{tab:mocov3-alpha}
\centering
\resizebox{0.98\textwidth}{!}{
    \begin{tabular}{c|ccc|ccc} 
    \hline
    \multirow{2}{*}{}         & \multicolumn{3}{c|}{Homogeneous Model}   & \multicolumn{3}{c}{Heterogeneous Model}   \\ 
    \cline{2-7}
                              & MoCo-v3 (1)  & MoCo-v3 (2)  & Supervised & MoCo-v3 (1)  & MoCo-v3 (2)  & Supervised \\ 
    \hline
    \multicolumn{7}{l}{{\cellcolor[rgb]{0.902,0.902,0.902}}\textbf{CIFAR10 Dataset}} \\
    $\alpha$=0.1              & 45.73        & 88.97        & 91.43      & 65.08&              78.05&             92.48\\
    $\alpha$=0.5              & 84.14        & 93.28        & 96.13      & 75.40&              84.85& 95.51\\
    $\alpha$=1.0              & 89.98        & 94.38        & 96.90      & 84.22& 89.12& 96.23      \\
    \hline
    \multicolumn{7}{l}{{\cellcolor[rgb]{0.902,0.902,0.902}}\textbf{CIFAR100 Dataset}} \\
    $\alpha$=0.1&             27.00&             60.27&            86.01& 34.92& 52.91& 84.65\\
    $\alpha$=0.5&             35.88&             67.24&            86.92& 40.83& 58.91& 85.61       \\
    $\alpha$=1.0&             39.51&             69.21&            88.27& 43.37& 60.41& 86.66       \\
    \hline
    \multicolumn{7}{l}{{\cellcolor[rgb]{0.902,0.902,0.902}}\textbf{SVHN Dataset}} \\
    $\alpha$=0.1&             42.35&             52.81&            72.91& 45.46& 53.73& 68.87\\
    $\alpha$=0.5&             71.75&             81.26&            88.58& 73.46& 79.06& 87.20\\
    $\alpha$=1.0&             79.85&             85.94&            91.68& 82.74& 85.26& 90.32\\
    \hline
    \end{tabular}
}
\end{table}

\subsubsection{Effect of hyper-parameter}
\label{Appendix_hyperparameter}
\textbf{Hyper-parameter in the loss function.}
Several factors can affect the model performance and we further inspect the sensitivity of FedHPL to the hyper-parameters of the loss function.
We select $\gamma$ from [0, 2] and $\mathcal{T}$ from [1.5, 5] on CIFAR10 dataset in the \textit{Dir} data and \textit{heterogeneous model} (ViT) setting over VPT-shallow and VPT-deep.
It can be observed from Figure~\ref{fig:hyper_parameter_compare} that the highest test accuracy is robust on hyper-parameters of the loss function while $\gamma$ and $\mathcal{T}$ have a quite obvious influence on the lowest test accuracy.
The lowest client accuracy in the final epoch has a better performance when $\mathcal{T}$=3.5 or 4.5 over VPT-shallow and VPT-deep.
The ideal performance on $\gamma$ is mainly located in the interval of [1.5, 2.0] over VPT-shallow while it is located in [0.5, 1.0] over VPT-deep. 
Above all, the appropriate hyper-parameter is a specific task for different situations in FedHPL and a more generalized pre-trained model is less affected by hyper-parameters.
We also find that the model has a more stable training and more accurate estimation over VPT-deep than VPT-shallow.
Furthermore, we find that the model performance is more stable when batch size $bs=16$.
For example, the accuracy is 91.86\% ($bs=16$) compared to 90.75\% ($bs=32$) on the SVHN dataset in the \textit{Dir} data and \textit{heterogeneous model} (ViT) setting and 97.89\% ($bs=16$) compared to 95.50\% ($bs=32$) on the CIFAR10 dataset in the \textit{IID} data and \textit{homogeneous model} (ViT) setting.

\textbf{Local epochs.}
We next investigate the number of local epochs on the model performance in Table~\ref{tab:local_epoch}.
It can be observed that the test accuracy is robust to the local epoch in the \textit{heterogeneous model} setting.
However, in the \textit{homogeneous model} setting, the performance has an increase with the number of local epochs, especially for the lowest test accuracy among clients.

\newpage

\begin{figure}[ht]
    \centering
    \subfloat[$\gamma$ in VPT-shallow ($\mathcal{T}$=3.0)]{\includegraphics[width=.5\linewidth]{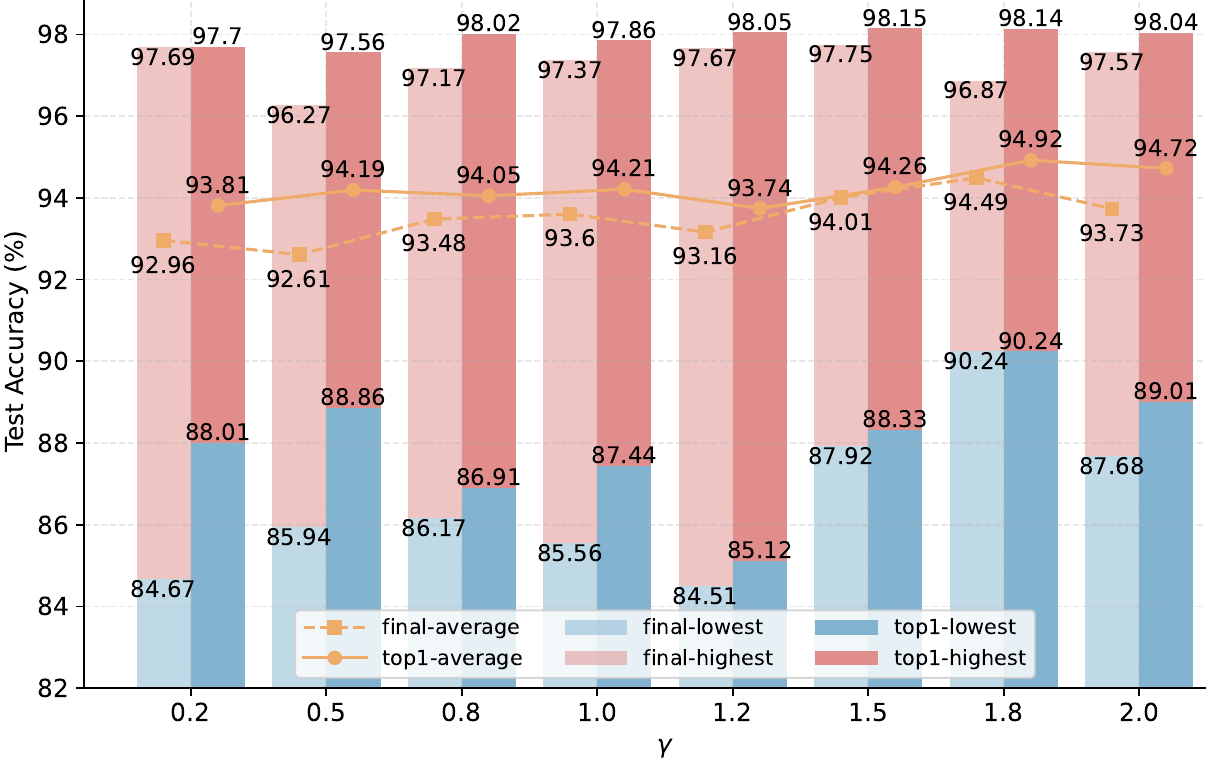}}
    \subfloat[$\gamma$ in VPT-deep ($\mathcal{T}$=4.5)]{\includegraphics[width=.5\linewidth]{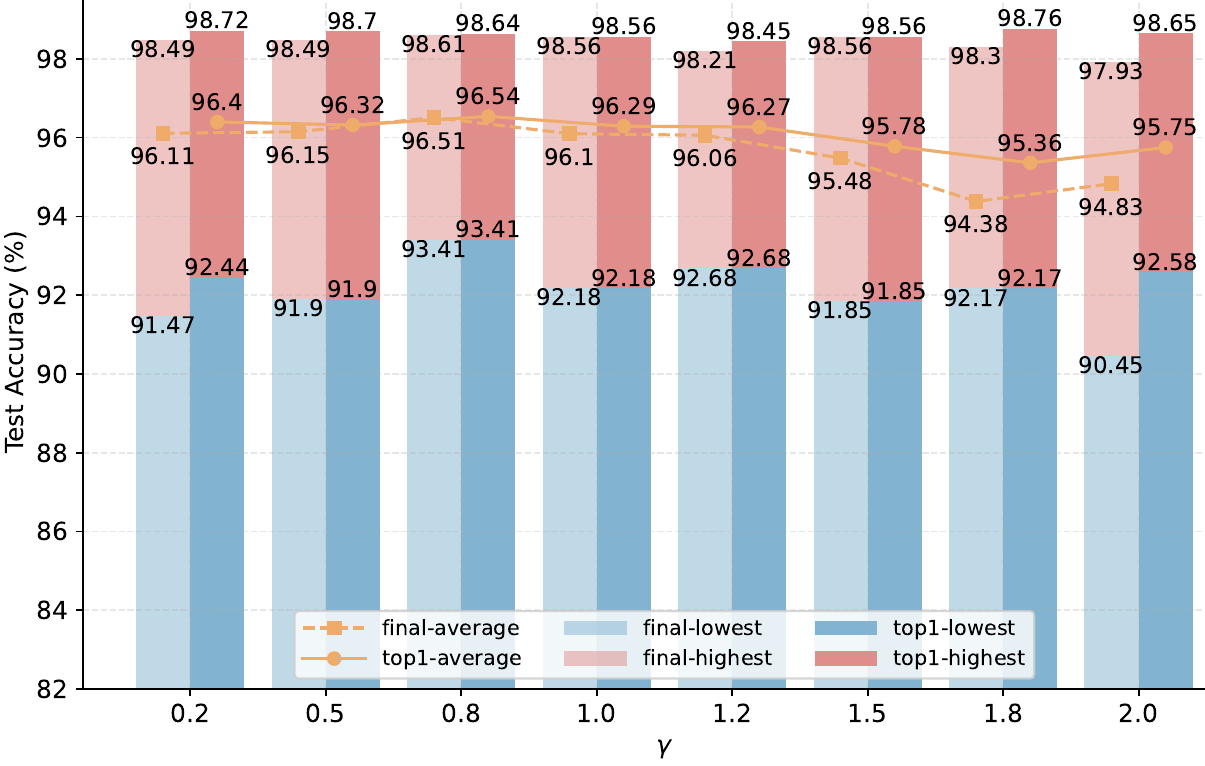}}\\
    \subfloat[$\mathcal{T}$ in VPT-shallow ($\gamma$=1.0)]{\includegraphics[width=.5\linewidth]{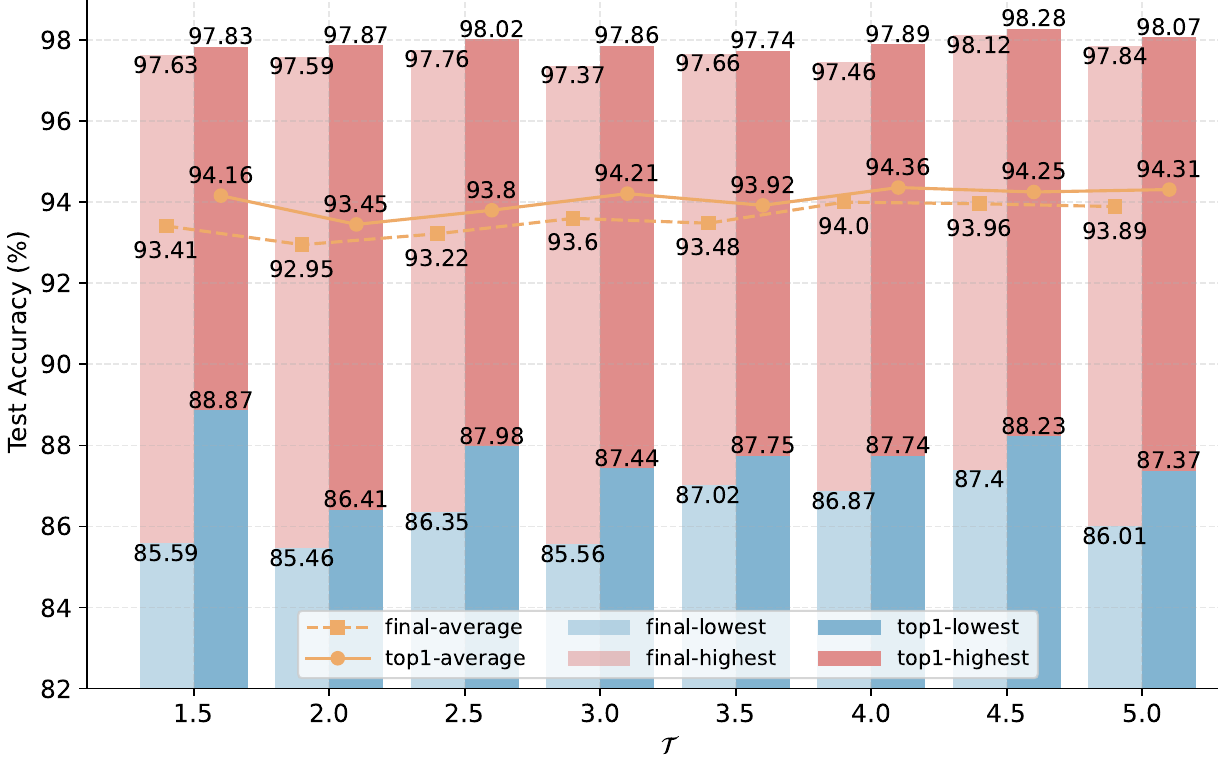}}
    \subfloat[$\mathcal{T}$ in VPT-deep ($\gamma$=1.0)]{\includegraphics[width=.5\linewidth]{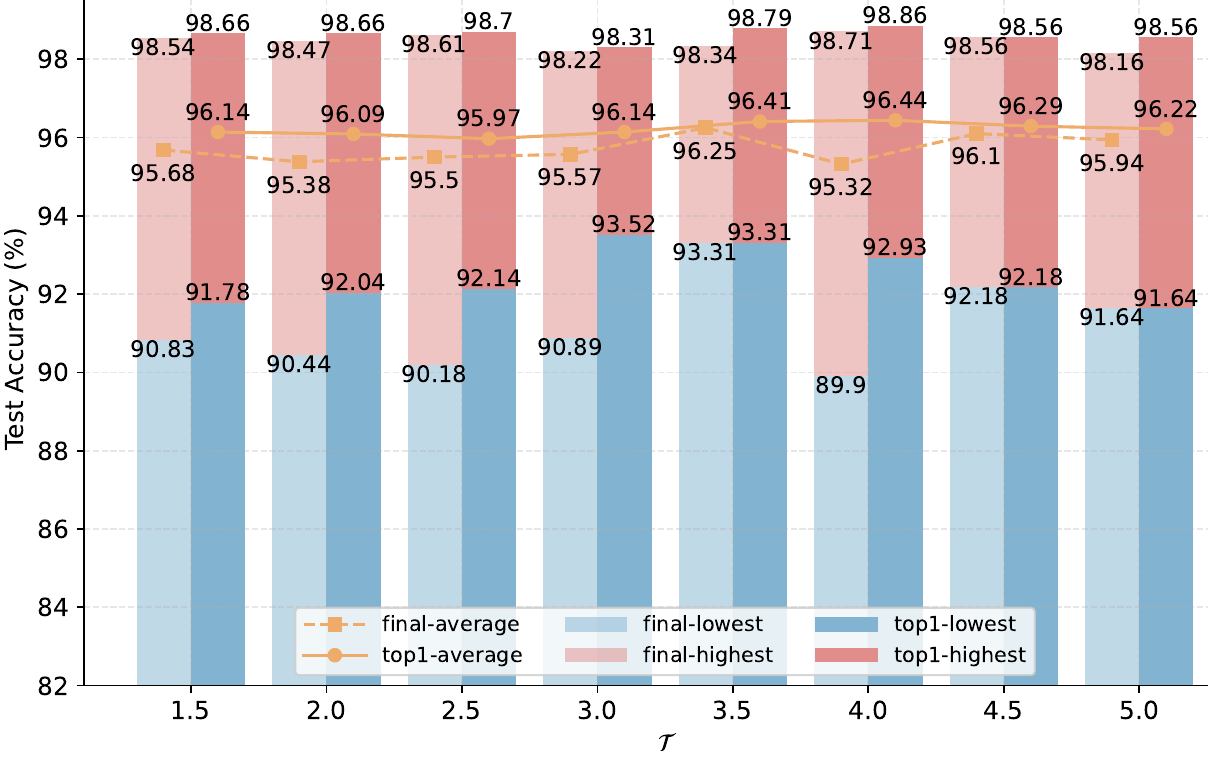}}\\
\caption{Ablation study on different hyper-parameters of the loss function on the CIFAR10 dataset in the Dir data and heterogeneous model (ViT) setting. We compare the test performance of clients on the lowest, average, and highest accuracy. We have explained the specific meaning in Figure~\ref{fig:ablation_prompt_length_insert}.}
\label{fig:hyper_parameter_compare}
\end{figure}

\begin{table}[hb]
\caption{Test accuracy (\%) on the CIFAR10 dataset in the Non-IID ($\alpha=0.1$) data setting. We show the test accuracy of the lowest client and highest client with the average test accuracy among clients over different local epochs $T_c$ with 10 global rounds.}
\label{tab:local_epoch}
\centering
    \begin{tabular}{c|ccc|ccc} 
    \hline
    \multirow{2}{*}{} &\multicolumn{3}{c|}{Homogeneous Model}   & \multicolumn{3}{c}{Heterogeneous Model}  \\
    \cline{2-7}
               & Lowest & Average & Highest  & Lowest & Average & Highest \\ 
    \hline         
    $T_c=1$    & 84.94  & 91.43   & 95.15    & 90.24  & 92.48   & 97.53    \\
    $T_c=2$    & 87.92  & 92.75   & 95.34    & 88.98  & 92.24   & 97.27    \\
    $T_c=3$    & 92.89  & 94.24   & 95.02    & 89.00  & 92.48   & 97.33    \\
    $T_c=4$    & 91.61  & 93.62   & 94.72    & 88.17  & 92.27   & 96.67   \\
    $T_c=5$    & 92.49  & 93.34   & 94.22    & 86.63  & 91.73   & 96.63    \\
    \hline
    \end{tabular}
\end{table}

\end{document}